\documentclass[12pt]{amsart}
\usepackage{amsbsy,amssymb,amsmath,amsthm,amscd,amsfonts,latexsym,amstext,delarray,
amsmath,graphicx,url} 
\usepackage{qtree}
\usepackage[margin=.8in]{geometry}
\usepackage{color}
\usepackage[all]{xy}

\newtheorem{thm}{Theorem}[section]
\newtheorem{prop}{Proposition}[section]
\newtheorem{cor}[thm]{Corollary}
\newtheorem{lem}[thm]{Lemma}

\newtheorem{defn}[thm]{Definition}
\newtheorem{rem}[thm]{Remark}
\newtheorem{ex}[thm]{Example}

\numberwithin{equation}{section}

\def\Q{{\mathbb Q}}
\def\Z{{\mathbb Z}}

\def\R{{\mathbb R}}

\def\P{{\mathbb P}}

\def\cA{{\mathcal A}}
\def\cB{{\mathcal B}}
\def\cC{{\mathcal C}}
\def\cD{{\mathcal D}}
\def\cE{{\mathcal E}}
\def\cF{{\mathcal F}}
\def\cG{{\mathcal G}}
\def\cH{{\mathcal H}}
\def\cI{{\mathcal I}}

\def\cL{{\mathcal L}}
\def\cM{{\mathcal M}}

\def\cO{{\mathcal O}}
\def\cP{{\mathcal P}}

\def\cR{{\mathcal R}}
\def\cS{{\mathcal S}}

\def\cV{{\mathcal V}}

\def\bP{{\mathbb P}}

\def\fB{{\mathfrak B}}
\def\fM{{\mathfrak M}}
\def\fT{{\mathfrak T}}
\def\fF{{\mathfrak F}}
\def\fC{{\mathfrak C}}
\DeclareMathOperator*{\Hom}{Hom}

\title{Syntax-semantics interface: an algebraic model}
\author{Matilde Marcolli, Robert C.~Berwick, Noam Chomsky}
\date{2023}
\address{Departments of Mathematics and of Computing and Mathematical Sciences, 
California Institute of Technology, Pasadena, CA 91125, USA}
\email{matilde@caltech.edu}
\address{Institute for Data, Systems, and Society,
Massachusetts Institute of Technology, 
Cambridge MA 02141, USA}
\email{berwick@csail.mit.edu}
\address{Department of Linguistics, 
University of Arizona, Tucson, AZ 85721, USA}
\email{noamchomsky@email.arizona.edu}
\email{chomsky@mit.edu}

\begin{document}
\maketitle

\begin{abstract}
We extend our formulation of Merge and Minimalism in terms of Hopf algebras to an
algebraic model of a syntactic-semantic interface. We show that methods adopted in
the formulation of renormalization (extraction of meaningful physical values) in 
theoretical physics are relevant to describe the extraction of meaning from syntactic 
expressions. We show how this formulation relates to computational models of
semantics and we answer some recent controversies about implications for generative
linguistics of the current functioning of large language models. 
\end{abstract}

\setcounter{tocdepth}{3}
\tableofcontents

\section{Introduction: modelling the syntax-semantics interface}

The modelling of the generative process of syntax, based on the core computational
structure of Merge, within the setting of the Minimalist Model, satisfies the
following fundamental properties:
\begin{enumerate}
\item a concise conceptual framework;
\item a precise mathematical formulation;
\item good explanatory power.
\end{enumerate}
For the most recent formulation of Minimalism,  the first and third property are articulated 
in \cite{Chomsky19}, \cite{Chomsky21}, \cite{Chomsky23} and in the upcoming
\cite{CSBFHKMS}. While the requirement of the existence of precise mathematical models has
been traditionally associated to sciences like physics, the fact that syntax is
essentially a computational process suggests
that a mathematical formulation should also be a desirable requirement in linguistics, and
more specifically in the modelling of I-language.
We presented a detailed mathematical model of syntactic Merge in our previous papers 
\cite{MCB} and \cite{MBC}. 

\smallskip

In comparison with syntax, the modeling of semantics is presently in a less
satisfactory state from the point of view of the same three properties listed above.
Some main approaches to semantics include forms of compositional semantics \cite{Pietro}, \cite{Pietroski},
truth-conditional semantics, semiring semantics \cite{Goodman}, and vector-space models 
(the latter especially in computational linguistics). General views of logic-oriented approaches
to semantics, which we will not be discussing in this paper, can be seen, for instance, 
in \cite{Scharp}, \cite{Tennant}, \cite{Zard}.
Each of these viewpoints has limitations of a different nature. 

\smallskip

Our purpose here is not to carry out a comprehensive comparative
analysis and criticism of contemporary models of semantics.  Rather,
we want to approach the problem of modeling the syntactic-semantic
interface on the basis of a list of abstract properties, and
articulate a possible mathematical setting that such properties
suggest. We can then compare existing models with the specific
structure that we identify. We will show that one can remain, to some
extent, agnostic about specific models of semantics, beyond some basic
requirements, and still retain a fundamental functioning model of the
interaction with syntax. This reflects a view of the syntax-semantics
interface that is primarily syntax-driven.

\smallskip

As in the case of our mathematical formulation of Merge in terms of
Hopf algebras, our guiding principle will be an analogy with
conceptually similar structures that arise in theoretical physics. In
particular, in the context of fundamental physical interactions
described by the quantum field theory, fundamental problem is the
assignment of ``meaningful" physical values to the computation of the
expectation values of the theory. This can be compared with the
assignment of meaning--semantics--to syntactic objects. More
precisely, assignment of meaning the quantum field theory setting
consists of the extraction of a finite (meaningful) part from Feynman
integrals that are in general divergent (produce meaningless
infinities). This process is known in theoretical physics as {\em
  renormalization}. While the renormalization problem and procedures
leading to satisfactory solutions for it have been known to physicists since
the development of quantum electrodynamics in the 1950s and 60s (see
\cite{BoShi}), a complete understanding of the underlying mathematical
structure is much more recent, (see \cite{CoKr}, \cite{CoMa}).

Even more recently, it has been shown shown that the same
mathematical formalism can be applied in the theory of computation, in
order to extract, in a similar way, computable ``subfunctions" from
non-decidable problems (undecidability being the analog in the theory
of computation of the unphysical infinities); see
\cite{Man1}, \cite{Man2}, and also \cite{DelMar}.

Assuming the conceptual standpoint that Internal or I-language is, in
essence, a computational process, the extension of the mathematical
framework of renormalization to the theory of computation suggests the
existence of a similar possible manifestation in
linguistics as well.  In the case of linguistics, one does not have to
deal with divergences (of a physical or computational nature); rather
one has to carry out a consistent assignment of meaning to syntactic
objects produced by the Merge mechanism, and reject inconsistencies
and impossibilities.  In the rest of this paper we plan to turn this
heuristic comparison into a precise formulation.

\smallskip

There are several reasons why developing such a mathematical model of
the syntax-semantics interface is desirable. Aside from general
principles based on the three ``good properties" of theoretical
modeling stated above, there are other possible applications of
interest.  For instance, a significant ongoing debate and controversy
has ensued from the recent development of large language models, with
various claims of incompatibility with, or ``disproval'' of the
generative linguistics framework itself. Since such theories, in our
view, ultimately describe computational processes (albeit most likely
also in our view of a different nature from those governing language
in human brains), a viable computational and mathematical setting is
required, where a specific comparative analysis can be carried out,
and such claims can be addressed.

\medskip
\subsection{Some conceptual requirements for a syntax-semantics interface} \label{RequireListSec}

We start our analysis by setting out a simple list of what we regard as desirable
properties of a model of the syntax-semantic interface. 

\begin{enumerate}
\item Autonomy of syntax
\item Syntax supports semantic interpretation 
\item Semantic interpretation is, to a large extent, independent of externalization  
\item Compositionality 
\end{enumerate}

The first requirement, the autonomy of syntax, expresses that the
computational generative process of syntax described by Merge is independent
of semantics. The second requirement can be seen as positing that the
syntax-semantic interface proceeds {\em from} syntax {\em to} semantics (a
syntax-first view), while syntax itself is not semantic in nature. The third
claim separates the interaction of the core computational mechanism of
syntax with a Conceptual-Intentional system, which gives rise to the 
syntax-semantic interface, from the interaction with 
an Articulatory-Perceptual or Sensory-Motor system, which includes
the process of externalization. While it is reasonable to assume a certain
level of interaction between these two mechanisms, with ``independence of
externalization" we emphasize that
semantic interpretation depends primarily on {\em structural relations} and 
proximity in the syntactic structure rather than on linear proximity of
words in a sentence. The compositional property is meant here simply as a
requirement of consistency across syntactic sub-structures. 

\smallskip

We also add another general principle that we will try to incorporate in our
model and that may be at odds with some of the traditional approaches to
semantics (such as the truth value based approaches). We propose the following fundamental
distinction between the roles of syntax and semantics in language
\begin{itemize}
\item Syntax is a computational process.
\item Semantics is {\em not} a computational process and is in
essence grounded on a notion of topological proximity.
\end{itemize}
The first statement is clear in the context of generative linguistics, and
in particular in the setting of Minimalism, where the computational
process is run by the fundamental operation Merge. The second assertion
requires some contextual clarification. Saying 
that semantics is only endowed with a notion of proximity of a topological nature 
does {\em not} mean that it is not possible, or desirable, to consider models of semantics where
additional structure is present, but rather that these additional properties (metric, linear,
semiring structures, for instance) only play a role to instantiate or
quantify proximity relations. The compositionality of semantics
does not require positing an additional computational structure on
semantics itself: the computational structure of syntax suffices to induce it.
In this view, semantics is not really a part of language itself, 
but rather an autonomous structure that deals with proximity classifications. 

\medskip
\subsection{Syntax}\label{SyntaxMergeSec}

On the syntax side of the syntax-semantic interface we assume the
formulation of free symmetric Merge presented in \cite{MCB}. This
accounts for the properties (1) and (3) in our list of \S
\ref{RequireListSec}: it provides a computational model of syntax that
is independent of semantics, and where the interface with semantics
takes place at the level of the free symmetric Merge, without
requiring prior externalization. Free symmetric Merge generates
syntactic objects, described by binary rooted trees without any
assigned planar embedding.  Thus, our choice of modeling the
syntax-semantic interface starting from the level of free symmetric
Merge as the syntactic part of the interface, has the effect of
ensuring that the interface of syntax and semantics (also sometimes
called the Conceptual-Intentional system) is parallel and separate
from the channel connecting the output of Merge to externalization
(the so-called Articulatory-Perceptual system), although interactions
between these two channels can be incorporated in the model (and will
be discussed in \S \ref{AssocSec} of this paper).

\smallskip

Summarizing the setting of \cite{MCB}, syntax is represented by the following data:
\begin{itemize}
\item a (finite) set $\cS\cO_0$ of {\em lexical items and syntactic features};
\item the set of {\em syntactic objects} $\cS\cO$, identified with the set $\fT_{\cS\cO_0}$ of
binary rooted trees (with no planar structure) with leaves labeled by $\cS\cO_0$,
generated as the free, non-associative, 
commutative magma over  the set $\cS\cO_0$, 
\begin{equation}\label{SOmagma}
 \cS\cO={\rm Magma}_{na,c}(\cS\cO_0, \fM)=\fT_{\cS\cO_0} \, ; 
\end{equation} 
\item the set of {\em accessible terms} of a syntactic object $T\in \fT_{\cS\cO_0}$, given by the
set of all the full subtrees $T_v\subset T$ with root a non-root vertex $v\in V(T)$; 
\item the commutative Hopf algebra of workspaces given by the vector space
$\cV(\fF_{\cS\cO_0})$ spanned by the set $\fF_{\cS\cO_0}$ of (finite) binary
rooted forests with leaves decorated by elements of the set $\cS\cO_0$, with
product given by the disjoint union $\sqcup$ and coproduct determined by
\begin{equation}\label{coprod}
 \Delta(T)=T\otimes 1 + 1 \otimes T + \sum_v F_{\underline{v}} \otimes T/F_{\underline{v}} \, , 
 \end{equation} 
with $F_{\underline{v}}= T_{v_q}\sqcup \cdots \sqcup T_{v_n}$ a collection
of accessible terms;
\item the action of Merge on workspaces 
$$ \fM = \sqcup \circ (\fB  \otimes {\rm id}) \circ \Delta $$
where $\cB$ is the grafting of components of a forest to a common root vertex,
or for a fixed pair of syntactic objects $S,S'$
\begin{equation}\label{MergeSS}
 \fM_{S,S'} = \sqcup \circ (\fB  \otimes {\rm id}) \circ \delta_{S,S'} \circ \Delta \, , 
\end{equation} 
where $\delta_{S,S'}$ selects matching pairs in the workspace (see \cite{MCB}
for a more detailed description).
\end{itemize}

We will use the notation $\cH =(\cV(\fF_{\cS\cO_0}),\sqcup,\Delta, S)$ for the
Hopf algebra described above. Note that since the Hopf algebra is graded,
the antipode $S$ is defined inductively using the coproduct, so that we can
equivalently just specify the bialgebra part of the structure,  $\cH =(\cV(\fF_{\cS\cO_0}),\sqcup,\Delta)$.

\medskip
\subsubsection{Remark on the Hopf algebra coproduct}\label{QuotSec}

We pointed out in \cite{MCB} that there are two possible ways of interpreting the
quotient $T/T_v$ (or more generally $T/F_{\underline{v}}$) in the
coproduct \eqref{coprod}, either as contraction of $T_v$ to its root vertex or as
deletion of $T_v$ (and taking the unique maximal binary tree determined
by the complement). In \cite{MCB} we argued that, if one wants to avoid
having to introduce labeling algorithms for the internal vertices of the trees,
and only have the leaves labelled by lexical items and syntactic features 
in $\cS\cO_0$, then the second procedure is preferable.

\smallskip

However, when it comes to interfacing syntax with semantics, it is
in fact better to retain the root vertex $v$ of $T_v$ in the quotient $T/T_v$
as that provides so-called {\em traces}, the empty categories left behind by ``movement'' 
implemented by so-called Internal Merge. 
As is familiar from the long historical discussion of what is called
reconstruction, the trace is needed for semantic parsing,
so in the context we consider here we will be using the quotient $T/T_v$
where $T_v$ is contracted to its root, marked as trace. Similarly, in
quotients $T/F_{\underline{v}}$ each tree component $T_{v_i}$ of the forest 
$F_{\underline{v}}$ is contracted to its root vertex $v_i$.\footnote{We thank Martin Everaert and Riny Huijbregts for this observation.}

\smallskip
\subsubsection{A comment about Tree Adjoining Grammars -- TAGs} \label{TAGsec}

Since readers of our previous papers \cite{MCB}, \cite{MBC} have
occasionally asked this question, we add here a very brief
clarification on the difference between the algebraic structure of
Merge described in \cite{MCB} and that of tree adjoining grammars
(TAGs). In the setting of TAGs, one considers a generative process
that depends on an initial choice of a given finite set of
``elementary trees" with vertex labels. In TAGs, in general, trees are
not necessarily assumed to be binary. There are two composition rules:
one composition operation (substitution rule) consists of grafting the
root of a tree to the leaf of another tree; a second composition
operation (a so-called adjoining rule) inserts at an internal vertex of a tree
with a label $x$ another tree with root labeled by $x$, and one of the
leaves also labeled by $x$. The adjoining rule can be obtained as a
suitable composition of grafting of roots to leaves, so the basic
generative operations of TAGs are the {\em operad compositions} of
rooted trees. Namely, if $\cO(n)$ denotes the set of trees in a given
TAG with $n$ leaves, then there are composition maps
\begin{equation}\label{operadi}
 \circ_i : \cO(n)\times \cO(m) \to \cO(n+m-1) 
\end{equation} 
that plug the root of a tree in $\cO(n)$ to the $i$-th leaf of a tree
in $\cO(m)$ resulting in a tree in $\cO(n+m-1)$.  Such operations,
subject to associativity and unitarity conditions, define the
algebraic structure of an {\em operad}.  Label matching conditions may
require partial operads, but we will not discuss this here.

\smallskip

In order to compare the TAG formalism with the algebraic formulation
of Merge of \cite{MCB}, one should note that there is an important
relation between the two as well as important differences: the latter
show that these two formalisms do {\em not} constitute the same
algebraic structure.  This is why, in our view, it is algebraic
structure that is essential to the line of work presented here, rather
than the formal language theoretic notion of weak generative capacity
(such as mild-context sensitivity).\footnote{In other words, this is
  not to deny that notions of generative capacity might be useful to
  illuminate one or another aspect of human language; simply that the
  algebraic approach presented here does not draw on this more
  familiar tradition.}

\smallskip

The relation between TAGs and Merge arises from the fact that, in the
Merge formalism of \cite{MCB}, recalled in \S \ref{SyntaxMergeSec}
above, the syntactic objects $T\in \cS\cO=\fT_{\cS\cO_0}$ are
generated as elements of the free non-associative commutative magma
\eqref{SOmagma} on the Merge operation $\fM$. This {\em does} have an
equivalent operad formulation, in terms of the quadratic operad freely
generated by the single commutative binary operation $\fM$, see
\cite{Holt}. Thus, there exists an equivalent way of formulating the
generative process for the syntactic objects in terms of operad
compositions \eqref{operadi}, that makes this generative process
appear similar to TAGs. However, the main difference between the two
lies in the fact that the Merge formalism does not {\em just} consist
of the generation of syntactic objects through the magma operation,
but also of the action of Merge on so-called workspaces, given by
forests $F\in \fT_{\cS\cO_0}$.

The action of Merge on workspaces is not determined {\em only} by the
operad underlying the $\cS\cO$ magma, but also requires the additional
datum of the Hopf algebra structure on workspaces. This makes it
possible to incorporate not just so-called External Merge, that is
involved in the magma $\cS\cO$, but also Internal Merge, that requires
an additional coproduct operation.

It is important to observe here that the operad underlying
$\cS\cO$ also determines a Hopf algebra, the associative, commutative
Hopf algebra on the vector space $\cV(\fF_{\cS\cO_0})$ spanned by
forests of binary rooted trees, used in \cite{MCB} to formulate the
action of Merge on workspaces.  However, this is {\em not} the same as the {\em
  non-associative}, commutative Hopf algebra structure induced by the
operad on the vector space $\cV(\fT_{\cS\cO_0})$ spanned by binary
rooted trees as in TAGs; see \cite{Holt2}, \cite{Holt3}.  This is the
key algebraic difference between TAGs and Merge.  The introduction of
workspaces and the action of Merge on workspaces thus amounts to a key
innovation in the modern Minimalist account.

\medskip
\subsection{Abstract head functions}

In order to formulate more precisely property (2) of our list in \S
\ref{RequireListSec}, we start by considering the role of the notion
of {\em head} in syntax and semantics.  In a syntactic tree, as is
familiar in general the syntactic category of the head determines the
category of the phrase (verb for Verb Phrase, etc.; here we use the
traditional terminology of ``verb phrase'' even though this is
actually described by a set).   Moreover, the syntactic head determines the
``type'' of objects described; hence it can be regarded as part of the
mechanism that interfaces syntax with semantics.

\smallskip

As we discussed in \cite{MBC}, one can define an abstract {\em head function}
on binary rooted trees $T$ (with no assigned planar structure) in the following way.

\begin{defn}\label{headfunc} {\rm 
A {\em head function} is a function $h$ defined on a subdomain ${\rm Dom}(h)\subset \fT_{\cS\cO_0}$,
that assigns to a $T\in {\rm Dom}(h)$ a map $h: T \mapsto h_T$, 
\begin{equation}\label{headLT}
h_T: V^o(T)\to L(T) 
\end{equation}
from the set $V^o(T)$
of non-leaf vertices of $T$ to the set $L(T)$ of leaves of $T$,
with the property that if $T_v \subseteq T_w$ and
$h_T(w)\in L(T_v)\subseteq L(T_w)$, then $h_T(w)=h_T(v)$.
We write $h(T)$ for the value of $h_T$ at the root of $T$. }
\end{defn}

\smallskip

This notion summarizes the main properties of the syntactic head,
though of course one can have many more abstract head functions that
do not correspond to the actual syntactic head. 

\smallskip

To see this note that our notion of head function of Definition~\ref{headfunc} can
be directly derived from the formulation of the notion of head and 
projection given by Chomsky in \S 4 of \cite{Chomsky-bare}.
The equivalence of these formulations follows immediately by observing 
that in \S 4 of  \cite{Chomsky-bare} the syntactic head
is characterized by the following inductive properties:
\begin{enumerate}
\item For $T=\fM(\alpha,\beta)$, with $\alpha,\beta\in \cS\cO_0$, the head $h(T)$ should be one or the
other of the two items $\alpha,\beta$. The item that becomes the head $h(T)$ is said to {\em project}.
\item In further projections the head is obtained as the 
``head from which they ultimately project, restricting the term head to terminal elements".
\item Under Merge operations $T=\fM(T_1,T_2)$ one of the two syntactic objects $T_1,T_2\in \cS\cO$
projects and its head becomes the head $h(T)$. The label of the structure $T$ formed by Merge is the head of
the constituent that projects. 
\end{enumerate}

\smallskip

\begin{lem}\label{headprojection}
The three properties listed above are equivalent to Definition~\ref{headfunc}.
\end{lem}

\begin{proof}
First observe that the three properties from \S 4 of \cite{Chomsky-bare} listed
above determine a function $h_T: V^o(T)\to L(T)$ from the set $V^o(T)$
of non-leaf vertices of $T$ to the set $L(T)$ of leaves of $T$. The function
is defined by ``following the head" determined by the three listed properties.
In other words, the root vertex of the tree carries a label, which by the listed
requirements is obtained as ``the head from which it ultimately projects", which
is assumed to be ``a terminal element". This means that we are assigning to
the root vertex a label $h(T)$ that is one of the items in $\cS\cO_0$ 
attached to the leaves $L(T)$. Similarly, for any other internal vertex
$v$ of $T$, one can view the subtree (accessible term) $T_v$ as the
Merge of two subtrees $T_v=\fM(T_{v_1}, T_{v_2})$ where 
$T_{v_i}$ are the two subtrees with roots at the vertices below $v$. The
same listed properties then ensures that we are mapping $v$ to a leaf $\ell(v)\in L(T_v)$
which agrees with either the head of $T_{v_1}$ or the head of $T_{v_2}$. Moreover,
this also ensures that the property of Definition~\ref{headfunc} is satisfied by the
function $h_T: V^o(T)\to L(T)$ obtained in this way. Indeed, suppose given 
$T_v \subseteq T_w$. If the function determined by the three properties above
satisfies $h_T(w)\in L(T_v)$ then it means that it is $T_v$ that projects, according
to the definition of \cite{Chomsky-bare}, hence $h_T(w)=h_T(v)$. This shows that
the definition of head in \S 4 of \cite{Chomsky-bare}  implies the one given in
Definition~\ref{headfunc}.

Conversely, suppose that we have an abstract head function as in 
Definition~\ref{headfunc}. We can see that it has to satisfy the three
properties of \S 4 of \cite{Chomsky-bare} in the following way. The
first property is immediate from the fact that $h_T: V^o(T)\to L(T)$ is a
function, which means that, if we consider any subtree of $T$ consisting
of two leaves with a common vertex above them, that is $T_v=\fM(\alpha, \beta)$,
then $h_T(v)$ has to be either $\alpha$ or $\beta$. To see that the second and third properties
also hold, consider first the full tree $T$. Since this is a binary rooted tree it is uniquely
describable in the form $T=\fM(T_1,T_1)$ for two other binary rooted trees $T_1,T_2$. 
Since the function $h_T$ takes values in the set $L(T)=L(T_1)\sqcup L(T_2)$, the head
$h(T)$ is in either $L(T_1)$ or in $L(T_2)$. Suppose it is in $L(T_1)$. The other case is
analogous. Then by Definition~\ref{headfunc} we have $h(T)=h(T_1)$, where we write
$h(T_v):=h_T(v)$. Continuing in the same way for each successive nodes, with the
corresponding unique decompositions $T_v=\fM(T_{v,1}, T_{v,2})$, we obtain, for
each internal vertex a path to a leaf, which follows the head, and provides the 
``head from which it ultimately projects" as desired in the second property
listed above, while at each step the third property holds. 
\end{proof} 

\smallskip

\begin{rem}\label{remheadsT}{\rm
There are two important remarks to make regarding the two equivalent
formulations of Definition~\ref{headfunc} and Lemma~\ref{headprojection}.
As we discussed in \S 4.2 of \cite{MBC}, a consistent definition (compatible with
the Merge operation) of a head function $h$ does not extend to the entire $\cS\cO$
but is defined on some domain ${\rm Dom}(h) \subset \cS\cO$, so 
the identification between the descriptions of Definition~\ref{headfunc} and of
\S 4 of \cite{Chomsky-bare} also holds on such domain.
Moreover, as we also discussed in \S 4.2 of \cite{MBC}, on a given $T\in \cS\cO$ 
there are $2^{\# V^o(T)}$ choices of a head function (which are in bijective
correspondence with the choices of a planar structure for $T$). This is why
we are saying above that, on a given $T$, there are more abstract head
functions than just the one that corresponds to the syntactic head 
(when the latter is well defined).  This does not matter as for most of the
arguments we are using that involve a head function $h$, the formal property
of Definition~\ref{headfunc} is the only characterization required. In terms of
explicit linguistics examples, one can think of the usual syntactic head as
presented in \cite{Chomsky-bare}.
}\end{rem}


\smallskip

As shown in   \cite{MBC}, it follows directly from the definition
that assigning a head function $h_T$ to a tree $T$ is equivalent to
assigning a planar embedding $\pi_{h_T}$ (every head function determines
a planar embedding and conversely). 

\smallskip

Thus, we can equivalently think of an assignment
\begin{equation}\label{headTs}
 h: T \mapsto h_T 
\end{equation} 
of a head function to every tree $T \in \fT_{\cS\cO_0}$
as a function
\begin{equation}\label{headDom}
 h : {\rm Dom}(h)\subset \fT_{\cS\cO_0} \to \Sigma^*[\cS\cO_0] 
\end{equation} 
to the set of all finite ordered sequences, of arbitrary length, 
in the alphabet $\cS\cO_0$, given by
$$ h(T) = L(T^{\pi_{h_T}}) \, , $$
where $T^{\pi_{h_T}}$ is the planar embedding of $T$ determined by
the head function, and $L(T^{\pi_{h_T}})$ is its ordered set of leaves.
Since it is equivalent to describe $h(T)$ as the ordered set $L(T^{\pi_{h_T}})$
or as a single leaf (the head) in $L(T)$, we will switch between these
two descriptions without changing the notation. 

\smallskip

 We have shown in \cite{MBC} that one does not have a 
 well-defined head function on the entire $\fT_{\cS\cO_0}$,
 hence we write here $h$ as defined on some domain 
${\rm Dom}(h)\subset \fT_{\cS\cO_0}$. The obstacle
to the extension of a head function to the entire set
$\fT_{\cS\cO_0}$ derives from the well-known issue of  {\em exocentric}
constructions, e.g., the traditional division of sentences into
Subjects and Predicates, namely cases of syntactic objects 
$T\in \cS\cO=\fT_{\cS\cO_0}$ that are obtained as the result of
External Merge $T=\fM(T',T'')$ where even if a
head function is well defined on $T'$ and $T''$, there
is no good way of comparing $h(T')$ and $h(T'')$ to
decide which one should become the head of 
$T=\fM(T',T'')$.  Abstract heads are thus partially defined functions.

\smallskip

It is interesting to observe here that this fact makes abstract heads amenable
to treatment according to the renormalization model used in the theory of
computation, where the source of ``meaningless infinities"
arises from what lies outside of the domain
where a function is computable, \cite{Man1}, \cite{Man2}.
We will indeed use this approach to construct a very simple
illustrative model of our proposed view of the
syntax-semantics interface in \S \ref{HeadIdRenSec}.

\medskip

\subsection{Algebraic Renormalization: a short summary}\label{AlgRenSec}

The physical procedure of {\em renormalization} can be formulated in
algebraic terms (see \cite{CoKr}, \cite{CoMa}, \cite{EbFKr}, \cite{EbFM}) 
using Hopf algebras and Rota--Baxter algebras. In this formulation, the procedure
describes a very general form of Birkhoff factorization, which separates out
an initial (unrenormalized) mapping into two parts of a convolution product,
with one term describing the desirable ({\em meaningful}) part and one term describing 
the {\em meaningless} part that needs to be removed (divergences in the case
of Feynman integrals in quantum field theory). 

\smallskip

The mathematical setting that describes renormalization in physics, which 
we summarize here, may seem far-fetched as a model for linguistics,
but the point here is that mathematical structures exist as flexible templates
for the description of certain types of universal fundamental processes
in nature, which are likely to manifest themselves in similar mathematical
form in a variety of different contexts. 

\smallskip

The Hopf algebra datum $\cH=(\cV, \cdot, \Delta, S)$, a vector space with
compatible multiplication, comultiplication (with unit and counit) and antipode, 
takes care of describing the underlying combinatorial data and their
generative process. In the case of quantum field theory these are the
Feynman graphs with their subgraphs. The Feynman graphs of a given
quantum field theory can be described as a generative process in two
different ways: one in terms of graph grammars (see \cite{MarPort}),
which is similar to the older formal languages approach in generative
linguistics, another in terms of a Hopf algebra  (see
\cite{CoKr}, \cite{CoMa}, \cite{EbFKr}, \cite{EbFM}). The comparison
between these two generative descriptions of Feynman graphs shows
direct similarities with what happens in the case of syntax, with
the difference between the old formal languages
approach and the new Merge approach in generative linguistics,
where syntactic objects and the workspaces with the action of Merge
can also be described in terms of Hopf algebras, as in \cite{MCB}.

\smallskip

The Hopf algebra structure is central to the renormalization process
and the coproduct operation is the key part of the structure that
is responsible for implementing the renormalization procedure, as
we will recall below. The other algebraic datum, the Rota-Baxter algebra
$(\cR,R)$ represents what in physics is called a ``regularization scheme". 
This is the choice of a model space where the factorization into
meaningful and meaningless parts takes place. There is an
important conceptual difference between these two algebraic 
objects $\cH$ and $\cR$, in the sense that $\cH$ is essentially
intrinsic to the process while $\cR$ is an accessory choice,
and in principle many different regularization schemes can be adopted
to achieve the same desired renormalization. In terms of our
linguistic model, one should think of this choice of a 
regularization scheme $\cR$ as the choice of {\em some} model
of semantics. As in the case of regularization in physics, we
view the specifics of such a model as accessories to the
interface we are describing, while we view the role of the
syntactic structures encoded in $\cH$ as the essential part.
This again reflects the view of a primarily syntax-driven
interface between syntax and semantics. 

\smallskip

In our context, this reflects the fact that there are several approaches to
the construction of possible models of semantics, which
are, in our view, not entirely satisfactory and not entirely compatible.
However, we argue that 
this is not as serious an obstacle as it might first
appear, in the sense that this is very much the situation
also with regularization schemes in the physics of
renormalization (where one has dimensional, cutoff,
zeta function regularizations, etc.) and yet one can still extract a viable
procedure of assignment of meaningful physical values,
consistently across the choices of regularization.
We will argue that indeed, a viable model of
the interface between syntax and semantics rests
upon specific formulations of semantics only through some
very simple abstract properties that can be satisfied
within different models.

\smallskip

We have here briefly recalled the detailed definition of the Hopf algebra
structure in \cite{MBC}; for details we refer the readers to our
discussion there. Recall that the
datum $\cH=(\cV, \cdot, \Delta, S)$ is assumed to be
a commutative, associative, coassociative, graded, connected Hopf
algebra, but it is in general {\em not} cocommutative.
We will fix this to be $\cH=(\cV(\cF_{\cS\cO_0}), \sqcup, \Delta)$, with
$\cF_{\cS\cO_0}$ the set of finite binary rooted forests (with no
planar structure), and
with $\Delta$ as in \eqref{coprod}, the grading via the number of leaves, 
and with the unique inductively defined antipode $S$.

\smallskip

For the Rota-Baxter part of the structure, we can distinguish two
cases, the algebra and the semiring case. The algebra case is
the one that was orginally introduced in the physics setting. 

\begin{defn}\label{minusBRdef}{\rm 
A Rota-Baxter  algebra $(\cR,R)$ of weight $-1$ is a
commutative associative algebra $\cR$ together with a
linear operator $R: \cR \to \cR$ satisfying the identity
$$ R(a)R(b)=R(aR(b))+R(R(a)b) - R(ab) \, , $$
for all $a,b\in \cR$. }
\end{defn}

The prototype example (relevant
to physics) is a Laurent series with the operator $R$ of projection
onto their polar (divergent) part. 

\smallskip

The case of a semiring (where addition is no
longer invertible), more closely related to settings like the
theory of computation, was introduced in \cite{MarTe}. 

\begin{defn}\label{semiRBdef}{\rm 
A Rota-Baxter semiring of weight $+1$
is a semiring $\cR$ together with a Rota-Baxter operator $R$ of weight $+1$.
This is an additive (with respect to the semiring addition) map 
$R: \cR \to \cR$ satisfying
$$  R(a)\odot R(b)= R(a\odot R(b)) \boxdot R(R(a)\odot b) \boxdot R(a\odot b)\, , $$
with $(\boxdot, \odot)$ the semiring addition and multiplication operations.
A Rota-Baxter semiring of weight  $-1$ similarly satisfies the identity
$$  R(a)\odot R(b) \boxdot R(a\odot b) = R(a\odot R(b)) \boxdot R(R(a)\odot b)\, . $$}
\end{defn}

\smallskip

Note that since semiring addition is not invertible, in this case we cannot move the term
$R(a\odot b)$ to the other side of the identity. The purpose of the 
Rota-Baxter operator $R$ is to project onto the ``part of interest'' (for
example, divergencies in physics). We will discuss in \S \ref{SemSpaceSec}
how to adapt Rota-Baxter data of the form $(\cR,R)$ to semantic models.

\smallskip

\begin{defn}\label{charHR}{\rm 
A character of a commutative Hopf algebra $\cH$ with values in a
commutative algebra $\cR$ is a map
$$ \phi: \cH \to \cR $$
which is assumed to be a {\em morphism of algebras}, hence it
satisfies $\phi(xy)=\phi(x) \phi(y)$.
In the case where $\cR$ is a semiring, we will consider two cases
of semiring-valued characters
\begin{enumerate}
\item {\em Semiring maps:} 
$$ \phi: \cH^{semi} \to \cR $$
defined on a subdomain $\cH^{semi}$ of $\cH$ that is a 
commutative semiring, with $\phi$ a morphism of commutative semirings.
\item {\em Maps on cones:} assuming that $\cH$ is defined over the field $\R$, we consider maps
$$ \phi: \cH^{cone} \to \cR $$
where the subdomain $\cH^{cone}$ of $\cH$ is a cone, closed under convex linear combinations
and under multiplication in $\cH$, with $\phi$ compatible with convex combinations and with
products,  $\phi(xy)=\phi(x)\odot \phi(y)$, with $\odot$ the semiring product.
\end{enumerate} }
\end{defn}

\smallskip

In physics such datum $\phi: \cH \to \cR$ describes the Feynman
rules for computing Feynman integrals in an assigned
regularization scheme (given by the Rota-Baxter datum).
In our setting, the map $\phi: \cH \to \cR$ is some
map from syntactic objects to a semantic space, which
includes the possibility of meaninglessness, when a
consistent semantics cannot be assigned. By ``consistent" here
we mean that assignment of semantic values to larger hierarchical
structures has to be compatible with assignments to sub-structures:
this is exactly the same consistency requirement that is
used in the physics of renormalization and that determines
the required algebraic structure.  The
multiplicativity condition here just means that, in a
workspace containing many different syntactic
objects, the image of each of them in the
semantic model $\cR$ is independent of the others. 
Of course, when different syntactic objects are
assembled together by the action of Merge, these
different images need to be compared for consistency:
this is indeed the crucial part of the interpretive process, that
corresponds to the {\em compositionality} requirement, 
number (4) on our list of desired properties for the
syntax-semantics interface. 

\smallskip

\begin{rem}\label{phialg}{\rm 
It is important to stress the fact that a character
$\phi: \cH \to \cR$ is only a map of algebras: it does
not know anything about the fact that $\cH$ also
has a coproduct $\Delta$ and that $\cR$ also has
a Rota-Baxter operator $R$. In particular, the target
$\cR$ does not carry a coproduct operation and $\phi$
is {\em not} a morphism of Hopf algebras.}
\end{rem}

The observation made in Remark~\ref{phialg}
will play an important role in our
linguistic model. It is in fact closely related to
the statement we made at the beginning of this
paper: the computational structure of syntax
--which as we explained in \cite{MCB} depends on
the coproduct structure of $\cH$--does not 
require an analogous computational counterpart
in semantics. We will discuss this point in
more detail in the following sections, where we will
show that, in our model, the compositional properties of semantics
are entirely governed by the computational structure of
syntax, along with the topological nature of semantics 
(as a classifier of proximity relations). This is a very
strong statement on the relative roles of syntax
and semantics, presenting what can be viewed as a strong ``syntax-first" model. 
While several of the examples we present in this paper 
will be simplified mathematical models aimed at 
illustrating the fundamental algebraic properties, we will discuss 
at some length how the principle we state here
can be understood in the case of Pietroski's model
of semantics, that we compare with our framework in
\S \ref{PietSec}.

\smallskip

In fact, in the physics setting as well as in our linguistics model,
the interaction between the two additional data, $\Delta$ and $R$, is
used to implement {\em consistency} across substructures (our desired
property of compositionality).  This happens by recursively
constructing a factorization (over the grading of the Hopf algebra),
in the following way.

\begin{defn}\label{Birkhoff}{\rm
A {\em Birkhoff factorization} of a character $\phi: \cH \to \cR$ is
a decomposition  
\begin{equation}\label{BirkS}
 \phi = (\phi_- \circ S) \star \phi_+ 
\end{equation} 
with $S$ the antipode and $\star$ the convolution
product determined by the coproduct $\Delta$
$$  (\phi_1\star \phi_2) \, (x)= (\phi_1\otimes\phi_2) \, \Delta(x)\, . $$}
\end{defn}

One interprets one of the two terms $\phi_+$ as the
meaningful renormalized part and the other $\phi_-$ 
as the meaningless part that needs to be removed. 
The semiring case is similar.

\begin{defn}\label{semiBirkhoff}{\rm
A Birkhoff factorization of a semiring character
$\phi: \cH^{semi}\to \cR$ is a factorization of the form
$$ \phi_+=\phi_-\star\phi . $$ }
\end{defn}

A Birkhoff factorization as in Definition~\ref{Birkhoff} is constructed inductively
using $R$ and $\Delta$ as follows.

\begin{prop}\label{RBfact}{\rm (\cite{CoKr}, \cite{EbFKr})}
If $(\cR,R)$ is a Rota--Baxter algebra of weight $-1$ and
$\cH$ is a commutative graded connected Hopf algebra, with
$\phi: \cH \to \cR$ a character, there is (uniquely up to
normalization) a Birkhoff factorization of the form \eqref{BirkS}
obtained inductively (on the Hopf algebra degree) as 
\begin{equation}\label{AlgBfact}
 \phi_-(x)=-R ( \phi(x) +\sum \phi_-(x') \phi(x'') ) \ \ \text{ and } \ \ 
\phi_+(x) = (1-R) ( \phi(x) +\sum \phi_-(x') \phi(x'') ) 
\end{equation}
where $\Delta(x)=1\otimes x + x \otimes 1 + \sum x' \otimes x''$,
with the $x', x''$ of lower degree. 
The $\phi_\pm: \cH \to \cR_\pm$ are algebra homomorphisms
to the range of $R$ and $(1-R)$. These are subalgebras
(not just vector subspaces), because of the Rota--Baxter identity
satisfied by $R$.
\end{prop}

\smallskip

\begin{rem}\label{remBprep}{\rm
One usually refers to the
expression
\begin{equation}\label{Bprep}
 \tilde\phi(x):= \phi(x) +\sum \phi_-(x') \phi(x'') 
\end{equation} 
as the {\em Bogolyubov preparation} of $\phi$ and writes
$\phi_-=-R(\tilde\phi)$ and $\phi_+=(1-R)(\tilde\phi)$.}
\end{rem}

\smallskip

The case of semirings is similar. 

\begin{prop}\label{semiRBfact} {\rm (\cite{MarTe})}
If $(\cR,R)$ a Rota--Baxter semiring of
weight $+1$ and $\cH$ is a commutative graded 
connected Hopf algebra with a semiring character
$\phi: \cH^{semi}\to \cR$, where $\cH^{semi}$
has an induced grading, one has a factorization
\begin{equation}\label{SemiBfact}
\begin{array}{rcl}
 \phi_-(x)=&  R(\tilde\phi(x))  & =R(\phi(x) \boxdot \phi_-(x')\odot \phi(x'')) \\[3mm]
 \phi_+(x) =& (\phi_- \star \phi)(x) & =\phi(x) \boxdot \phi_-(x) \boxdot \phi_-(x')\odot \phi(x'') \\[3mm]
 =& \phi_- \boxdot \tilde\phi \, ,&  \end{array}
 \end{equation}
where the $\phi_\pm$ are also multiplicative with respect to
the semiring product, $\phi_\pm(xy)=\phi_\pm(x)\odot \phi_\pm(y)$.
\end{prop}

\begin{rem}\label{semiRBminus}{\rm 
In the case with $(\cR,R)$ a Rota--Baxter semiring of
weight $-1$,  one still obtains a Birkhoff factorization of the 
form \eqref{SemiBfact}. In this case both $\phi_\pm$ still 
satisfy the multiplicative property if $R$ has the additional
property that 
\begin{equation}\label{R1eq}
R(x\odot y)\boxdot R(x)\odot R(y)=R(x)\odot R(y),
\end{equation} 
see \cite{MarTe}. This happens 
for instance if $R(x+y)\leq R(x)+R(y)$ in $\cR=(\R \cup \{ -\infty \}, \max, +)$. }
\end{rem}

\medskip
\subsection{Semantic spaces}\label{SemSpaceSec} 

If we follow the idea described above of a syntax-semantics interface
modeled after the formalism of algebraic renormalization in physics,
and we encode the syntactic side of the interface in terms of the
Hopf algebra model of Merge and Minimalism as we described in \cite{MCB}, we
then need a general description of what type of mathematical
objects should be feasible on the semantic side, so that a 
Birkhoff factorization mechanism as above can be used to implement
the assignment of semantic values to syntactic objects. As we will be
illustrating in a series of different examples in the following sections,
Birkhoff factorizations of the form \eqref{AlgBfact} or \eqref{SemiBfact}
will serve the purpose, in our model, of checking and implementing 
consistency of semantic assignments throughout all substructures of given 
syntactic hierarchical structures, through the use of a combination of
values on substructures provided by the Bogolyubov preparation \eqref{Bprep}
and the use of a Rota--Baxter operator as a way of checking the possible
failures of consistency across substructures.  Since the target of our
map from the Hopf algebra of syntax has to be a model of a ``semantic space",
again we proceed first by trying to identify certain key formal properties that we
would like to have for such ``semantic spaces" (thought of in similar terms to
the regularization schemes in the physics of renormalization). 

\smallskip

We first discuss in \S \ref{NeuroSec} some analogies of the type of model that we have
in mind, originating in neuroscience. The first is a neuroscience model 
that is somewhat controversial (and that will play no direct role in this paper,
except in the form of an analogy) while the second is a well established
result on neural codes and homotopy types. 

\subsubsection{Neuroscience data and syntax-semantics interface models}\label{NeuroSec}

Neuroscience data that  study the human brain's handling of syntax and
semantics in response to auditory or other signals (see \cite{Friede} and
\cite{BFCB}, \cite{FCMMB}), e.g., in experiments measuring ERP
(event-related brain potentials) waveform components and in functional
magnetic resonance imaging studies, display rapid recognition of
syntactic violations and activation in the middle and posterior
superior temporal gyrus for both semantic and syntactic violations.
In contrast,the anterior superior temporal gyrus and the frontal operculus
are activated by syntactic violations. Syntax and semantics has been
claimed to be disentangled in such experiments, by using artificial grammars:
syntactic errors in simple cases that do not involve significant
hierarchical syntactic structures appear related to activation of the
frontal operculum, while the type of syntactic structure building that
is modeled by the Merge operation appears related to activation in
the most ventral anterior portion of the BA 44 part of Broca's
area.

Additional semantic information shows involvement of other areas of
the brain, in particular the BA 45 area.  This suggests a possible
``syntax-first" model of language processing in the brain, with an
initial structure building process taking place at the syntactic level
and an interface with semantics through the connectome involving the
frontal operculus, BA 44, and BA 45. It should be noted though that
this proposal regarding brain regions implicated in syntax and
semantics has been strongly contested, for example according to the
results of \cite{Fedor}, that dispute the disentanglement and
partitioning into areas of the syntax-first proposal of \cite{Friede}.
We only mention this proposal here as an analogy that can help
illustrate some for our modeling of the syntax-semantics interface
according to the list of properties outlined in \S
\ref{RequireListSec} above.  While we understand that this view is
considered controversial by some, it does furnish a suggestive analogy
for some of the basic geometric requirements that we will be assuming
about the semantic side of the interface we wish to model.

\smallskip

Another insight from neuroscience that we would like to carry into our
modeling is the idea of information encoded via covering spaces and
homotopy types. This is well known in the setting of visual stimuli when
hippocampal place cells, that fire in response to a restricted area of 
the spatial stimulus, are analyzed to address the question of how neuron
spiking activity encodes and relays information about the stimulus space.
In such settings one can show that patterns of neuron firing and their receptive
fields determine a covering that (under a convexity hypothesis) can reconstruct
the stimulus space up to homotopy, see \cite{Curto} and the mathematical
survey in \cite{Man3}. While this picture is specific to visual
stimuli, an important idea that can be extracted from it is the role of covering spaces
(in particular covering spaces associated with binary codes) 
in encoding proximity relations, and the role of convexity in such
covering spaces. We will incorporate these ideas in a general
basic picture of a notion of ``semantic spaces" that can be compatible 
with how semantic information may effectively be stored in human brains.
It was already observed in \cite{ManMar} that this structure should be
part of modeling of semantic spaces. 

\subsubsection{Formal properties of semantic spaces}

As basic structure for an adequate parameterizing space
for semantics, we focus on two compositional aspects:
measuring degrees of proximity, and a notion of 
agreement/disagreement. We argue that, at the least,
semantics should be able to compare different semantic
values (points in a semantic space) in terms of their
level of agreement/disagreement, and to form new
semantic points by some form of combination/interpolation
of previously achieved ones. The type of ``interpolation"
considered may vary with specific models, but in general
we can think of it in the following related forms:
\begin{itemize}
\item geodesic paths 
\item convex combinations
\item overlapping open neighborhoods.
\end{itemize}
A typical example that would combine these
forms of combination/interpolation is provided 
by a geodesically convex Riemannian manifold.
Another aspect to take into 
consideration is the idea that, for instance, one can usually
associate with a lexical item a collection of different ``semes,''
hence points in a ``semantic space." In other words, the target of
a map from lexical items and syntactic objects should
allow for such ``lists of semes."

\smallskip

A very simple mathematical structure where notions
of agreement/disagreement, proximity, and lists are
simultaneously present, and combination operations
are possible is of course a vector space structure,
and for this reason it happens that frequently used 
elementary computational models of 
semantics tend to be based on vectors and vector
space operations. More sophisticated geometric
models of semantics based on spaces with
properties of convexity, local coordinates 
representing semic axes, and realizations of
notions of similarity, were presented for example
in \cite{Garden1}, \cite{Garden2}. Such geometric
models also incorporate the possibility of covering
spaces, intersections of open sets, and homotopy, as a way
of realizing a ``meeting of minds" model of \cite{Garden1}, 
\cite{Garden2}, where
different observers may produce somewhat different
sets of semantic associations with the same
linguistic items (see the corresponding discussion in \cite{ManMar}).

\smallskip

Additional structure can be incorporated, if one desires for example
to include a notion akin to that of ``independent events."  This can
be achieved by working with spaces that have also a product operation,
such as algebras, rings, or semirings, or that can be mapped to a
space with this kind of structure, where such independence hypotheses
can be tested. Thus, for example, elementary operations like
assignments of truth values, or of probabilities/likelihood estimates,
fall within this category, and are usually performed by mapping to
some (semi)-ring structure. More generally, rings, algebras, and
semirings can be seen as repositories for comparisons with specific
test hypotheses, probing agreement/disagreement, or likelihood, of
representation along a chosen semic axis.  We will discuss a few such
examples in the following sections.

\smallskip
\subsubsection{Concept spaces in and outside of language}

In this viewpoint, the type of fundamental structure that we associate with semantic
spaces is not strictly dependent on their role in
language. Indeed the idea of extracting classifications from
certain kinds of sensory data and associating with them some
representation where proximity and difference can be evaluated
is common to other cognitive processes. Conceptual spaces
associated with vision are intensely studied in the context of both
neuroscience and artificial intelligence, and in that case certainly the
most relevant structures involved are topological in nature (see for
example the theory of perceptual manifolds, \cite{Chung}). This
suggests that it is possible to consider a model where the conceptual
spaces that syntax interfaces with in language would be of
an essentially similar nature as other conceptual spaces, 
and not necessarily endowed with additional structure specific 
to their role in language, with all the required structure that
is of a specifically linguistic nature being provided by syntax.

\smallskip

Formulated in such terms, this leads to a view of semantics
that is essentially external to language and becomes a
part of linguistics through the presence of a map {\em from} syntax. 
A more nuanced position, as we will illustrate in specific examples
that follow, endows the semantic conceptual spaces
with just enough additional structure extending the topological 
notion of proximity, to make the mapping from syntax sufficiently
robust to induce a compositional structure on semantics,
modeled on the Merge operation in syntax. 

\smallskip

For ease of computation,
we will be using examples where such additional structure, aimed
at quantifying proximity relations, consists of metrics with convexity properties 
and/or evaluations in semirings. This viewpoint will bring us
close to Pietroski's model of compositional semantics, \cite{Pietroski},
where a compositional structure in semantics is modeled on the Merge operation
of syntax. One significant difference in our setting, though, is that
we do not need to posit a separate compositional/computational operation 
on semantics itself (why should a Merge-type operation develop twice, 
once for syntax and once for semantics?).  In our model, the compositionality of
semantics is directly induced by the computational structure of
syntax through the Birkhoff factorization mechanism 
described above. This will constitute the key to our interface model. 

\smallskip

Of course one should allow for enough structure on the
semantics side to incorporate the possibility of 
conjunctions of predicates, as well as a way of distinguishing the
possibilities of mapping to conjunctions, predicate saturation, 
existential closure. We will discuss more of this in \S \ref{PietSec}.
The main point we want to stress here is that one does not need
two parallel generative computational processes, one on the
side of syntax and one on the side of semantics (as would be
the case if we were to assume that our maps $\phi: \cH \to \cR$
are Hopf algebra homomorphisms, see Remark~\ref{phialg}).  
What one has instead is a map between two different kinds 
of mathematical structures, only one of which (syntax) is 
constructed by a recursive generative process.

\medskip
\section{Syntax-Semantics Interface as Renormalization: Toy Models}\label{HeadIdRenSec}

\subsection{A simple toy model: Head-driven syntax-semantics interface}\label{HeadIdRenSec1}

We discuss, as a first illustrative example, a very simple-minded toy
model of the type of syntax-semantics interface we are proposing. The
examples we present in this section are intentionally oversimplified
in order to more easily illustrate the main formal aspects.

\smallskip

Consider the semiring $(\R \cup \{ -\infty \}, \max, +)$
where the addition is the maximum (with $-\infty$ as the unit of addition), and with product the
usual sum of real numbers (with the rule that $-\infty + x =-\infty$), with $0$ as the unit of the
semiring multiplication $+$.

\begin{lem}\label{ReLUlem}
The ReLU operator $R: x\mapsto x^+=\max\{ x, 0 \}$ is a Rota--Baxter
operator of weight $+1$ on $\cR=(\R \cup \{ -\infty \}, \max, +)$.
\end{lem}

\proof
To see this, we need to check that the Rota--Baxter relation
$$ x^+ + y^+ =\max\{ (x^+ +y)^+ , (x+y^+)^+, (x+y)^+ \} $$
is verified for all $x,y\in \R \cup \{ -\infty \}$. The following table
shows that this is indeed the case

\medskip
\begin{center}
\begin{tabular}{|c|c|c|c|c|}
\hline 
& $x \leq 0$, $y\leq 0$ & $x\geq 0$, $y\leq 0$ & $x\leq 0$, $y\geq 0$ & $x\geq 0$, $y\geq 0$ \\
\hline 
$x^+ + y^+$ & $0$ & $x$ & $y$ & $x+y$ \\
\hline
$(x^+ +y)^+$ & $0$ & $\left\{ \begin{array}{ll} x+y & x+y\geq 0 \\ 0 & x+y\leq 0 \end{array}\right. $ &
$y$ & $x+y$ \\
\hline
$(x+y^+)^+$ & $0$ & $x$ & $\left\{ \begin{array}{ll} x+y & x+y\geq 0 \\ 0 & x+y\leq 0 \end{array}\right. $ & $x+y$ \\
\hline
$(x+y)^+$ & $0$ & $\left\{ \begin{array}{ll} x+y & x+y\geq 0 \\ 0 & x+y\leq 0 \end{array}\right. $ & $\left\{ \begin{array}{ll} x+y & x+y\geq 0 \\ 0 & x+y\leq 0 \end{array}\right. $ & $x+y$ \\
\hline
$\max$ & $0$ & $x$ & $y$ & $x+y$ \\
\hline
\end{tabular}
\end{center}
\bigskip
\endproof

\begin{rem}\label{idRB}{\rm 
The identity operator $R={\rm id}$ on the same semiring $(\R \cup \{ -\infty \}, \max, +)$
is a Rota--Baxter operator of weight $-1$. }
\end{rem}

\begin{defn}\label{headmap1}{\rm
Consider a semantic space $\cS$ with a map $s: \cS\cO_0 \to \cS$
that assigns a meaning (a point in $\cS$) to the lexical items and the syntactic features in $\cS\cO_0$. 
Given a tree $T\in \fT_{\cS\cO_0}$ and a leaf $\ell\in L(T)$, we write
$\lambda(\ell)\in \cS\cO_0$ for the label (lexical item or syntactic feature) assigned to that leaf.
Given a head function $h$, defined on a domain ${\rm Dom}(h)\subset \fT_{\cS\cO_0}$,
we obtain a map
$$ s\circ h: {\rm Dom}(h)\subset \fT_{\cS\cO_0} \to \cS\, , \ \ \  T\mapsto s(\lambda(h(T))\, , $$
where $h(T)\in L(T)$ is the head.}
\end{defn}

We now assume that the semantic space $\cS$ has {\em probes}, given by
functions $\Upsilon: \cS\to \R$, that check the degree of agreement or 
disagreement with some particular semantic hypothesis. We assume that,
for an $s\in \cS$, a value $\Upsilon(s)< 0$ means that there is 
disagreement between the semantic object $s$ and the semantic hypothesis $\Upsilon$,
while a value $\Upsilon(s)>0$ signifies agreement, with the magnitude $|\Upsilon(s)|$
signifying the amount of agreement or disagreement. A value $\Upsilon(s) =0$ signifies
indifference. 

\begin{ex}\label{vecprobe}{\rm
In the case of the familiar vector space model of semantics, such a probe can be obtained by
taking the inner product with a specified hypothesis-vector,
$$ \Upsilon(s)=\langle s, v_\Upsilon \rangle $$
where the semantic hypothesis being tested is semantic proximity to a chosen vector $v_\Upsilon$.
}\end{ex}

\begin{lem}\label{phiUpsilon}
Suppose given a semantic space $\cS$, a probe $\Upsilon: \cS\to \R$, a map
$s: \cS\cO_0 \to \cS$ assigning semantic values to lexical items and syntactic features,
and a head function $h$ defined on a domain ${\rm Dom}(h)\subset \fT_{\cS\cO_0}$.
Let $\cV(\fF_{\cS\cO_0})^{semi}\subset \cV(\fF_{\cS\cO_0})$ denote the semiring 
of linear combinations $\sum_i c_i F_i$ with $c_i\geq 0$. Then the data $(\Upsilon, s,h)$
determine a semiring homomorphism
$$ \phi_{\Upsilon, s,h}: \cV(\fF_{\cS\cO_0})^{semi} \to \R\cup\{ -\infty \}\, . $$
\end{lem}

\proof The data $(\Upsilon, s,h)$ determine a map
$$ \Upsilon_{s,h}: \fT_{\cS\cO_0} \to \R \cup \{ -\infty \} $$
\begin{equation}\label{sheadprobe}
 \Upsilon_{s,h} : T \mapsto \left\{ \begin{array}{ll} 
\Upsilon(s(\lambda(h(T))) & T\in {\rm Dom}(h) \\
-\infty & T\notin {\rm Dom}(h) \, .
\end{array} \right. 
\end{equation}
The value $-\infty$ in the case of $T\notin {\rm Dom}(h)$ here represents the case
where the comparison with the hypothesis in the probe cannot be performed due to
the lack of a well-defined head in the tree $T$.
This map can be extended from trees to a forest by setting
$$ \phi_{\Upsilon, s,h}: \fF_{\cS\cO_0}  \to  \R \cup \{ -\infty \} \, ,  \ \ \ 
\phi_{\Upsilon, s,h}(F)=\sum_a \Upsilon_{s,h}(T_a)\, , \ \ \text{ for } F=\sqcup_a T_a \, . $$
We can further extend this map to the subdomain $\cV(\fF_{\cS\cO_0})^{semi}\subset \cV(\fF_{\cS\cO_0})$ 
by setting 
$$ \phi_{\Upsilon, s,h}(\sum_i c_i F_i) = \boxdot_i \phi_{\Upsilon, s,h}(F_i) \odot \log(c_i) =
\max_i \{ \phi_{\Upsilon, s,h}(F_i) +\log(c_i) \} \, . $$
\endproof

The extension to linear combinations is needed for
formal consistency. In the case of sums where all the coefficients are $1$ 
the corresponding $\log(c_i)$ term vanishes. 

\smallskip

\begin{rem}\label{toyrem}{\rm
One reason why this simple-minded toy model is too oversimplified is
that the assignment $\phi_{\Upsilon, s,h}$ only follows the semantic
value of the head of the tree, hence it only uses the semantic values
already attached to the leaves of the tree.  However, in general we want to obtain
new points in semantic space, as the lexical items attached to the
leaves are related and combined inside more elaborate syntactic objects. 
We will show in \S \ref{HeadIdRenSec2}
how to correct this problem and obtain more refined toy models. }
\end{rem}

\smallskip

To see how our interface model works in this simplified example,
we first perform the Birkhoff factorization with respect to the Rota--Baxter
operator $R={\rm id}$ of weight $-1$ and then with respect to the 
ReLU Rota--Baxter operator $R=(\cdot)^+$ of weight $+1$.

\begin{lem}\label{toyRBid}
For a semiring homomorphism $\phi : \cV(\fF_{\cS\cO_0})^{semi} \to \cR=(\R\cup\{\infty \}, 
\max, +)$, where the values $\phi(T)$ signify agreement/disagreement between
a semantic value assigned to the tree $T$ and a semantic probe, 
the Birkhoff factorization with $R={\rm id}$ has the effect of checking, for a given
syntactic object $T\in \fT_{\cS\cO_0}$, and all chains of subforests
$F_{\underline{v}_N}\subset F_{\underline{v}_{N-1}}\subset \cdots \subset F_{\underline{v}_1} \subset T$,
when the {\em combined} agreement with the semantic probe of the parts
$$ \phi(F_{\underline{v}_N})+\phi(F_{\underline{v}_{N-1}}/F_{\underline{v}_N})
+\cdots +\phi(T/F_{\underline{v}_1} ) $$
is greatest, and is at least as good as the overall agreement $\phi(T)$.
 \end{lem}

\proof The Birkhoff factorization with respect to the Rota--Baxter
operator $R={\rm id}$ of weight $-1$ simply gives $\phi_- =\tilde\phi$,
so that we have
$$ \phi_- (T)=\tilde\phi(T)=\max\{ \phi(T), 
\sum_{i=1}^N \phi(F_{\underline{v}_i})+ \phi(F_{\underline{v}_{i-1}}/F_{\underline{v}_i})\}\, $$
where $F_{\underline{v}_N} \subset F_{\underline{v}_{N-1}}\subset \cdots F_{\underline{v}_0} =T$
is a nested sequence of subforests (collections of accessible terms, and the maximum is taken
over all such sequences of arbitrary length $N\geq 1$. 
\endproof

\begin{cor}\label{toycor1}
For the case of $\phi=\phi_{\Upsilon, s,h}$ the Birkhoff factorization 
as in Lemma~\ref{toyRBid} has the effect of checking, for a given
syntactic object $T\in \fT_{\cS\cO_0}$, and all chains of subtrees (subforests)
$T_{v_N}\subset T_{v_{N-1}}\subset \cdots \subset T_{v_1} \subset T$,
when the {\em combined} agreement with the semantic probe is maximal. 
If $\phi_{\Upsilon, s,h}(T)>0$, this maximum is bounded below by the 
sum of values on the chain of subtrees with $h(T_{v_i})=h(T)$ 
which is $N\cdot \phi_{\Upsilon, s,h}(T)$ with $N$ the length of the path
from the root of $T$ to the leaf $h(T)$. If $\phi_{\Upsilon, s,h}(T)<0$, on the
other hand, the maximum is bounded below by the $\phi_{\Upsilon, s,h}(T)+ M\cdot 
\phi_{\Upsilon, s,h}(T_v)$ where $T_v$ is an accessible term with 
$\phi_{\Upsilon, s,h}(T_v)>0$ and $M$ is the length of the path from $v$ to the leaf $h(T_v)$.
\end{cor}

\proof
Observe that we have $\phi_{\Upsilon, s,h}(T/T_v)=\phi_{\Upsilon, s,h}(T)$,
since if $h(T)\notin T_v$ then quotienting the subtree $T_v$ will not affect 
the head, and if $h(T)\in T_v$ then $h(T)=h(T_v)$, by the properties of head
functions, and we label the leaf of $T/T_v$ with a {\em trace} carrying the semantic value
that was assigned to the leaf $h(T_v)$, and similarly for the case of $T/F_{\underline{v}}$.
Note that here we take quotients as contractions of each component of the
subforest, as discussed in \S \ref{QuotSec}.

For simplicity we write out in full only the case 
where each $F_{\underline{v}_k}$ consists of a single subtree $T_{v_k}$
as the more general case of forests is analogous. In this case we are computing
$$ \phi_{\Upsilon, s,h,-}(T)= \max\{ \phi_{\Upsilon, s,h}(T), \phi_{\Upsilon, s,h}(T)
+\phi_{\Upsilon, s,h}(T_1),\cdots ,
\phi_{\Upsilon, s,h}(T)+\phi_{\Upsilon, s,h}(T_1)+\cdots 
 +\phi_{\Upsilon, s,h}(T_N) \} $$
where $N$ is the longest chain of nested accessible terms in $T$. The maximum
is achieved at sequences $T_k \subset \cdots \subset T_1\subset T$ where
all $\phi_{\Upsilon, s,h}(T_i)>0$ and as large as possible, that is, at the 
chains of nested accessible terms that achieve the {\em combined} maximal 
agreement with the probe. 

For example, for a chain of length $N=1$, that is, a single accessible term
$T_v \subset T$, we are comparing $\phi_{\Upsilon, s,h}(T)$ and
$\phi_{\Upsilon, s,h}(T)+\phi_{\Upsilon, s,h}(T_v)$, hence we are checking
whether $\phi_{\Upsilon, s,h}(T_v)>0$ or $\phi_{\Upsilon, s,h}(T_v)<0$,
that is, whether individual accessible terms of $T$ have heads $h(T_v)$
that semantically agree with the probe $\Upsilon$ of not. Clearly,
among all subtrees $T_v$ one can always find some for which
$\phi_{\Upsilon, s,h}(T)+\phi_{\Upsilon, s,h}(T_v)> \phi_{\Upsilon, s,h}(T)$,
namely subtrees for which $h(T_v)=h(T)$. The case of longer chains is analogous.

It is then clear that a lower bound in the case $\phi_{\Upsilon, s,h}(T)>0$ is
obtained by following the path from the root of $T$ to the head $h(T)$, while
in the case $\phi_{\Upsilon, s,h}(T)<0$ one maximizes over collections of accessible
terms with positive values $\phi_{\Upsilon, s,h}(T_{v_i})>0$ and one such
collection is obtained by following the head of any $T_v$ that has $\phi_{\Upsilon, s,h}(T_v)>0$. 
\endproof

\smallskip

We compare this to taking the Birkhoff factorization with respect to the ReLU 
Rota--Baxter operator $R(x)=x^+=\max\{ x, 0 \}$ of weight $+1$. This shows
that using different Rota--Baxter structures on the target semiring corresponds
to performing different tests of semantic compositionality. 

\begin{lem}\label{toyRB2}
For the semiring homomorphism $\phi_{\Upsilon, s,h}: \cV(\fF_{\cS\cO_0})^{semi} \to \cR=(\R\cup\{\infty \}, 
\max, +)$, consider the Birkhoff factorization with respect to the ReLU 
Rota--Baxter operator $R(x)=x^+=\max\{ x, 0 \}$ of weight $+1$. In this case,
the value of $\phi_{\Upsilon, s,h,-} (T)$ is computed as a maximum value
$\phi_{\Upsilon, s,h}(F_{\underline{v}_N})+\phi_{\Upsilon, s,h}(F_{\underline{v}_{N-1}})+\cdots +\phi_{\Upsilon, s,h}(F_{\underline{v}_1})+\phi_{\Upsilon, s,h}(T)$, over all nested sequences with the property that
all $\phi_{\Upsilon, s,h}(F_{\underline{v}_i})>0$ and, in the case where $\phi_{\Upsilon, s,h}(T)<0$, with
$\sum_i \phi_{\Upsilon, s,h}(F_{\underline{v}_i})> |\phi_{\Upsilon, s,h}(T)|$. The maximum
computing  $\phi_{\Upsilon, s,h,-} (T)$ is bounded below by
$N \phi_{\Upsilon, s,h}(T)$, with $N$ the length of the path from the root of $T$ to the leaf $h(T)$,
in the case with $\phi_{\Upsilon, s,h}(T)>0$ and by $\phi_{\Upsilon, s,h}(T)+ M\cdot 
\phi_{\Upsilon, s,h}(T_v)$ where $T_v$ is any accessible term 
with $\phi_{\Upsilon, s,h}(T_v)> |\phi_{\Upsilon, s,h}(T)|$ 
and $M$ is the length of the path from $v$ to the leaf $h(T_v)$, when $\phi_{\Upsilon, s,h}(T)<0$.
\end{lem}

\proof
We obtain in this case
$$ \phi_{\Upsilon, s,h,-} (T) =\max \{ \phi_{\Upsilon, s,h}(T), (\cdots ( \phi_{\Upsilon, s,h}(F_{\underline{v}_N})^+ 
+\cdots + \phi_{\Upsilon, s,h}(F_{\underline{v}_{i-1}}/F_{\underline{v}_i}))^+  + \cdots 
+ \phi_{\Upsilon, s,h}(T/F_{\underline{v}_0}))^+ \}^+ \, ,  $$
over all nested sequences of subforests of arbitrary length $N\geq 1$ as above. By the
same argument as in Lemma~\ref{toyRBid} about heads of subtrees $T_v$ and quotient
trees $T/T_v$, in the case of chains of subtrees $T_{v_N}\subset T_{v_{N-1}}\subset \cdots
\subset T_{v_1}\subset T$, this gives
$$ (\cdots (( \phi_{\Upsilon, s,h}(T_{v_N})^+ +\phi_{\Upsilon, s,h}(T_{v_{N-1}}))^+ \cdots 
+ \phi_{\Upsilon, s,h}(T_{v_1}) )^++\phi_{\Upsilon, s,h}(T))^+\, , $$
and similarly for forests (with sums over the component trees), and then ReLU is applied
to the maximum taken over all these sums. 

For example, for a chain of length $N=1$, one compares $\phi_{\Upsilon, s,h}(T)$ with $\phi_{\Upsilon, s,h}(T)+\phi_{\Upsilon, s,h}(T_1)$, so that
$\max\{ \phi_{\Upsilon, s,h}(T), (\phi_{\Upsilon, s,h}(T)+\phi_{\Upsilon, s,h}(T_1)^+)^+ \}^+$ has value $\phi_{\Upsilon, s,h}(T)$ if $\phi_{\Upsilon, s,h}(T)>0$ and $\phi_{\Upsilon, s,h}(T_1)<0$,
value $\phi_{\Upsilon, s,h}(T)+\phi_{\Upsilon, s,h}(T_1)$ if $\phi_{\Upsilon, s,h}(T)>0$ and $\phi_{\Upsilon, s,h}(T_1)>0$, or if $\phi_{\Upsilon, s,h}(T)<0$ and $\phi_{\Upsilon, s,h}(T_1)>0$
with $\phi_{\Upsilon, s,h}(T)+\phi_{\Upsilon, s,h}(T_1)>0$, and value $0$ if $\phi_{\Upsilon, s,h}(T)<0$  and $\phi_{\Upsilon, s,h}(T_1)<0$, or if
$\phi_{\Upsilon, s,h}(T)<0$ and $\phi_{\Upsilon, s,h}(T_1)>0$ with $\phi_{\Upsilon, s,h}(T)+\phi_{\Upsilon, s,h}(T_1)<0$. 

Thus, we see that, when $\phi_{\Upsilon, s,h}(T)>0$, the value $\phi_{\Upsilon, s,h,-} (T)$ is bounded below
by $N \phi_{\Upsilon, s,h}(T)$, where $N$ is the length of the path from the root of $T$ to the
leaf $h(T)$, as in Corollary~\ref{toycor1}. However, when $\phi_{\Upsilon, s,h}(T)<0$ the 
Birkhoff factorization with respect to the ReLU gives a more refined test than
the Birkhoff factorization with respect to $R={\rm id}$ of Lemma~\ref{toyRBid}
and Corollary~\ref{toycor1}. Indeed, in this case we not only search over nested
sequences with $\phi_{\Upsilon, s,h}(T_{v_N}) +\phi_{\Upsilon, s,h}(T_{v_{N-1}}) \cdots + \phi_{\Upsilon, s,h}(T_{v_1}) >0$
but also we further require that individual terms are positive and that 
$\phi_{\Upsilon, s,h}(T_{v_N}) +\phi_{\Upsilon, s,h}(T_{v_{N-1}}) \cdots + \phi_{\Upsilon, s,h}(T_{v_1}) > | \phi_{\Upsilon, s,h}(T) |$
because of applying ReLU to the result of the sum.
In particular, one obtains such a lower bound by following the head
of any accessible term $T_v$ with $\phi_{\Upsilon, s,h}(T_v)> |\phi_{\Upsilon, s,h}(T)|$
as stated. 
\endproof

A case where $\phi_{\Upsilon, s,h}(T)<0$ with the maximum realized by a
sequence of postive terms with $\phi_{\Upsilon, s,h}(T_{v_N}) +\phi_{\Upsilon, s,h}(T_{v_{N-1}}) \cdots + \phi_{\Upsilon, s,h}(T_{v_1}) > | \phi_{\Upsilon, s,h}(T) |$
signifies a situation where the semantic value assigned to the head $h(T)$
is in {\em disagreement} with the semantic probe used, but there are accessible terms
in $T$ that are individually in agreement with the semantic probe 
and whose combined agreement is greater than 
the magnitude of the disagreement for $h(T)$. 

\begin{rem}\label{substragree}{\rm 
The construction illustrated in Lemma~\ref{toyRBid},
Corollary~\ref{toycor1}, 
and Lemma~\ref{toyRB2}
above can be seen 
as a way of extracting substructures where agreement/disagreement with
a given semantic value is concentrated.}
\end{rem}

\smallskip

As mentioned at the beginning of this section and in Remark~\ref{toyrem},
the example semiring homomorphism $\phi_{\Upsilon, s,h}(T)$ 
used in Lemma~\ref{toyRBid}, Corollary~\ref{toycor1}, and Lemma~\ref{toyRB2}
is unsatisfactory because it only uses the semantic values 
assigned to the leaves of the syntactic objects $T$ through the
map $s: \cS\cO_0 \to \cS$ and does not create new semantic values
assigned to the syntactic objects $T$ themselves that go beyond the
value already assigned to its head $h(T)$ leaf. This is obviously not
how an assignment of semantic values to sentences should work, and was only
discussed here as a way to show, in the simplest possible form, how Birkhoff
factorizations work. We now move on to more realistic models.
These will again be simplified toy models, but we will gradually
introduce more realistic features.

\medskip
\subsection{Head-driven interfaces and convexity} \label{HeadIdRenSec2}

We now assume that our semantic space model $\cS$
is a geodesically convex region inside a Riemannian manifold $(M,g)$.
A region $\cS\subset M$ is geodesically convex if, for any given points $s,s'\in \cS$
minimal length geodesic arcs $\gamma:[0,1]\to M$ with $\gamma(0)=s$ and
$\gamma(1)=s'$ are contained in the region, $\gamma(t)\in M$ for all $t\in [0,1]$.

\smallskip

This includes in particular the cases where $\cS$ is a vector space or a simplex.
In these cases, we write $\{ \lambda s + (1-\lambda) s' \,|\, \lambda\in [0,1] \}$ for
the segment connecting $s,s'$ in $\cS$ (the convex combinations of $s$ and $s'$).
With a slight abuse of notation, in the more general case of geodesically convex 
regions inside a Riemannian manifold, we will still write $\lambda s + (1-\lambda) s'$ to
indicate the point $\gamma(\lambda)$ along a given minimal geodesic arc 
$(\gamma(t))_{0\leq t \leq 1}$ in $\cS$.

\smallskip

We assume, as above, that there is a map $s: \cS\cO_0 \to \cS$ that assigns
semantic values to the lexical items and syntactic features. 

\smallskip
\subsubsection{Comparison functions}\label{PCfuncSec}

We assume that the semantic space $\cS$ is endowed with
one of the following additional data:
\begin{enumerate}
\item On the product $\cS\times \cS$ there is a function
\begin{equation}\label{ProbS}
 \bP: \cS\times \cS \to [0,1] 
\end{equation} 
that evaluates the probability that two points $s,s'$ are semantically associated (interpreted 
as the frequency with which they are semantically associated within a specified context).
We assume that $\bP$ is symmetric, $\bP(s,s')=\bP(s',s)$, i.e.~that it factors
through the symmetric product
$$ \bP: {\rm Sym}^2(\cS) \to [0,1] \, . $$
One can additionally assume that $\bP$ is a probability measure on $\cS\times \cS$,
although this is not strictly necessary in what follows. 
If the underlying space $\cS$ is convex, we always assume that $\bP$ is a biconcave function. 
\item On the product $\cS\times \cS$ there is a function
\begin{equation}\label{corrS}
 \fC: \cS\times \cS \to \R
\end{equation} 
that evaluates the level of semantic agreement/disagreement between two points $s,s'$,
with $|\fC(s,s')|$ measuring the magnitude of agreement/disagreement and 
${\rm sign}( \fC(s,s') )=\fC(s,s')/|\fC(s,s')| \in \{ \pm 1 \}$ measuring whether there is
agreement or disagreement. Again we assume that the function $\fC$ is symmetric.
In the case of a semantic vector space $\cS$ one can additionally assume that $\fC$ is
obtained from a symmetric bilinear form by
\begin{equation}\label{Ccosine}
\fC(s,s')=\frac{ \langle s, s' \rangle }{\| s \| \, \| s' \|} \, ,
\end{equation}
which gives the usual cosine similarity, but in general it is not necessary for $\fC(s,s')$
to be of the form \eqref{Ccosine}. 
\end{enumerate}

\smallskip

This type of comparison functions $\bP: \cS\times \cS \to [0,1]$ as in \eqref{ProbS} or
$\fC: \cS\times \cS \to \R$ as in \eqref{corrS}, should really be thought of, more generally, 
as a collection $\bP=\{ \bP_\sigma \}$ or $\fC=\{ \fC_\sigma \}$, where the index $\sigma$
runs over certain syntactic functions (in the sense of functional relations between 
constituents in a clause). For example, suppose that one looks at the two sentences ``dog bites man" and
``man bites dog.''  In the first case the VP determines a point on the geodesic
arc in $\cS$ between the points $s({\rm bite})$ and $s({\rm man})$ at a distance
$\bP(s({\rm bite}),s({\rm man}))$ from the vertex $s({\rm bite})$. The value
$\bP(s({\rm bite}),s({\rm man}))\in [0,1]$ evaluates the degree of ``likelihood"
of this association.

\smallskip
\subsubsection{Threshold Rota-Baxter operators}\label{cRBsec}

As in the cases discussed in the previous section, we can consider a semiring $\cP$
endowed with a Rota--Baxter structure. 

\begin{lem}\label{probRB}
Consider the semiring $\cP=([0,1],\max, \cdot, 0,1)$. Then
the threshold operators $$ c_\lambda: \cP \to \cP \ \ \ \text{
with } \ \ \  \lambda\in [0,1]\, , $$ given by 
\begin{equation}\label{thrc1}
c_\lambda(x)=\left\{ \begin{array}{ll}
x & x < \lambda \\
1 & x \geq \lambda 
\end{array}\right. 
\end{equation}
are Rota--Baxter operators of weight $-1$ that satisfy the property \eqref{R1eq}.
\end{lem}

\proof We can compare the values in the Rota--Baxter identity as follows:

\medskip
\begin{center}
\begin{tabular}{|c|c|c|c|c|}
\hline 
& $x <\lambda$, $y< \lambda$ & $x\geq \lambda$, $y< \lambda$ & $x< \lambda$, $y\geq \lambda$ & $x\geq \lambda$, $y\geq \lambda$ \\
\hline
$c_\lambda(xy)$ & $xy$ & $xy$ & $xy$ & $\left\{ \begin{array}{ll} xy & xy< \lambda \\ 1 & xy\geq  \lambda \end{array}\right. $ \\
\hline 
$c_\lambda(x) c_\lambda(y)$ & $xy$ & $y$ & $x$ & $1$ \\
\hline
$c_\lambda(c_\lambda(x) y)$ & $xy$ & $y$ &
$xy$ & $1$ \\
\hline
$c_\lambda(xc_\lambda(y))$ & $xy$ & $xy$ & $x$ & $1$ \\
\hline
\end{tabular}
\end{center}
Indeed, we have $x,y,\lambda \in [0,1]$, hence if either $x< \lambda$ or $y< \lambda$
then $xy < \lambda$. The 
the maximum of the first two rows is $\max\{ c_\lambda(xy), c_\lambda(x) c_\lambda(y) \}=c_\lambda(x) c_\lambda(y)$, which shows that the identity \eqref{R1eq} holds. Moreover, the maximum between
the last two rows of the table above is also equal to $c_\lambda(x) c_\lambda(y)$ so that the
Rota--Baxter identity of weight $-1$ holds. 
\endproof

\smallskip
\subsubsection{$\cP$-valued semiring character} \label{PvalPhiSec}

We then consider constructions of a character. For our target semiring 
$\cP$, we can consider characters $\phi: \cH^{cone} \to \cP$ with 
domain a convex cone inside $\cH$,
which ensures that if generators $F\in \fF_{\cS\cO_0}$ are mapped to $\cP$,
linear combinations that are in the cone will also map to $\cS$.

\begin{lem}\label{phiP}
Suppose given a semantic space $\cS$ that is geodesically convex, endowed with
a function $s: \cS\cO_0 \to \cS$ and a function $\bP: {\rm Sym}^2(\cS) \to [0,1]$ as above.
Also assume given a head function $h$ defined on a domain ${\rm Dom}(h)\subset \fT_{\cS\cO_0}$.
The function $s: \cS\cO_0 \to \cS$ extends to a map $s: {\rm Dom}(h) \to \cS$, and 
these data determine a character given by a map
$$ \phi_{s,\bP,h}: \cH^{cone} \to \cP \, , $$
with $\cH^{cone}$ the cone of convex linear combinations $\sum_i a_i F_i$ with $0\leq a_i$ and
$\sum_i a_i=1$, and forests $F_i\in \fF_{\cS\cO_0}$. The character is defined 
on the generators by $\phi_{s,\bP,h}(T)=0$ for $T\notin {\rm Dom}(h)$, while 
for $T\in {\rm Dom}(h)$ the value $\phi_{s,\bP,h}(T)$ is inductively determined
by the description of $T$ as iterations of the Merge operation $\fM$ in the magma \eqref{SOmagma}.
It is extended to $\cH^{cone}$ by $\phi_{s,\bP,h}(F)=\prod_k \phi_{s,\bP,h}(T_k)$, for 
$F=\sqcup_k T_k$, and
$\phi_{s,\bP,h}(\sum_i a_i F_i)=\max_i a_i \phi_{s,\bP,h}(F_i)$.
\end{lem}

\proof To an unordered pair $\fM(\alpha,\beta)=\{ \alpha, \beta \}$ of $\alpha,\beta\in \cS\cO_0$
we assign a value in $\cP$ in the following way. If the tree
$T=\fM(\alpha,\beta)\in \cS\cO=\fT_{\cS\cO_0}$ is not in ${\rm Dom}(h)$ we assign
value $\phi_{s,\bP,h}(T)=0$. If $T\in {\rm Dom}(h)$, consider the value $$p_{\alpha,\beta}:=\bP(s(\alpha),s(\beta))$$
and define $s(T)\in \cS$ as 
\begin{equation}\label{sTconvex}
 s(T)= p s(\alpha) + (1-p) s(\beta) 
\end{equation} 
where $p\in [0,1]$ is
\begin{equation}\label{pTconvex}
 p=\left\{ \begin{array}{ll} p_{\alpha,\beta} & \alpha = h(T) \\ 1-p_{\alpha,\beta} & \beta=h(T) \, .
\end{array}\right. 
\end{equation}
We then set
\begin{equation}\label{phisPconvex}
 \phi_{s,\bP,h}(\fM(\alpha,\beta))=p_{\alpha,\beta} \, . 
\end{equation} 
We then proceed inductively. If $T=\fM(T_1,T_2)$ is not in ${\rm Dom}(h)$  we set
$\phi_{s,\bP,h}(T)=0$. If it is in ${\rm Dom}(h)$,  then by the properties of
head functions, $T_1$ and $T_2$ are also in ${\rm Dom}(h)$. So we can
assign to $T$ the point $s(T)\in \cS$ given by
$$ s(T)= p\, s(T_1) + (1-p)\, s(T_2) $$
where
\begin{equation}\label{pT}
 p=\left\{ \begin{array}{ll} p_{s(T_1),s(T_2)} & h(T) = h(T_1) \\ 1-p_{s(T_1),s(T_2)} & h(T)=h(T_2) \, .
\end{array}\right. 
\end{equation}
with
$$ p_{s(T_1),s(T_2)}= \bP(s(T_1),s(T_2)) \, . $$
We then set
$$ \phi_{s,\bP,h}(T)= p_{s(T_1),s(T_2)} \, . $$
It is clear that this determines a map
$$ \phi_{s,\bP,h}: \cH^{cone} \to \cP \, , $$
with $\phi_{s,\bP,h}(\sum_i a_i F_i)=\max_i \{ a_i \phi_{s,\bP,h}(F_i) \}$ and
$\phi_{s,\bP,h}(F)=\prod_k \phi_{s,\bP,h}(T_k)$, 
for $F=\sqcup _kT_k \in \fF_{\cS\cO_0}$.
\endproof 

\begin{rem}\label{toy2rem}{\rm
The semiring-valued character $\phi_{s,\bP,h}$ constructed in Lemma~\ref{phiP}
improves on the construction of the character $\phi_{\Upsilon,s,h}$ of Lemma~\ref{phiUpsilon}
in the sense that the values $\phi_{s,\bP,h}(T)$ assigned to syntactic object do not
depend uniquely on the semantic values of the lexical items, but also on other points
of semantic space $\cS$, obtained as convex combinations of values assigned to
lexical items. However, it should still be regarded as a toy model case, as the way
in which these combinations are obtained and the corresponding value of 
$\phi_{s,\bP,h}(T)$ is computed is still overly simplistic. We show in \S \ref{HeadIdRenSec3}
another similar simplified toy model example, with a choice of semiring-valued
character that combines properties of $\phi_{s,\bP,h}$ of Lemma~\ref{phiP} and
$\phi_{\Upsilon,s,h}$ of Lemma~\ref{phiUpsilon}.
}\end{rem}

\smallskip

Note that we have, in principle, two simple choices of how to extend inductively
\eqref{pTconvex} from the cherry tree case $T=\fM(\alpha,\beta)$ to the more general
case $T=\fM(T_1,T_2)$. One is to define $p_{s(T_1),s(T_2)}$ 
as in \eqref{pT}, with $p_{s(T_1),s(T_2)}= \bP(s(T_1),s(T_2))$, inductively 
using the previously constructed points $s(T_1)$ and 
$s(T_2)$. Another possibility, more similar to our previous example 
$\phi_{\Upsilon,s,h}$ of Lemma~\ref{phiUpsilon}, is to define it using the heads, 
$\bP(h(T_1),h(T_2))$. To see why the option of \eqref{pT} is clearly preferable,
consider the following example. Take the 
three sentences ``man bites dog",  ``man bites apple", ``dog bites man". 
Denoting by M,B,D,A the respective points in $\cS$ associated to these
lexical items, the points associated to the respective sentences are shown
in the diagram in Figure~\ref{MBDfig}.

\begin{figure}[h]
\begin{center}
\includegraphics[scale=0.2]{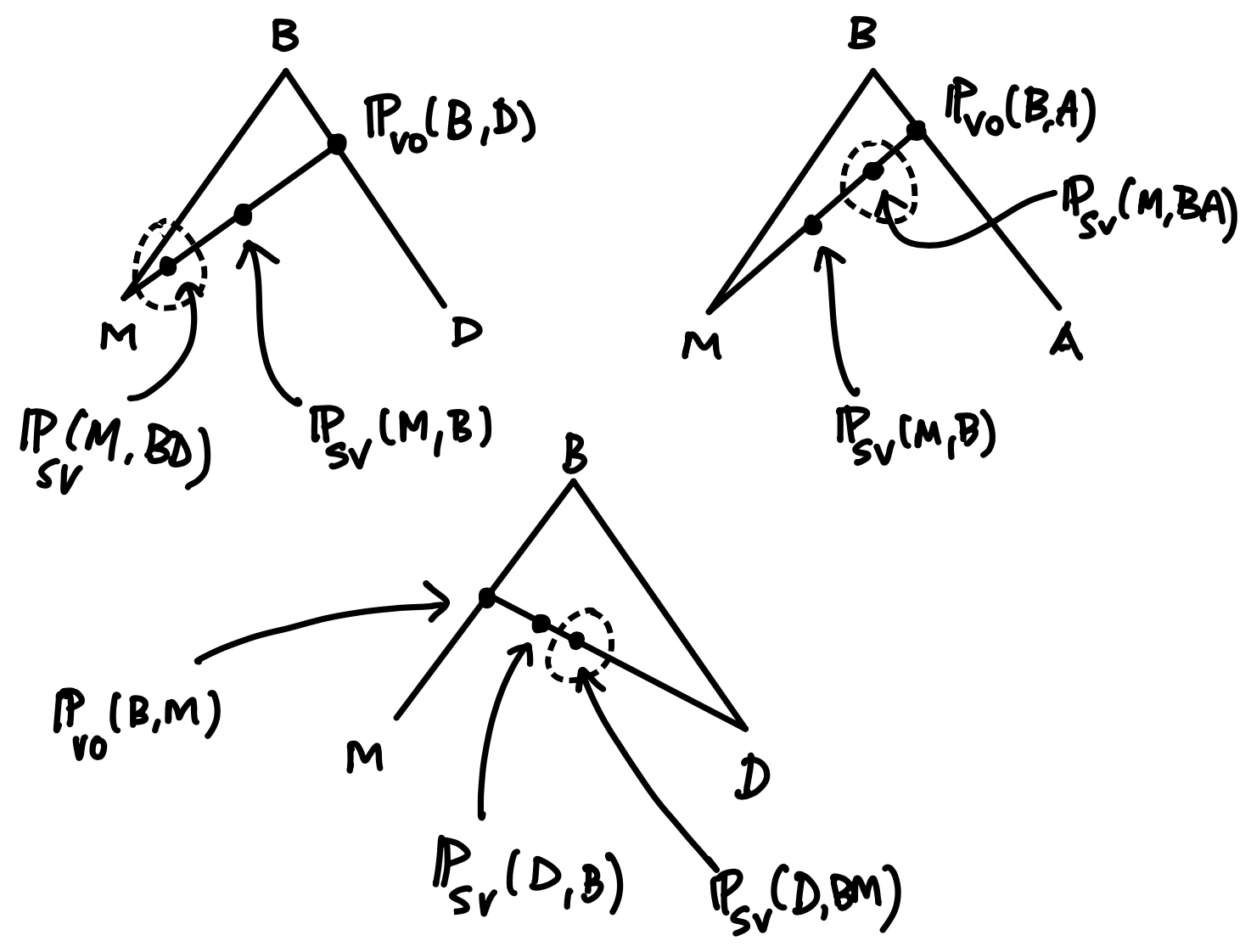}  
\caption{Sketch of different semantic points constructed by geodesic arcs for the three sentences ``man bites dog",  ``man bites apple", ``dog bites man", and 
with the two different choices of $p_{s(T_1),s(T_2)}=\bP(s(T_1),s(T_2))$ (circled) 
or $\bP(h(T_1),h(T_2))$. \label{MBDfig}}
\end{center}
\end{figure}

In the sentence ``dog bites man", the VP determines a point on the geodesic
arc in $\cS$ between the points $B$ and $M$ at a distance
$\bP_\sigma(B,M)$ from the vertex $B$, where in
this case $\sigma$ is the verb-object relation and the value
$\bP_\sigma(B,M))\in [0,1]$ expresses the degree 
of ``likelihood" of this association in the relation $\sigma$. One then considers, on
the geodesic arc in $\cS$ between this point associated to the VP phrase
and the point $D$, a new point. In the case of the choice $p_{s(T_1),s(T_2)}=\bP(s(T_1),s(T_2))$
as in \eqref{pT}, this point is located at a distance either $\bP_{\sigma'}(D,BM)$,
where we write $BM$ for the point $s(\fM(B,M))$ associated to the VP by the
procedure just described and $\sigma'$ is the subject-verb relation between $D$ and $h(\fM(B,M))$. 
In the case where we use $\bP(h(T_1),h(T_2))$, this
point is located at a distance $\bP_{\sigma'}(D,B)$ where $\sigma'$ the subject-verb relation. 
The cases of the second and third sentences are analogous as sketched in Figure~\ref{MBDfig}.
One can see in a simple example like this, why the choice $p_{s(T_1),s(T_2)}=\bP(s(T_1),s(T_2))$
is preferable to $\bP(h(T_1),h(T_2))$ by comparing the location of points in the first two cases
in Figure~\ref{MBDfig}. If one uses $\bP(h(T_1),h(T_2))$ the length of the arc of geodesic between 
$M$ and the point $BD$, respectively $BA$ is in both cases determined by the same
value $\bP_{\sigma'}(M,B)$, while in the case of $\bP(s(T_1),s(T_2))$ one has
different lengths $\bP_{\sigma'}(M,BD)<< \bP_{\sigma'}(M,BA)$. 

\subsubsection{Birkhoff factorization with threshold operators}\label{cRBfactSec}
The Birkhoff factorization of the character $\phi_{s,\bP,h}$ with respect to the
threshold Rota--Baxter operators provides a way of searching for substructures
with large semantic agreement between constituent parts. 
More precisely, we have the following.

\begin{prop}\label{BFphisPh}
The Birkhoff factorization of the character $\phi_{s,\bP,h}$ of Lemma~\ref{phiP}
with respect to the Rota--Baxter operators $c_\lambda$ of weight $-1$ 
identifies, as elements that achieve the maximum, 
those accessible terms $T_v \subset T$ with values $\phi_{s,\bP,h}(T_v)$ above a threshold $\lambda$,
identifying substructures within $T$ that carry large semantic agreement between their constituent parts.
\end{prop}

\proof If we perform the Birkhoff factorization of the character $\phi_{s,\bP,h}$ using
the Rota--Baxter operator $c_\lambda$ of weight $-1$, we obtain
$$ \phi_{s,\bP,h,-}(T)=c_\lambda (\tilde\phi_{s,\bP,h}(T))= $$
$$ c_\lambda (\max\{ \phi_{s,\bP,h}(T), c_\lambda (\cdots
c_\lambda(\phi_{s,\bP,h}(F_{\underline{v}_N})) 
\phi_{s,\bP,h}(F_{\underline{v}_{N-1}}/F_{\underline{v}_N})) \cdots \phi_{s,\bP,h}(T/F_{\underline{v}_0}) \}) $$
over nested chains of subforests of all possible lengths $N$, as before.
Again we can look for simplicity at the case of subtrees, as the value on forests is the
semiring product of the values on the tree components. When
we look at chains of length $N=1$ with subtrees, we are comparing
$\phi_{s,\bP,h}(T)$ to the value
$c_\lambda(\phi_{s,\bP,h}(T_v))\cdot \phi_{s,\bP,h}(T/T_v)$. Arguing as above, we
have $$c_\lambda( \max\{ \phi_{s,\bP,h}(T), c_\lambda(\phi_{s,\bP,h}(T_v))\cdot \phi_{s,\bP,h}(T/T_v)\})=$$
$$c_\lambda(\max\{ p_{s(T_1),s(T_2)}, c_\lambda(p_{s(T_{v,1}) s(T_v,2)})\cdot p_{s(T_1),s(T_2)} \})=
c_\lambda(p_{s(T_1),s(T_2)})$$ where this time the maximal value is realized by all the terms  
$T_v \subset T$ that have $p_{s(T_{v,1}) s(T_v,2)}\geq \lambda$ and $p_{s(T_1),s(T_2)}\geq \lambda$. 
Note that longer
sequences will have products with intermediate terms 
$\phi_{s,\bP,h}(F_{\underline{v}_{i-1}}/F_{\underline{v}_i})<1$ hence will 
not achieve the same maximum. Thus, the maximizers are accessible terms
that carry large semantic agreement between their constituent parts.
\endproof

\smallskip

For example, suppose that we consider again the two sentences 
``dog bites man" and ``man bites dog". As shown above, the resulting semantic
points associated to these two sentences are, as they should be, in different locations in $\cS$.
Moreover, the fact that one will have $\bP_{\sigma'}(M,BD) << 
\bP_{\sigma'}(D,BM)$ when $\sigma'$ is the subject-verb relation,
implies that the threshold operators $c_\lambda$ discussed in the previous section will
filter out the second sentence before the first. 
\subsubsection{From geodesic arcs to convex neighborhoods}\label{GeodNeighSec}

The construction of the character $\phi_{s,\bP,h}$
of Lemma~\ref{phiP} is also a toy model. It is better
than the initial oversimplified toy model of Lemma~\ref{phiUpsilon} (see Remark~\ref{toyrem}),
because it does not use only the points in the semantic space $\cS$ associated to
the head leaf, but it still uses only geodesic arcs in the semantic space $\cS$. 
Passing from a zero-dimensional
to a one-dimensional  representation of syntactic relations is an improvement, and as we
will discuss in \S \ref{SyntImageSec} it is already sufficient to obtain an embedded image
of syntax inside semantics (in essence because the syntactic objects are themselves 
1-dimensional tree structures). However, this representation
can be improved by considering, along with geodesic arcs, higher
dimensional convex structures like simplexes and geodesic neighborhoods of points.
While we will not expand this approach in the present paper, it is worth mentioning
some ideas that relate to some of what we will be discussing in the following sections. 
Given a syntactic object $T\in \cS\cO$ with $T\in {\rm Dom}(h)$, a 
geodesically convex semantic space $\cS$, and a mapping $s: {\rm Dom}(h)\to \cS$
constructed as in Lemma~\ref{phiP}, we can consider the points $s(T_v) \in \cS$
associated to all the accessible terms of $T$. (See \S \ref{SyntImageSec} below,
for the embedding properties of this map.) Now consider geodesic balls $B_v(\epsilon)$
in $\cS$ centered at the points $s(T_v)$ with radius $\epsilon>0$. Here by geodesic ball
we mean the image under the exponential map of a ball in the tangent space. We assume
the injectivity radius of $\cS$ is larger than the maximal distance between the points $s(T_v)$
for all $v\in V(T)$. In terms of the semantic space, a geodesic neighborhood
around a given point $s\in \cS$ represents all the close semantic associations to the
semantic point $s$ recorded in $\cS$. We can then vary the scale $\epsilon$ of the geodesic 
balls and form simplicial complex (a Vietoris--Rips complex) associated to the 
intersections of these geodesic balls (see Figure~\ref{VRcxfig}). As the scale $\epsilon>0$
varies, one obtains a filtered complex, according to the familiar construction of {\em persistent topology}
(see \cite{EdelHar}). The scale $\epsilon$ provides another form of filtering that generalizes  what we
previously described in terms of the threshold operators $c_\lambda$. In this case, the persistent
structures that arise can be seen as detecting ``collections of substructures that carry 
higher semantic relatedness" inside the given hierarchical structure $T$.

\begin{figure}[h]
\begin{center}
\includegraphics[scale=0.2]{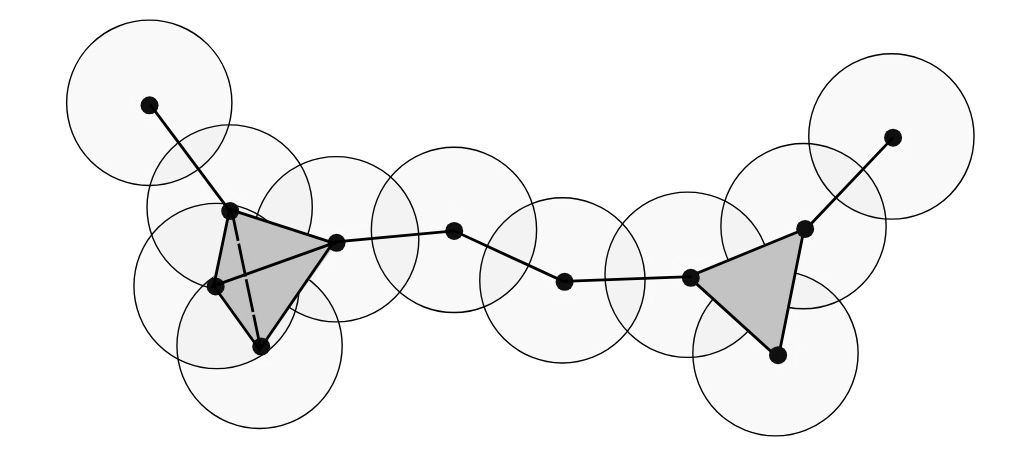}  
\caption{Example of a Vietoris--Rips complex. \label{VRcxfig}}
\end{center}
\end{figure}

\subsection{Head-driven interfaces and vector models}
\label{HeadIdRenSec3}

Consider now the case where the semantic space $\cS$ is modeled by a
vector space, and assume that it is endowed with a function
$\fC: \cS\times \cS  \to \R$ that describes the level of semantic agreement,
as in \S \ref{PCfuncSec}.
This may be based on cosine similarity or on other methods: the detailed
form of $\fC$ is not important in what follows, beyond the basic property
described in \S \ref{PCfuncSec}.

\smallskip
\subsubsection{Max-plus-valued semiring character}

We discuss an example where we consider again the max-plus semiring 
$\cR=(\R\cup \{ -\infty \}, \max, +)$ and a semantic comparison function
of the form $\fC: \cS\times \cS  \to \R$ as discussed in \S \ref{PCfuncSec}.

\begin{lem}\label{phiCsh}
Consider the semiring $\cR=(\R\cup \{ -\infty \}, \max, +)$. The data of a
function $\fC: \cS\times \cS  \to \R$ as above, a function
$s: \cS\cO_0\to \cS$ and a head function defined on a domain
${\rm Dom}(h)\subset \fT_{\cS\cO_0}$ determine a semiring-valued
character
$$ \phi_{s,\fC,h}: \cH^{semi} \to \cR\, , $$
with $\cH^{semi}$ the semiring of linear combinations $\sum_i a_i F_i$ with $a_i\geq 0$.
\end{lem}

\proof For any tree $T\notin {\rm Dom}(h)$ we set $\phi_{s,\fC,h}(T)=-\infty$.
We then consider only trees that are in ${\rm Dom}(h)$.
As in Lemma~\ref{phiP} we start by considering the case of a tree
of the form $T=\fM(\alpha,\beta)=\{ \alpha, \beta \}$ with $\alpha,\beta\in \cS\cO_0$.
We assign to this tree a value in $\cR$ obtained by computing
$\fC(s(\alpha),s(\beta))\in \R$ and considering the line $L_{\alpha,\beta}$, in the vector space $\cS$,
through the points $s(\alpha)$ and $s(\beta)$,
$$ L_{\alpha,\beta}=\{  t \alpha + (1-t) \beta =\beta + t(\alpha-\beta) \,|\, t\in \R \} \, , $$
if $\beta=h(T)$ (exchanging $\alpha$ and $\beta$ if $\alpha=h(T)$, that is, replacing $t$ with $1-t$). 
We then define 
$$ t_{\alpha,\beta}=\fC(\alpha,\beta) $$
\begin{equation}\label{sTvec}
 s(T):=\beta + t_{\alpha,\beta}  (\alpha-\beta)\in L_{\alpha,\beta} \, . 
\end{equation} 
This has the effect of creating a new point $s(T)$ which moves the value $s(h(T))$
along the line $L_{\alpha,\beta}$ in the direction $\alpha$ (or in the opposite direction)
depending on the agreement/disagreement sign of  $\fC(\alpha,\beta)$.
We then set
$$ \phi_{s,\fC,h}(\fM(\alpha,\beta))=\left\{ \begin{array}{ll} \fC(\alpha,\beta) & \beta=h(T) \\
1-\fC(\alpha,\beta) & \alpha=h(T) \end{array} \right. $$
We can then proceed inductively, setting, for $T=\fM(T_1,T_2)\in {\rm Dom}(h)$
$$ t_T =\left\{ \begin{array}{ll} \fC(s(T_1),s(T_2))  & h(T)=h(T_2) \\
1-\fC(s(T_1),s(T_2))  & h(T)=h(T_1) \end{array} \right. $$
$$ s(T) = t_T s(T_1) + (1-t_T) s(T_2) $$
$$ =\left\{ 
\begin{array}{ll} s(T_2) +t_T (s(T_1)-s(T_2)) & h(T)=h(T_2) \\
s(T_1) +t_T (s(T_2)-s(T_1)) & h(T)=h(T_1)
\end{array} \right. $$
$$ \phi_{s,\fC,h}(T=\fM(T_1,T_2))=t_T \, .  $$
Setting $\phi_{s,\fC,h}(F)=\sum_k \phi_{s,\fC,h}(T_k)$ for $F=\sqcup_k T_k$ and
$\phi_{s,\fC,h}(\sum_i a_i F_i)=\max\{ a_i \phi_{s,\fC,h}(F_i) \}$ then completely determines
$\phi_{s,\fC,h}$ on $\cH^{semi}$. 
\endproof

\smallskip
\subsubsection{Hyperplane arrangements}

The following observation follows from Lemma~\ref{phiCsh},
rephrased in a more geometric way.

\begin{lem}\label{lincombCsh}
Let $S_\fC$ denote the multiplicative subsemigroup  of
$\R^*$ generated by the set of non-zero elements in $\fC(s(\cS\cO_0)\times s(\cS\cO_0))$.
For $T\in \fT_{\cS\cO_0}$ in ${\rm Dom}(h)$, let  $L(T)$ be the set of
leaves of the tree. We write, for simplicity of notation, $s(L(T))$ for
the set of vectors $s(\lambda(L(T))) \subset \cS$.
Let $S_T\subset S_\fC\subset \R^*$ be the multiplicative semigroup
generated by the set $\R^*\cap \fC(s(L(T))\times s(L(T)))$. 
The vector $s(T)$ of \eqref{sTvec} is in the linear span of the set 
$s(L(T))$ with coefficients in $S_T$.
\end{lem}

\proof
Suppose given a binary rooted tree $T \in {\rm Dom}(h)\subset \fT_{\cS\cO_0}$, with $L(T)$ its set of leaves.
By the recursive procedure of Lemma~\ref{phiCsh}, based on the construction of $T$ by repeated
application of free symmetric Merge $\fM$, as an element in the magma \eqref{SOmagma},
the resulting point $s(T)$ in the vector space $\cS$ is a linear combination of the vectors $s(\ell)$ with
$\ell \in L(T)$ (where we write $s(\ell)$ as a shorthand notation for $s(\lambda(\ell))$,
$$ s(T)=\sum_{\ell\in L(T)} a_\ell \, s(\ell) \in {\rm span}(L(T)) $$
with coefficients $a_\ell$ in the multiplicative subsemigroup $S_T\subset S_\fC$.
\endproof

\smallskip

\begin{lem}\label{hyparrSO0}
If $\fC$ on the vector space $\cS$ is given by a cosine similarity as in \eqref{Ccosine},
then the set of vectors $s(\cS\cO_0)\subset \cS$ determines an associated hyperplane arrangement
$\cH\cA_{\cS\cO_0}$ of hyperplanes
\begin{equation}\label{arrSO0eq}
 \cH\cA_{\cS\cO_0} =\{ H_\lambda=\{ v\in \cS \,|\, \langle v,s(\lambda) \rangle=0 \}\,|\, \lambda \in \cS\cO_0, \, s(\lambda)\neq 0  \} \, , 
\end{equation} 
where the hyperplane $H_\lambda$ describes all semantic vectors that are neutral with
respect to $s(\lambda)$, namely vectors $v\neq 0$ with $\fC(v,s(\lambda))=0$. 
\end{lem}

This is immediate, as the set of hyperplanes here is simply given by the normal hyperplanes 
to the  given set of vectors under the inner product that also defines the cosine similarity. 

\smallskip

One can then see the construction of the character $\phi_{s,\fC,h}$ of Lemma~\ref{phiCsh}
in the following way.

\begin{lem}\label{hyparrCsh}
The vectors $s(T)$, for $T\in {\rm Dom}(h)\subset \fT_{\cS\cO_0}$,
give a refinement of the hyperplane arrangement $\cH\cA_{\cS\cO_0}$ of Lemma~\ref{hyparrSO0},
with a resulting arrangement 
\begin{equation}\label{arrSOeq}
 \cH\cA_{\cS\cO} =\{ H_T=\{ v\in \cS \,|\, \langle v,s(T) \rangle=0 \}\,|\, T \in \fT_{\cS\cO_0}, \, 
 s(T)\neq 0  \} \, , 
\end{equation} 
where the values $t_{T_v}=\phi_{s,\fC,h}(T_v)$, with $v\in V(T)$ determine which 
chambers of the complement of the arrangement $\cH\cA_{\cS\cO_0}$ the hyperplane $H_T$ 
crosses. 
\end{lem}

\proof The inductive construction of $\phi_{s,\fC,h}$ in Lemma~\ref{phiCsh} shows
that, for $\alpha,\beta\in \cS\cO_0$ the value $\phi_{s,\fC,h}(\fM(\alpha,\beta))=t_{\alpha,\beta}$
determines which chambers of the complement of $H_\alpha \cup H_\beta$ the hyperplane
$H_{\fM(\alpha,\beta)}$ crosses, depending on the sign of $t_{\alpha,\beta}$ and of $1-t_{\alpha,\beta}$. 
Inductively, the same applies to the role of $t_T=\phi_{s,\fC,h}(T)$ in determining 
the position of $H_T$ with respect to $H_{T_1}$ and $H_{T_2}$, hence the role of the 
values $t_{T_v}$, for the accessible terms $T_v\subset T$, in determining the position
of $H_T$ with respect to $\cH\cA_{\cS\cO_0}$.
\endproof

\smallskip
\subsubsection{ReLU Birkhoff factorization}

We then consider, in this model, the effect of taking the Birkhoff factorization
with respect to the ReLU Rota-Baxter operator of weight $+1$. Note that 
this gives an instance of a situation quite familiar from the theory of neural 
networks, where a ReLU function is applied 
to certain linear combinations and an optimization is performed over the result. 

\smallskip

\begin{prop}\label{ReLUphiCsh}
The Birkhoff decomposition of the character $\phi_{s,\fC,h}$ of Lemma~\ref{phiCsh},
with respect to the ReLU Rota--Baxter operator of weight $+1$ selects, for a given
tree $T$, chains $T_{v_N}\subset T_{v_{N-1}}\subset \cdots \subset T_{v_1}\subset T$ 
of accessible terms of $T$ where each $\phi_{s,\fC, h}(T_{v_i})>0$ and of 
maximal values among all accessible terms of $T_{v_{i-1}}$, that is, every
$T_{v_i}$ optimizes the value of the character among the available accessible 
terms.
\end{prop}

\proof As in Lemma~\ref{toyRB2}, we consider
$$ \phi_{s,\fC, h,-} (T) =\max \{ \phi_{s,\fC, h}(T), (\cdots ( \phi_{s,\fC, h}(F_{\underline{v}_N})^+ +\cdots +
\phi_{s,\fC, h}(F_{\underline{v}_{i-1}}/F_{\underline{v}_i}))^+ +
\cdots )^+ + \phi_{s,\fC, h}(T/F_{\underline{v}_0}) \}^+ \, , $$ over
all nested sequences of subforests of arbitrary length $N\geq 1$.  For
chains of length $N=1$, considering the case of subtrees
$T_v \subset T$, we are comparing $\phi_{s,\fC, h}(T)$ and
$\phi_{s,\fC, h}(T_v)^+ + \phi_{s,\fC, h}(T/T_v)$.  Again we have
$h((T/T_v)_1)=h(T_1)$ and $h((T/T_v)_2)=h(T_2)$, with
$T/T_v=\fM((T/T_v)_1, (T/T_v)_2)$, so that
$\phi_{s,\bP,h}(T/T_v)=\phi_{s,\bP,h}(T)$. Thus, the maximum
$\max\{ \phi_{s,\fC, h}(T), \phi_{s,\fC, h}(T_v)^+ + \phi_{s,\fC,
  h}(T/T_v) \}^+= (\phi_{s,\fC, h}(T_v)^+ + \phi_{s,\fC, h}(T/T_v))^+$
is achieved at the largest positive value $\phi_{s,\fC, h}(T_v)$ over
all accessible terms $T_v\subset T$.  The next step then compares this
maximal value with the values
$(\phi_{s,\fC, h}(T_w)^+ + \phi_{s,\fC, h}(T_v))^+ + \phi_{s,\fC,
  h}(T)$ over all accessible terms $T_w\subset T_v$ and the maximum is
again realized at the largest positive $\phi_{s,\fC, h}(T_w)$ among
these. This shows that the overall maximum is achieved at the longest
chain
$T_{v_N}\subset T_{v_{N-1}}\subset \cdots \subset T_{v_1}\subset T$ of
accessible terms where each $T_{v_i}$ has $\phi_{s,\fC, h}(T_{v_i})>0$
and of maximal values among all accessible terms of $T_{v_{i-1}}$.
\endproof

\subsection{Not a tensor-product model of semantic compositionality} \label{TensorSemSec} 

While the examples of characters, Rota--Baxter structures, and
Birkhoff factorizations considered above are just a simplified model,
they are already good enough to illustrate some important points.
Consider for example the property, mentioned in Remark~\ref{phialg},
that characters are {\em not} morphisms of coalgebras, but only
morphisms of algebras (or semirings).  This has important
consequences, such as the fact that we are {\em not} dealing here with
what is often referred to as ``tensor product based" connectionist
models of computational semantics, such as \cite{Smol}. The
compositional structure of such tensor product models has in our view been rightly
criticized (see for instance \cite{MarDou}) for not being compatible
with human behavior.  Indeed one can easily see the problem with such
models: the idea of ``tensor product based" compositionality is that,
given vectors $s(\alpha), s(\beta)\in \cS$ for lexical items
$\alpha,\beta$, one would assign to a {\em planar} tree
$T=\fM_{nc}(\alpha,\beta)$ a vector
$s(\alpha)\otimes s(\beta)\in \cS\otimes \cS$ and correspondingly
evaluate cosine similarity between $T$ and another
$T'=\fM_{nc}(\gamma,\delta)$ in the form
$\fC(\alpha,\gamma)\cdot \fC(\beta,\delta)$.

There are several obvious problems with such a proposal. In a simple
example with lexical items $\alpha=\gamma =${\em light} and
$\beta=${\em blue} and $\delta=${\em green}, the planar trees $T=${\em
  light blue} and $T'=${\em light green} should have {\em closer}
semantic values $s(T)$ and $s(T')$ than the values $s(\beta)$ and
$s(\delta)$ (since both colors share the property of being light), but
a measure of similarity of the product form
$\fC(\alpha,\gamma)\cdot \fC(\beta,\delta)$ would just be equal to
$\fC(\beta,\delta)$.  A further issue with these tensor-models, from
our perspective, is that this type of model would require previous
planarization of trees and cannot be defined at the level of the
products of free symmetric Merge.

In contrast, in the type of model we are discussing these issues do
{\em not} arise.  While we have described in \cite{MCB} and \cite{MBC}
the Merge operation on workspaces in terms of a coproduct on a Hopf
algebra of binary rooted forests, that maps to a tensor product
$\Delta: \cH \to \cH\otimes \cH$ (since comultiplication has two
outputs), the characters used for mapping to semantic spaces have no
requirement of compatibility with coproduct structure.  Indeed, in our
setting we would {\em not} assign to a tree $T=\fM(\alpha,\beta)$ a
tensor product of vectors and a product of cosine similarities, but a
{\em linear combination} $s(T)=t_T s(\alpha)+ (1-t_T)s(\beta)$, that
is indeed seemingly more directly compatible with the empirically
observed human behavior, as described in \cite{MarDou}.

\subsection{Boolean semiring}\label{BooleSec}

As a final example of a simple toy syntax-semantics interface model, in preparation for the
discussion of \S \ref{HeadIdRenSec2} we consider the simplest choice
of semiring, namely the Boolean semiring
\begin{equation}\label{BooleanB}
\cB=(\{ 0,1 \}, \vee, \wedge)=(\{0,1 \}, \max, \cdot)\, . 
\end{equation}
Assignments of values in the Boolean semiring can be regarded as a form of truth-valued
semantics, where one assigns a $0/1$ (F/T) value to (parts of) sentences
or to syntactic objects. 

\smallskip

A map $\phi: \fT_{\cS\cO_0} \to \cB$ is an assignment of truth values,
extended to $\phi: \fF_{\cS\cO_0}\to \cB$ by $\phi(F)=\prod_i \phi(T_i)$ for $F=\sqcup_i T_i$.
We use the identity as Rota--Baxter operator.

The Bogolyubov preparation $\tilde\phi$ is then given by
\begin{equation}\label{BprepareB}
 \tilde\phi(T)=\max\{ \phi(T), \phi(F_{\underline{v}})\phi(T/F_{\underline{v}}), \ldots ,
\phi(F_{\underline{v}_N}) \phi(F_{\underline{v}_{N-1}}/F_{\underline{v}_N})\cdots 
\phi(T/F_{\underline{v}_1}) \} \, , 
\end{equation}
with the maximum taken over all chains of nested forests of accessible terms.
Thus, $\tilde\phi$ detects, in cases where the truth value assigned to $T$ may be 
False ($\phi(T)=0$), the longest chains of decompositions into accessible terms and
their complements which separately evaluate as True, hence identifying where the
truth value changes from T to F when substructures are combined into the full structure. 

\smallskip

While we will not include in this work a specific discussion of truth conditional semantics,
we can use the example above to illustrate some known difficulties with that model and
possibly some way of reconsidering some of the issues involved. We look at a
simple example, mentioned in the criticism of truth conditional semantics in
Pietroski's work \cite{Pietro}, that consists of the observation that, while
the truth conditions of ``{\em France is a republic}" and ``{\em France is hexagonal}"
are satisfied, the sentence ``{\em France is a hexagonal republic}" seems weird, 
due to the semantic mismatch in the expression
``{\em hexagonal republic}".

We view this example in the light of an assignment $\phi: \cH \to \cB$
and the corresponding Birkhoff factorization with the identity
Rota--Baxter operator as written above.  We can assume that $\phi$
assigns value $\phi(T)=1$ when $T$ has a well determined associated
truth condition and $\phi(T)=0$ when it does not. Thus, the trees
corresponding to ``{\em France is a republic}" and ``{\em France is
  hexagonal}" would have value $1$, because a country can be a
republic and can have a certain type of shape on a map, while the tree
corresponding to ``{\em hexagonal republic}" would have value $0$ if
we agree that a polygonal shape is not one of the attributes of a form
of state governance. The tree $T$ that corresponds to ``{\em France is
  a hexagonal republic}" contains an accessible term $T_v$ that
corresponds to ``{\em hexagonal republic}" and accessible terms (in
this case leaves) $\ell$ and $\ell'$ that correspond to the lexical
items ``{\em hexagonal}'' and ``{\em republic}". Each accessible term
$T_v$ has a corresponding quotient $T/T_v$. The Bogolyubov preparation
$\tilde\phi$ of \eqref{BprepareB} then takes the form
$$  \tilde\phi( \Tree [ a [ b  [ c  d ] ] ])=\max\{  \phi( \Tree [ a [ b  [ c  d ] ] ]),  \,\,
\phi(a) \phi(\Tree[b [ c d ]]) , \,\, \phi(c) \phi(\Tree[a [ b d ]]), \,\, \phi(d) \phi(\Tree[a [ b c ]]), $$
$$ \,\, \phi(\Tree[ b [ c d ]] ) \phi(a),\,\,
\phi(\Tree[ c d ]) \phi(\Tree[ a b ]), \ldots \}\, , $$
where the $\ldots$ stand for the remaining terms of the coproduct that involve a forest
of accessible terms rather than a single one, which can be treated similarly.
Thus, while one would have $\phi(T)=0$, the value of $\tilde\phi(T)=1$ detects the presence of
substructures (the third and fourth among the explicitly listed terms on the right-hand-side of the formula above) that do have well defined
truth conditions.

This more closely reflects the fact that, when parsing the original
sentence for semantic assignments, one does indeed detect the presence
of the two substructures that have unproblematic truth conditions, and
the fact that these do not combine to assign a truth condition to the
full tree $T$, causing a mismatch between the values of $\phi(T)$ and
$\tilde\phi(T)$.  This manifests itself in the weird impression resulting
from the parsing of the full sentence.

\medskip

\section{The image of syntax inside semantics} \label{SyntImageSec}

The examples illustrated above demonstrates one additional property of
this model of syntax-semantics interface: syntactic objects are
mapped, together with their compositional structure under Merge,
inside semantic spaces and so are, at least in principle,
reconstructible from this syntactic ``shadow" projected on the model
used for the representation of semantic proximity relations. This
observation is in fact of direct relevance to the current controversy
about the relationship between large language models and generative
linguistics, as we discuss more explicitly below in \S \ref{TransfSec}
below.  For now, let us add some additional detail to this picture.

\smallskip

Consider again the setting of Lemma~\ref{phiP} above.

\begin{prop}\label{treesinS}
Let $\cS$ be a semantic space that is a geodesically convex Riemannian manifold, 
endowed with a semantic proximity function
$\bP: {\rm Sym}^2(\cS)\to [0,1]$ with the property that, for $s\neq s'$ one has $\bP(s,s')\in (0,1)$, 
and a map $s: \cS\cO_0 \to \cS$ that assigns semantic
values to lexical items and syntactic features. Let $h$ be a head function with domain
${\rm Dom}(h)\subset \fT_{\cS\cO_0}$. These data determine embeddings of
trees $T\in {\rm Dom}(h)$ {\em inside} the semantic space $\cS$.
\end{prop}

\proof 
Arguing as in Lemma~\ref{phiP}, we can use the convexity property of $\cS$
and the function $\bP$ to extend $s: \cS\cO_0 \to \cS$ to a function $s: {\rm Dom}(h) \to \cS$,
inductively on the generation via Merge of objects $T\in \fT_{\cS\cO_0}$, by setting, for
$T\in {\rm Dom}(h)$
\begin{equation}\label{sT}
 s(T)= p\, s(T_1) + (1-p)\, s(T_2) \ \ \  \text{ for } \ \ \   T = \Tree[ $T_1$  $T_2$ ] 
\end{equation} 
\begin{equation}\label{pT1T2}
 p=\left\{ \begin{array}{ll} p_{s(T_1),s(T_2)} & h(T) = h(T_1) \\ 1-p_{s(T_1),s(T_2)} & h(T)=h(T_2) 
\end{array}\right. 
\ \ \ \text{ with } \ \ \   p_{s(T_1),s(T_2)}= \bP(s(T_1),s(T_2)) \, . 
\end{equation}
We can then obtain an embedding $\cI(T)$ of $T$ inside $\cS$ in the following way.
First the function $s: \cS\cO_0 \to \cS$ determines a position $s(\lambda(\ell))$ in $\cS$ for
every leaf of $T$, with $\lambda(\ell)$ the label in $\cS\cO_0$ assigned to the leaf $\ell\in L(T)$.
Note that the same lexical item $\lambda\in \cS\cO_0$ may be
assigned to more than one leaf in $L(T)$ so that this assignment $\ell \mapsto s(\lambda(\ell))$
is not always an embedding of $L(T)$ in $\cS$. For each pair $\ell,\ell' \in L(T)$ that are adjacent in $T$
the syntactic object 
$$ T_{v_{\ell,\ell'}} = \Tree[ $\ell$ $\ell'$ ] $$
with $v_{\ell,\ell'}$ the vertex above the leaves $\ell, \ell'$, 
is in ${\rm Dom}(h)$, since $T$ is, and \eqref{sT} assigns to it a point in $\cS$ on the
geodesic arc between $s(\lambda(\ell))$ and $s(\lambda(\ell'))$, where these two points
are distinct since $T_{v_{\ell,\ell'}}\in {\rm Dom}(h)$. We then obtain embeddings of all
the subtrees $T_{v_{\ell,\ell'}}$ in $\cS$ by taking the image $\cI(T_{v_{\ell,\ell'}})$
to consist of the geodesic arc $t s(\lambda(\ell)) + (1-t) s(\lambda(\ell'))$ with $t\in [0,1]$
with root at the point $s(T_{v_{\ell,\ell'}})$.

We proceed similarly for the subsequent steps
of the construction of $T$ in the $\cS\cO$ magma, by obtaining the image $\cI(T_v)$ of
a subtree $T_v$ as the union of the images $\cI(T_{v,1})$ and $\cI(T_{v,2})$, where
$T_v =\fM(T_{v,1}, T_{v,2})$, and the geodesic arc between $s( T_{v,1} )$ and $s( T_{v,2} )$
with root vertex at $s(T_v)$. 
The images $\cI(T)$ of trees $T\in {\rm Dom}(h)$ constructed are
in general immersions rather than simply embeddings because of the possible
coincidence of the points assigned to some of the leaves, as well as because of
possible intersections of the geodesic arcs at points that are not tree vertices.
Both of these issues can be readily resolved to obtain embeddings. Indeed,
the semantic space $\cS$ will be in general high dimensional. As long as it
is of dimension larger than two, crossings of strands of a diagram can be
eliminated by a very small perturbation. In the case of leaves carrying the
same lexical item, one can argue that the different context (in the sense
of the different subtree) in which the item appears will naturally slightly modify
its semantic location in $\cS$. This can be modeled by a small movement
of the endpoints of the geodesic arc to the interior of the arc (which functions
as modifier of the semantic proximity relations). This deforms the
immersions to embeddings. 
\endproof

\smallskip

The assumption that the function $\bP$ that measures semantic relatedness
has values $0< \bP(s,s') <1$ whenever $s\neq s'$ means that we model a situation
where different points in the semantic space $\cS$ are never completely semantically
disjoint or entirely coincident. In such a semantic space model, even an apparently
``nonsensical" pair would not score $0$ under the function $\bP$, so that,  for example,  different 
locations in $\cS$ would distinguish ``colorless green" from ``colorless red", 
as different (mental) images of (absence of) green rather than red color. The fact that the expression
is semantically awkward would correspond to a small (but non-zero) value of $\bP(s,s')$
that affects the (metric) shape of the resulting image tree (that in the geometric
setting we describe in \S \ref{AssocSec} below will end up located very near a 
boundary stratum of the relevant moduli space).

On the other hand, if we allow for the possibility that
$\bP(s,s')=0$ or $\bP(s,s')=1$, for some pairs $s\neq s'$, the construction of Proposition~\ref{treesinS}
would no longer yield an embedding, since for  lexical items mapped to such pairs the
root of the associated Merge tree would map to one or the other leaf rather than to an
intermediate point. Such models will result in certain syntactic trees being mapped to
degenerate image trees in $\cS$, that are located not just near, but on the boundary strata of the
moduli spaces we introduce in \S \ref{AssocSec} below. 
Such cases should also be taken into consideration. (We will see the
relevance of this in the context of Pietroski's semantics in \S \ref{PietSec} below.)
Here we focus on models where this situation can be avoided.

\smallskip

It is important to note that the image of the syntactic trees 
$T\in  {\rm Dom}(h)\subset \fT_{\cS\cO_0}$ inside the semantic space $\cS$ 
is like a static photographic image, rather than a dynamical computational
process. Indeed, {\em all} computational manipulations of syntactic objects
are performed by Merge on the syntax side of the interface, not inside
the space $\cS$, which does not have on its own a computational structure.
The only property of $\cS$ that is used to obtain an embedded
copy of the syntactic tree are proximity relations (here realized in the form of geodesic
convexity). 

\smallskip

In particular, given that the construction above determines an embedding of 
syntactic trees in semantic spaces, one can consider the {\em inverse
problem} of reconstructing syntactic objects and the action of Merge from
their image under this embedding. In other words, given enough measurements 
of semantic proximities in text, can we reconstruct the underlying generative process of
syntax?  Since the computational mechanism of syntax is not directly acting on
semantic spaces, and one is only able to see the embedding of the syntactic
objects, it is reasonable to expect that this inverse problem (reconstructing
the map $\tilde\phi$ of the syntax-semantics interface from the embedding $\cI$)
could be, and we suspect probably is, computationally hard. (See
\cite{Manning} for recent work that in a certain sense attempts to
solve this problem, but not within the explicit framework we describe here.)
We will return to discuss another instance
of this problem, in the context of large language models, in section
\S \ref{TransfSec}. 
In the next section, we further discuss the image of syntax inside semantics
and its relation to the Externalization of free symmetric Merge.

\medskip

\section{Head functions, moduli spaces, associahedra, and Externalization}\label{AssocSec}

We now revisit the simple model of \S \ref{PvalPhiSec}, with the recursive construction
of semantic values associated with trees in the domain of a head function. We view
here the same construction in terms of points in a moduli space of metric trees
introduced in \cite{DevMor}, related to moduli spaces of real curves of genus
zero with marked points. We will show that this viewpoint provides further
insight into the geometry of an Externalization process that introduces language-dependent
planarization of the syntactic trees, and the interaction between the core generative
process of free symmetric Merge and the Conceptual-Intentional system (the syntax-semantics
interface), and an Externalization mechanism that interfaces the same core
computational process with the Articulatory-Perceptual or Sensory-Motor system. 
This will provide a more careful and elaborate explanation of the viewpoint we sketched 
in the Introduction regarding independence of the syntax-semantics interface and
Externalization. The relation between these two mechanisms can also be approached
in a geometric form. 

\smallskip
\subsection{Preliminary discussion}\label{PrelimAssocSec}

In the formulation of Minimalism in terms of the free symmetric Merge
as the core computational mechanism, as presented in \cite{CSBFHKMS}
and formalized mathematically in our previous work \cite{MCB},
\cite{MBC}, the generative process of syntax produces hierarchical
structures through syntatic objects and the action of Merge on
workspaces (formalized in \cite{MCB} in terms of the Hopf algebra
$\cH$ of binary rooted forests with no assigned planar structure). A
mechanism of {\em Externalization} takes place after this generative
process. This mechanism describes the connection to the Sensory-Motor
system, that due to its physical and physiological nature externalizes
language in the form of a temporally ordered sequence of words,
realized as sounds or signs or writing (or, inversely, for
parsing). The necessity of temporal ordering in the Externalization of
language requires a {\em planarization} of the binary rooted trees
(syntactic objects), as the choice of a planar structure is equivalent
to the choice of an ordering of the leaves. This choice of
planarization is subject to language-dependent constraints, through
the syntactic parameters of languages.  In \cite{MCB} we proposed a
mathematical formalism for Externalization based on a suitable notion
of correspondences.

\smallskip

In the previous sections of this paper, we have analyzed possible models (some
of them highly simplified) of how the products of the free Merge generative process
of syntax can be mapped to semantic spaces, where the main property of
semantic space we have used is a notion of topological/metric 
proximity. This type of mapping of syntax to semantics is designed to directly
apply to the hierarchical structures produced by the free symmetric Merge, without
having to first pass through the choice of a planar structure as is done in the
externalization process. This mapping to semantic spaces represents the
interaction between the core computational mechanism of Merge with the 
Conceptual-Intentional system.

\smallskip

These two mechanisms are illustrated as the two top arrows depicted in Figure~\ref{SemExtFig}.
This part of the picture corresponds to property (3) on the list in \S \ref{RequireListSec}, that
semantic interpretation is, to a large extent, independent of Externalization.  However, obviously
the Externalization process and the mapping to semantic spaces need to be compatibly combined,
as figure Figure~\ref{SemExtFig} suggests. The goal of the rest of this section is to introduce the
mathematical framework in which both processes simultaneously coexist. 

\smallskip

\begin{figure}[h]
\begin{center}
\includegraphics[scale=0.2]{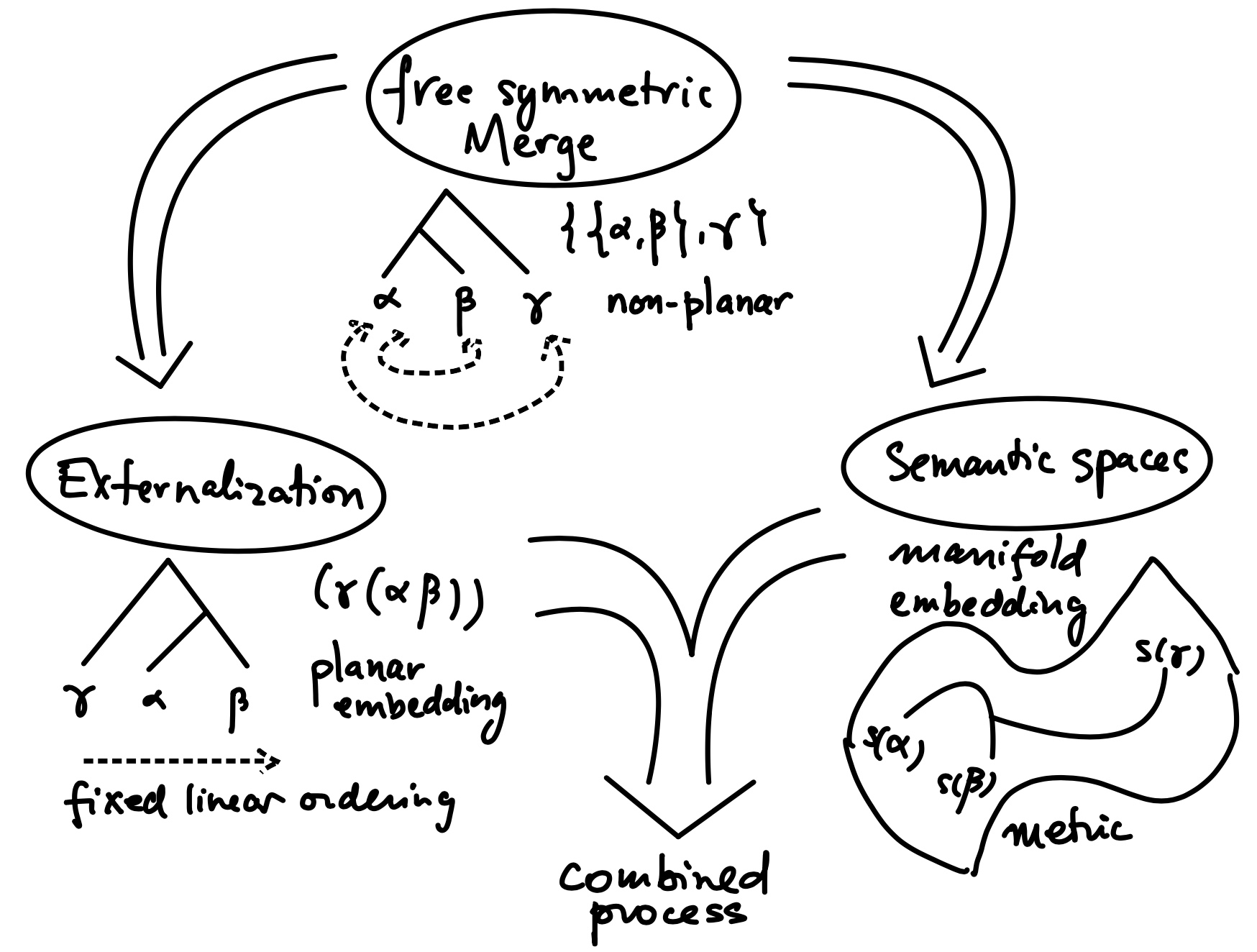} 
\caption{Free symmetric Merge, Externalization, and Semantic Spaces. \label{SemExtFig}}
\end{center}
\end{figure}

\smallskip

We proceed in the following way. First we introduce a framework designed for the
comparison of different planar structures on the same abstract binary rooted tree.
Since the planarization of Externalization is language dependent, we need a space
where different planarization can be cosidered. Such a space is well studied in
mathematics and is called the {\em associahedron}. We recall its properties in \S \ref{ModSpSec}.
At the same time, we want to keep track of the fact that the hierarchical structures produced
by the free symmetric Merge have also acquired a metric structure through its mapping
to semantic spaces, where this metric structure keeps track of information about
semantic relatedness, across substructures. This assignment of metric data on (non-planar)
binary rooted trees is also described by a well known mathematical object, the BHV moduli
space, that we also discuss in \S \ref{ModSpSec}. 

\smallskip

Thus, we present a formulation where, taken separately (as in the top arrows of Figure~\ref{SemExtFig})
the Externalization and the mapping to semantic spaces result, respectively, in the assignment to
a given syntactic object $T\in \cS\cO$ with $n$ leaves of a vertex in the $K_n$ associahedron,
and of a point in the BHV$_n$ moduli space. 

\smallskip

\begin{figure}[h]
\begin{center}
\includegraphics[scale=0.2]{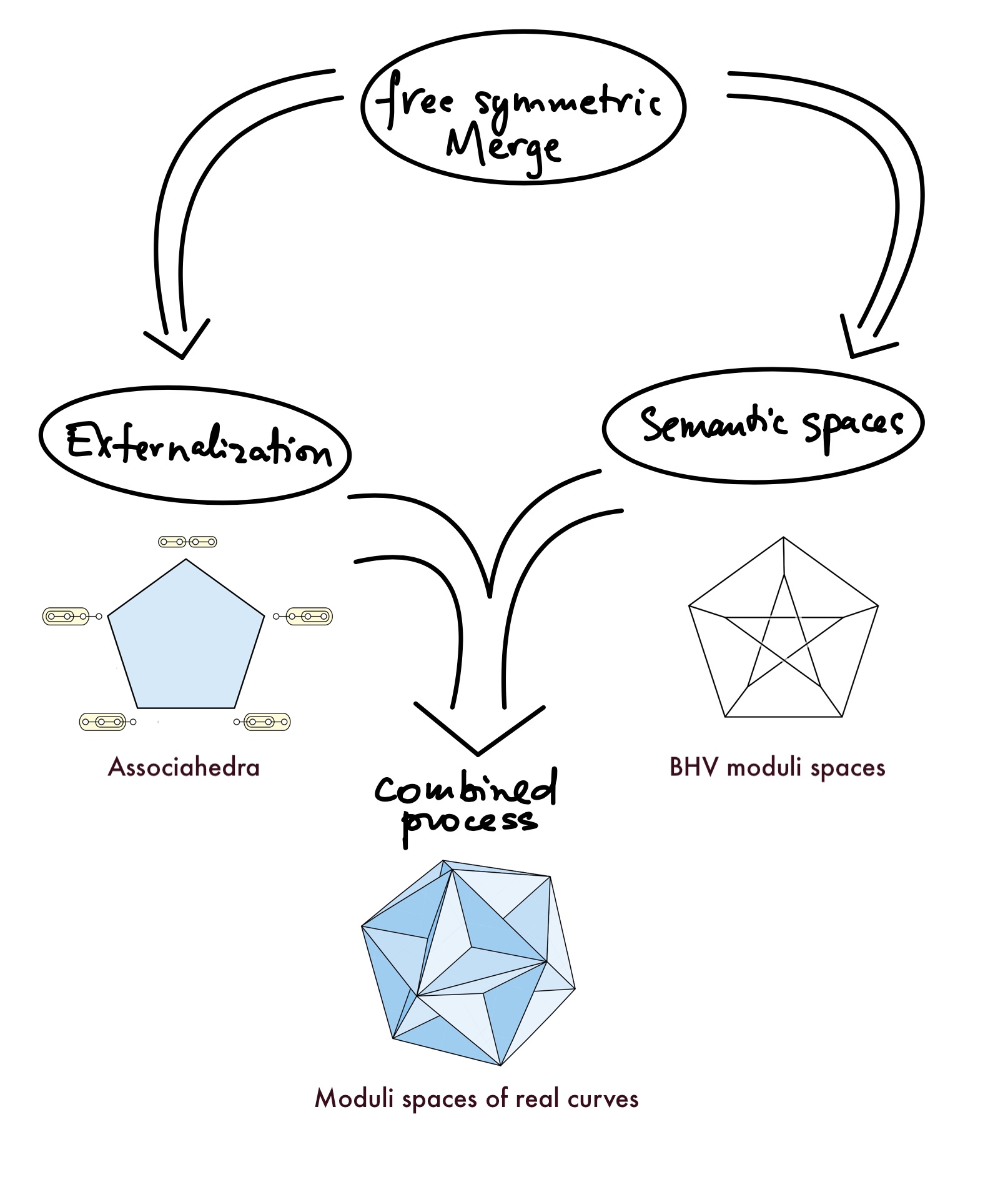} 
\caption{Free symmetric Merge, Externalization, and Semantics, and the respective moduli spaces. \label{SemExtFig2}}
\end{center}
\end{figure}

\smallskip

These two geometric objects, the associahedron and the BHV moduli space, naturally combine
into another space, which accounts for what happens when we enrich the combinatorial
associahedron with metric data. This is again a geometric object that is very well
known in mathematics, where it is identified with a certain moduli space of curves, $\bar M_{0,n}^{or}(\R)$.
We review in \S \ref{ModSpSec} the relation between these three fundamental spaces 
$K_n$, BHV$_n$, and $\bar M_{0,n}^{or}(\R)$, see Figure~\ref{SemExtFig2}.

\smallskip

In the subsequent sections \S \ref{PvalTModSpSec} and
\S \ref{OrigamiSec} we explain more in detail how the mapping to semantic spaces and
Externalization can be seen in this perspective. We include a discussion of
how Kayne's LCA algorithm, Cinque's sbtract functional lexicon, and constraints 
implemented by syntactic parameters appear in this formulation.


\smallskip
\subsection{Associahedra and moduli spaces of trees and curves} \label{ModSpSec}

We recall here some general facts about moduli spaces of abstract and
planar binary trees, and their relation to the moduli space of genus zero real curves with
marked points. For a more detailed account we refer the reader to \cite{BHV}, \cite{BoVo},
and \cite{DevMor}.

\smallskip

The Stasheff associahedron $K_n$ is a convex polytope of dimension $n-2$, where
the vertices correspond to all the balanced parentheses insertions on an ordered string 
of $n$ symbols (equivalently, all planar binary rooted trees on $n$ leaves) 
and the edges are given by a single application of the associativity rule.
For example the $1$-dimensional associahedron $K_3$ is the graph
with a single edge and two vertices 
$$ ((ab)c) \longleftrightarrow (a(bc)) \, . $$
The $2$-dimensional associahedron $K_4$ is similarly a pentagon, while 
the $3$-dimensional associahedron $K_5$ is illustrated in Figure~\ref{AssocFig}.
Faces of the associahedron $K_n$ are products of lower dimensional associahedra. 
These strata $K_{n_i}$ correspond to the degeneration of a binary tree where some
of the internal vertices acquire higher valencies. The description in terms of planar
binary rooted trees has an equivalent formulation in terms of triangulations of an
$n+1$-gon by drawing diagonals.

\begin{figure}[h]
\begin{center}
\includegraphics[scale=0.35]{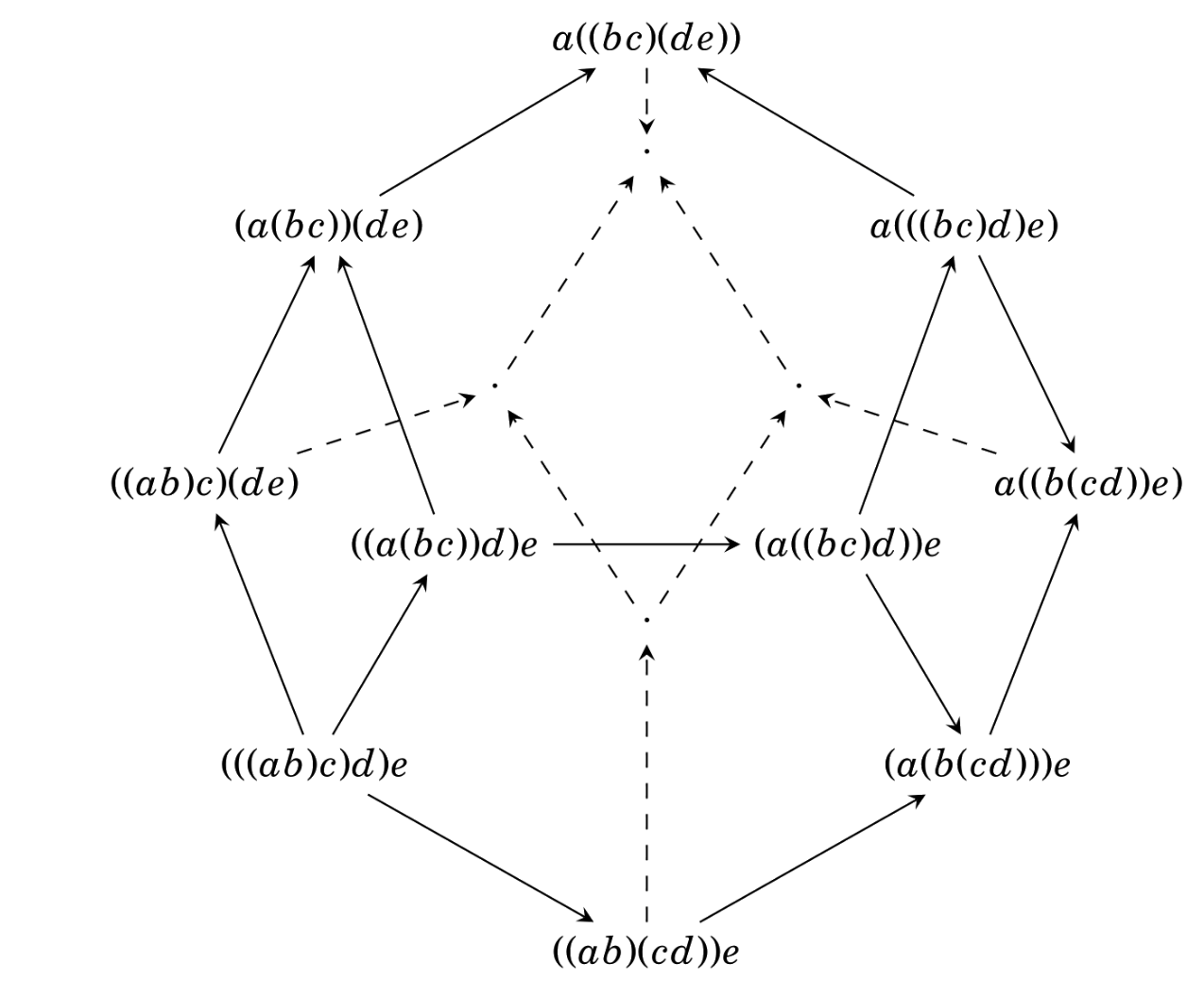} 
\includegraphics[scale=0.35]{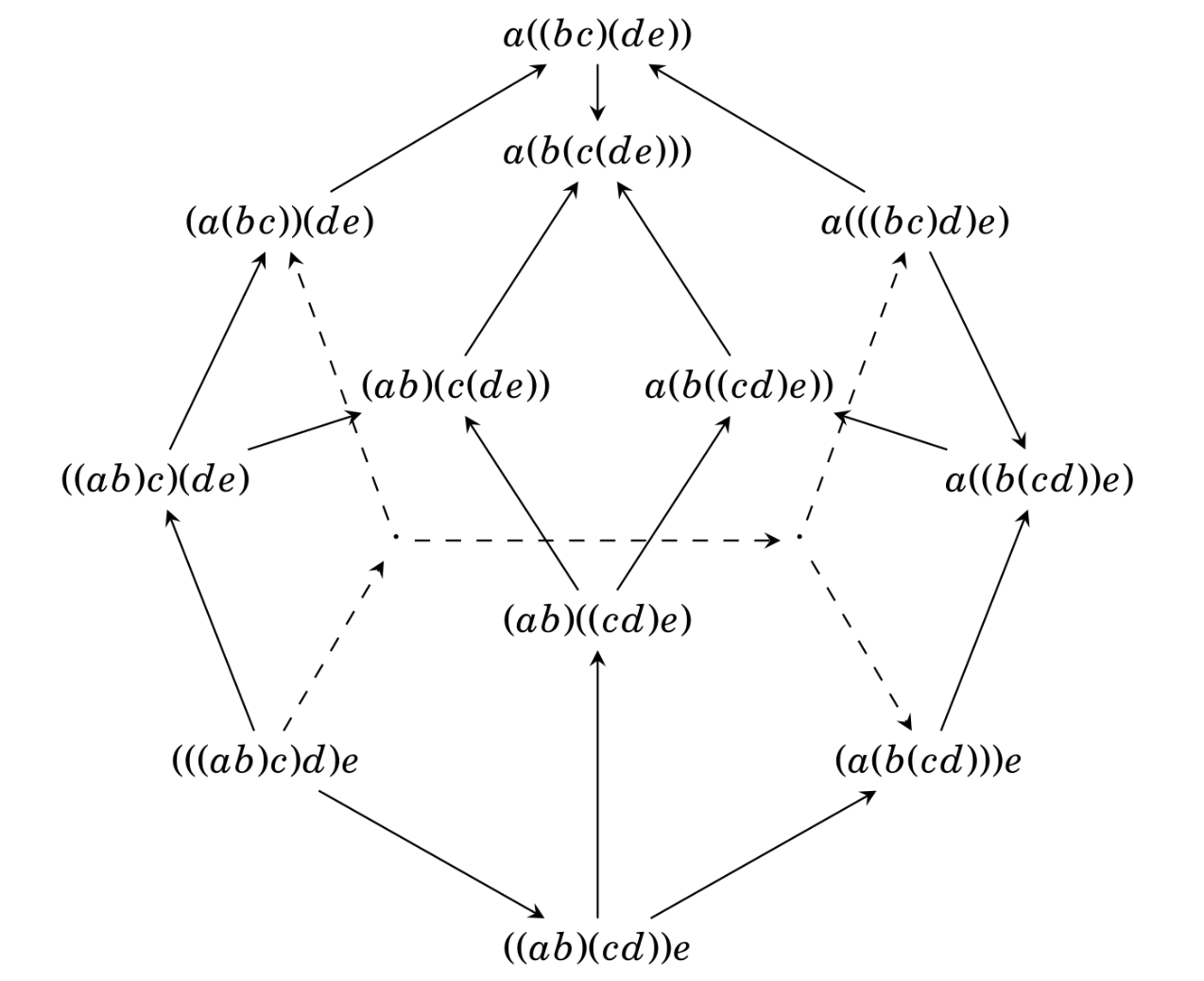} 
\caption{The Stasheff associahedron $K_5$, front and back view. \label{AssocFig}}
\end{center}
\end{figure}

\smallskip

Boardman and Vogt \cite{BoVo} showed that the associahedron $K_n$ can
be decomposed into $C_{n-1}$ cubes of dimension $n-2$, where $C_{n-1}$
is the Catalan number 
$$ C_{n-1} = \frac{1}{n} \binom{2n-2}{n-1} \, . $$
The decomposition of the associahedron $K_4$ is illustrated in Figure~\ref{FigK4}.

\smallskip

Each vertex of the associahedron can be identified with a {\em planar} binary rooted tree.
A way to interpret the polytope points here is as metric structures on planar binary rooted trees
that assign weights in $\R_\geq 0$ to the internal edges of the tree,
with degeneracies along the faces and vertices of the cubic decomposition, see 
Figure~\ref{FigK4} for $K_4$. Each cube in the decomposition parameterizes the
(normalized) choices of weights for the internal edges for the planar tree structure 
associated to that cube, and the faces are glued according to the transitions from
one tree structure to an adjacent one, as dictated by the associahedron structure. 

\begin{figure}[h]
\begin{center}
\includegraphics[scale=0.5]{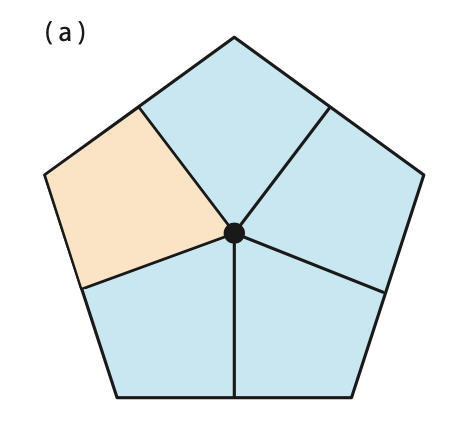} 
\includegraphics[scale=0.5]{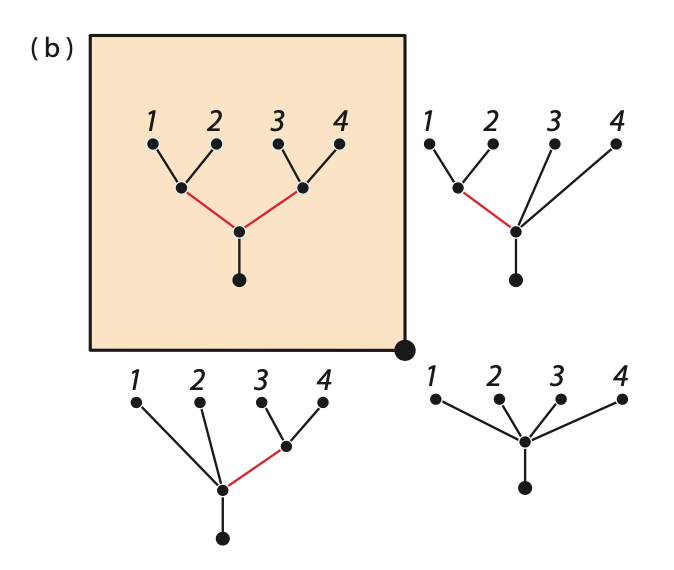} 
\caption{The Stasheff associahedron $K_4$ with its cubic decomposition (a), and parameterization 
of planar metric binary rooted trees (b), figures by Satyan Devadoss from \cite{DevMor}. \label{FigK4}}
\end{center}
\end{figure}

It was further shown by Devadoss and Morava \cite{DevMor} that
the parameterization of planar binary rooted trees with weights on the
internal edges in terms of the (open cells of the) associahedron and its cubic decomposition
can then be related to compactifications $\overline{M}_{0,n+1}(\R)$ of moduli spaces
of real curves of genus zero with $(n+1)$-marked points. 
The key idea here is that the ordered leaves of a planar binary rooted trees can be embedded
as an ordered set of points in the real line, where the coordinates of the points are obtained
from the weights assigned at the internal edges of the tree as a function $e^{-W}$ of the sum $W$ of
the weights along the path from the root to one of the leaves (see the example in Figure~\ref{treesM0nFig}).  
Note that, while the open cells of the associahedron correspond to binary trees, the boundary strata
of these cells contain trees with higher valences (corresponding to the limits of binary trees
when one or more of the edge lengths go to zero). Since the trees coming from syntax
are binary (see \cite{MCB} for our discussion on why Merge operators with higher arity
are excluded) the image from syntax will lie inside the open cells. The boundary structure
is still important though, because boundaries of cells in the associahedron encode all the
possible structural changes to the underlying hierarchical structures (syntactic objects).

\begin{figure}[h]
\begin{center}
\includegraphics[scale=0.5]{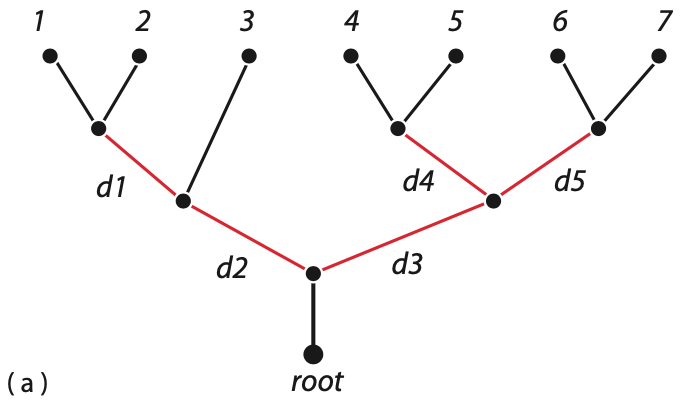}  \\ \bigskip
\includegraphics[scale=0.5]{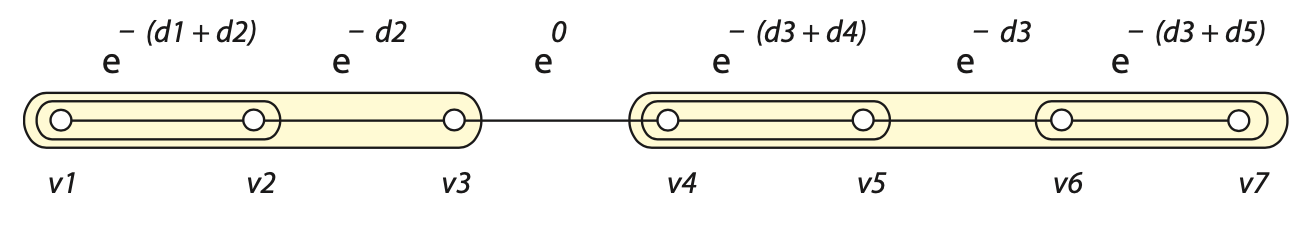} 
\caption{A planar binary rooted tree with weighted internal edges (a), and the
associated ordered configuration of points on the real line, figures by Satyan Devadoss from \cite{DevMor}. \label{treesM0nFig}}
\end{center}
\end{figure}

As shown by Devadoss in \cite{Deva}, the orientation double cover 
$\overline{M}^{or}_{0,n+1}(\R)$ of $\overline{M}_{0,n+1}(\R)$ can
be decomposed into a collection of $n!$ copies of the associahedron $K_n$,
where the $n!/2$ associahedra of $\overline{M}_{0,n+1}(\R)$ correspond
to the permutations of the $(n+1)$ points on the real line preserving the cyclic
order of $\{ 0, 1, \infty \}$, with gluings corresponding to certain twist operations
on the triangulated $(n+1)$-gons (see Figure~\ref{M04Fig} for the example of $n=3$.
Note that for $n\leq 3$, the moduli space $\overline{M}_{0,n+1}(\R)$ is orientable so
one does not see the role of the orientation double cover; see
\cite{DevMor} for a more
detailed discussion of the more general case). 

\begin{figure}[h]
\begin{center}
\includegraphics[scale=0.65]{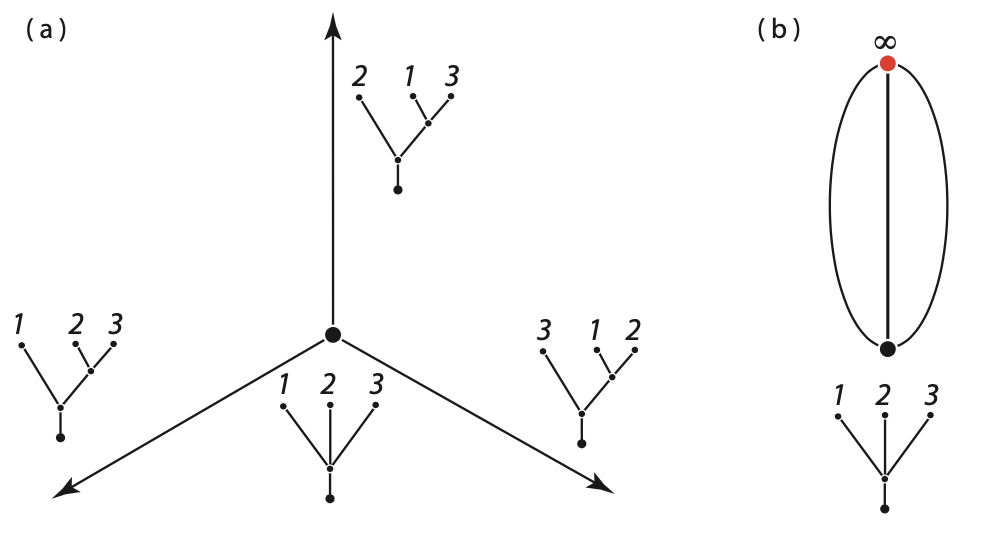}  
\caption{The moduli space ${\rm BHV}_3$ of abstract binary rooted trees 
and its one-point compactification ${\rm BHV}_3^+$, figure by Satyan Devadoss from \cite{DevMor}. \label{BHVmodFig}}
\end{center}
\end{figure}

One can also consider the moduli space ${\rm BHV}_n$ of abstract binary rooted trees with 
$n$ leaves (with no assigned planar structure) along with weighted internal edges, and their
one-point compactification ${\rm BHV}^+$, constructed by Billera, Holmes, and Vogtmann, \cite{BHV}.
The moduli space ${\rm BHV}_n$  is obtained by considering all the $(2n-3)!!$ abstract
binary rooted trees with $n$ labeled leaves. All these trees have $n-2$ internal edges.
For each tree, one considers an orthant $\R_{\geq  0}^{n-2}$, which represents all
the possible choices of a weight (length) for the internal edges. These orthants are
glued along the common faces (which correspond to shrinking one of the internal edges)
and this gives the space ${\rm BHV}_n$.
The link $\cL_n$ of the origin in ${\rm BHV}_n$ is an $(n-3)$-dimensional
simplicial complex. In the case $n=3$ it consists of three points. For $n=4$ it is
the Peterson graph of Figure~\ref{L4Fig}. In general, there are $(2n-3)!!$  
top $(n-3)$-dimensional simplexes  of $\cL_n$ (e.g.~15 edges in the case of $\cL_4$)
that correspond to the different trees, and two of them share a face when the corresponding
trees give rise to the same quotient tree when contracting an internal edge. 

\begin{figure}[h]
\begin{center}
\includegraphics[scale=0.35]{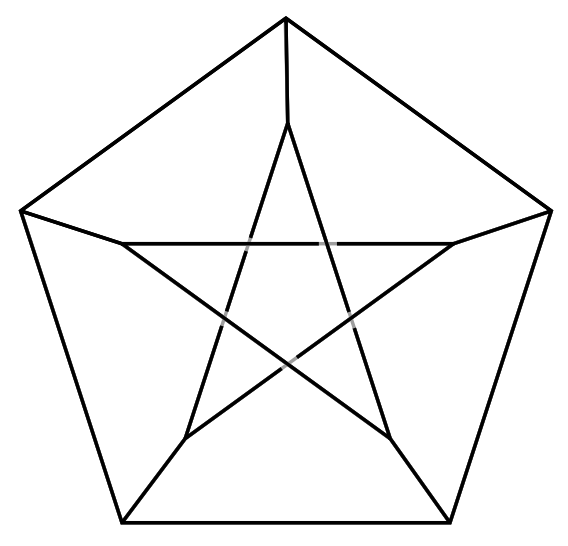}  
\caption{The Peterson graph is the link $\cL_4$ of the origin in ${\rm BHV}_4$. \label{L4Fig}}
\end{center}
\end{figure}

It is shown in \cite{DevMor} that there is a projection map between these moduli spaces,
\begin{equation}\label{projPin}
 \Pi_n: \overline{M}^{or}_{0,n+1}(\R) \twoheadrightarrow {\rm BHV}_n^+\, , 
\end{equation} 
with a finite projection that is generically $2^{n-1}$-to-$1$, obtained by an origami
folding of the cubes of the cubical decomposition of the associahedra in 
$\overline{M}^{or}_{0,n+1}(\R)$, according to the formula
$$ n! \cdot C_{n-1} = 2^{n-1} \cdot (2n-3) !! \, , $$
where the left-hand-side lists the $C_{n-1}$ cubes of the $n!$ associahedra of
$\overline{M}^{or}_{0,n+1}(\R)$, and the right-hand-side lists the $(2n-3)!!$ 
 simplexes of dimension $(n-3)$ of $\cL_n$, and $2^{n-1}$ is the multiplicity of
the generic fibers of the projection map. Note that $2^{n-1}$ is the number of
different planar structures for a given abstract binary rooted tree on $n$ leaves,
since such a tree has $n-1$ non-leaf vertices and the total number of
planar embeddings can be obtained by choosing one of two possible
planar embeddings (left/right) for each pair of edges below a given non-leaf vertex.
The origami folding quotient takes each $(n-2)$-dimensional cube and folds it in half in
each direction, obtaining $2^{n-2}$ foldings, with $2$ copies of each cube in the
orientation double cover, so that one obtains $2^{n-1}$ points in each general fiber.
We will see this more explicitly in \S \ref{PvalTModSpSec}, applied to our setting.

\smallskip
\subsection{Head functions, convex semantic spaces, and metric trees} \label{PvalTModSpSec}

With these facts in hand, now consider again the setting we discussed
in our simple example of \S \ref{PvalPhiSec}.

\smallskip

Consider the
set of all $(2n-3)!!$ abstract binary rooted trees with $n$ labeled leaves. Suppose that
the leaves are labeled by a given (multi)set $\{ \lambda_i \}_{i=1}^n$ of lexical items
and syntactic features in $\cS\cO_0$. If this is a multiset instead of a set, we still
interpret the multiple copies of a given item in $\cS\cO_0$ as {\em repetitions}, not
as {\em copies}, in the sense that they can play different roles in structure formation
via applications of Merge--hence we will still regard them as distinct labels. Thus,
we have the following geometric description of our data.
\begin{itemize}
\item For any choice of the lexical items associated to the leaves, we obtain a corresponding
copy of the moduli space ${\rm BHV}_n$. 
\item The link of the origin $\cL_n\subset {\rm BHV}_n$ can
be seen as an assignment of weights to the internal edges that is normalized
(for example by requiring that the total sum of weights is equal to $1$).
\item We write ${\rm BHV}_n(\Lambda)$ for $\Lambda=\{ \lambda_i \}_{i=1}^n$ for the
copy of ${\rm BHV}_n$ that corresponds to the given choice $\Lambda$ of the lexical
items assigned to the leaves. 
\item We similarly write $\cL_n(\Lambda)$ for the
associated copy of $\cL_n$.
\end{itemize}

\smallskip

\begin{prop}\label{headComplex}
The choice of a head function $h$ determines simplicial subcomplexes
$\cL_n(\Lambda,h) \subset \cL_n(\Lambda)$, ${\rm BHV}_n(\Lambda,h) \subset {\rm BHV}_n(\Lambda)$,
$M_n(\Lambda,h) \subset \overline{M}^{or}_{0,n+1}(\R)$, compatible with the
maps relating these moduli spaces, and a lift of $\cL_n(\Lambda,h)$ and ${\rm BHV}_n(\Lambda,h)$
inside $M_n(\Lambda,h)$, determined by the planar structure $\pi_h$ associated to the head function.
\end{prop}

\smallskip

\proof
The choice of the head function $h$ selects, for each of these copies 
$\cL_n(\Lambda)\subset {\rm BHV}_n(\Lambda)$, a simplicial subcomplex 
$\cL_n(\Lambda,h) \subset \cL_n(\Lambda)$ and the associated cone
${\rm BHV}_n(\Lambda,h) \subset {\rm BHV}_n(\Lambda)$, where the set of top 
$(n-3)$-dimensional simplexes of $\cL_n(\Lambda,h)$ corresponds 
to the subset of the given $(2n-3)!!$ trees that belong to ${\rm Dom}(h)$. 

\smallskip

Let $M_n(\Lambda,h) \subset \overline{M}^{or}_{0,n+1}(\R)$ denote the
locus in $\overline{M}^{or}_{0,n+1}(\R)$ obtained as a pre-image under
the projection map of the image ${\rm BHV}_n(\Lambda,h)^+$ of the cone 
${\rm BHV}_n(\Lambda,h)$ in the one-point compactification ${\rm BHV}_n^+$,
\begin{equation}\label{MnLambdah}
M_n(\Lambda,h)  := \Pi_n^{-1} ( {\rm BHV}_n(\Lambda,h)^+ ) \, .
\end{equation}

\smallskip

A point in ${\rm BHV}_n(\Lambda,h)$ is a pair $(T,\underline{\ell})$ of an
abstract binary rooted tree on $n$ leaves labeled by the points of $\Lambda$
together with a set $\underline{\ell}=(\ell_k)_{k=1}^{n-2}$ of weights $\ell_i\in \R_{\geq 0}$
assigned to the internal edges of $T$. The $2^{n-1}$ points in the fiber 
$\Pi_n^{-1}(T,\underline{\ell}) \subset M_n(\Lambda,h)$ are given by the points $(T^\pi, \underline{\ell})$,
where $T^\pi$ ranges over all the possible planarizations $\pi$ of $T$ and the lengths of the
internal edges stay the same. 

\smallskip 

We have seen that the choice of a head function $h$ determines
an associated planar structure $\pi_h$ for all trees $T\in {\rm Dom}(h)$. 
Thus, the choice of a head function determines a lift of the subcomplex
${\rm BHV}_n(\Lambda,h)$ (and in particular of $\cL_n(\Lambda,h) \subset \cL_n(\Lambda)$)
to a subcomplex of $M_n(\Lambda,h) \subset \overline{M}^{or}_{0,n+1}(\R)$. 
\endproof

\smallskip

Consider then, as in  \S \ref{PvalPhiSec}, a semantic space $\cS$ that is a 
geodesically convex subspace of a Riemannian manifold, together with
a map $s: \cS\cO_0 \to \cS$. Assume that, for points in $\cS$, we can 
evaluate the frequency of semantic relatedness in a specified context in terms 
of a biconcave function $\bP: {\rm Sym}^2(\cS)\to [0,1]$.  

\begin{prop}\label{ellfromh}
Let $T \in {\rm Dom}(h)\subset \fT_{\cS\cO_0}$ be a tree with $n$ leaves. 
The data $(s: \cS\cO_0 \to \cS, \bP)$ determine a set 
$\underline{\ell}=(\ell_k)_{k=1}^{n-2} \in \R^{n-2}_{\geq 0}$ 
of weights assigned to the internal edges of $T$. Thus the data 
$(s: \cS\cO_0 \to \cS, \bP)$ determine a point $(T, \underline{\ell}^{(h,s,\bP)}) \in \cL_n(\Lambda,h)$
and a point in the corresponding fiber of the projection from $M_n(\Lambda,h)$.
\end{prop}

\proof
To see this, we proceed as
in  \S \ref{PvalPhiSec}. For each of the $n-2$ vertices $v$ of $T$ that are
neither the root nor one of the leaves, consider the two subtrees $T_{v,1}$ and
$T_{v_2}$ that have root vertices $v_1, v_2$ immediately below $v$, and 
compute $p_v:= \bP(s(T_{v,1}), s(T_{v,2}))$, where $h(T_{v,i})$ is the head
leaf of the subtree $T_{v,i}$. We label the $(n-2)$ internal edges by the target vertex $v$
(where the tree is oriented away from the root) and we take $\ell_v=p_v$.  

\smallskip

Thus, we have that, for a tree $T \in {\rm Dom}(h)\subset \fT_{\cS\cO_0}$ on $n$ leaves labeled
by $\Lambda$, the choice of a head function $h$, together with
the choice of a map $s: \cS\cO_0 \to \cS$ and of the function 
$\bP: {\rm Sym}^2(\cS)\to [0,1]$ determines, after an overall normalization 
of the weights, a point $(T, \underline{\ell}^{(h,s,\bP)}) \in \cL_n(\Lambda,h)$,
and a corresponding point $(T^{\pi_h}, \underline{\ell}^{(h,s,\bP)}) \in \cL_n(\Lambda,h)$
in the fiber above $(T, \underline{\ell}^{(h,s,\bP)})$ in $M_n(\Lambda,h)\subset \overline{M}^{or}_{0,n+1}(\R)$.
\endproof

\smallskip

\begin{rem}\label{Accposition}{\rm 
Note that in  \S \ref{PvalPhiSec} we used the same coordinates $\bP(s(T_{v,1}), s(T_{v,2}))$
to assign points $s(T_v)=p_v s(T_{v,1}) + (1-p_v) s(T_{v,2})$ or
$s(T_v)=p_v s(T_{v,2}) + (1-p_v) s(T_{v,1})$ (according to whether the
head $h(T_v)$ matches the head of either of the two subtrees). Thus, according to this
construction, the weight of an internal edges of $T$ obtained as in Proposition~\ref{ellfromh} 
reflects the positions in the semantic space $\cS$ of the accessible term below that edge.}
\end{rem}

\smallskip

As a result, we can view the construction of the character $\phi_{s, \bP, h}$ of  \S \ref{PvalPhiSec} 
equivalently as the construction of a section.

\begin{cor}\label{phisigma}
The construction of the character $\phi_{s, \bP, h}$ of  \S \ref{PvalPhiSec} is 
equivalent to the construction of a partially defined section 
\begin{equation}\label{sigmaSsec}
 \sigma_{s, \bP, h,n} : {\rm BHV}_n \to \overline{M}^{or}_{0,n+1}(\R) 
\end{equation} 
which is defined over
$$ {\rm Dom}(\sigma_{s, \bP, h,n})={\rm BHV}_n(\Lambda,h)\, , $$
and a partially defined map 
\begin{equation}\label{sPhmap}
s_{\bP, h}: \fT_{\cS\cO_0} \to \cup_n \cL_n(\Lambda,h)
\end{equation}
with ${\rm Dom}(s_{\bP, h})={\rm Dom}(h)$. The construction of the character $\phi_{s, \bP, h}$ 
of  \S \ref{PvalPhiSec} is equivalent to the construction of the composite map
$\sigma_{s, \bP, h}  \circ s_{\bP, h}$. 
\end{cor}

\smallskip

For the case $n=3$, the projection maps are illustrated in Figure~\ref{M04Fig}.

\begin{figure}[h]
\begin{center}
\includegraphics[scale=0.65]{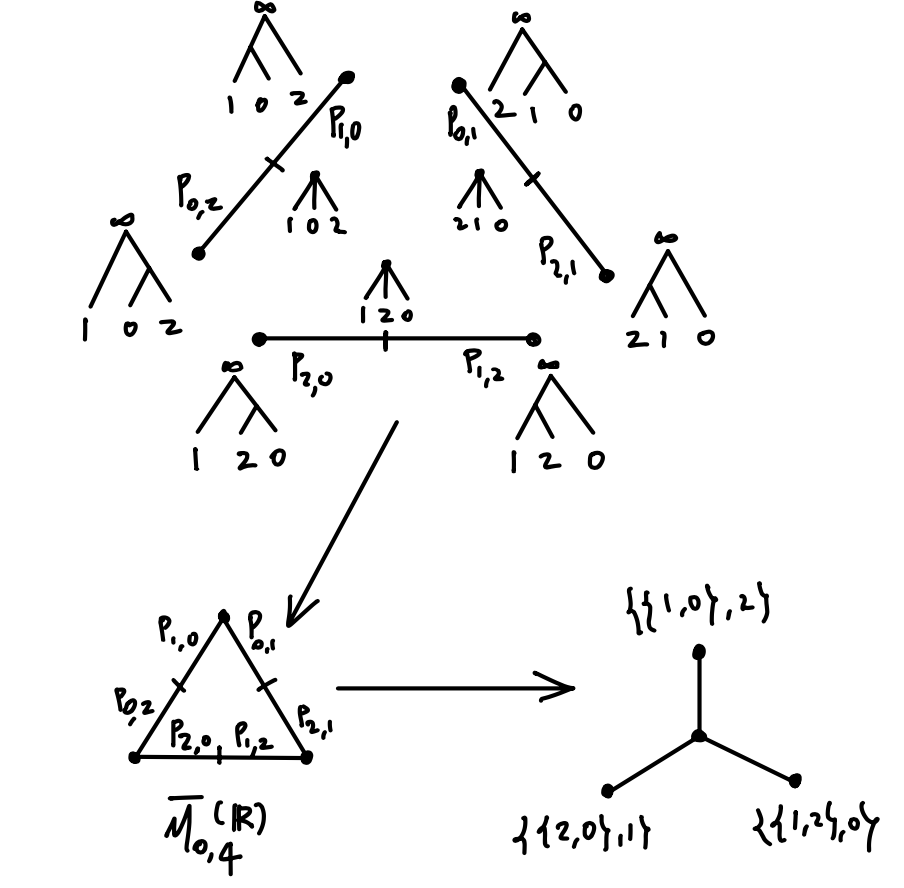}  
\caption{The projections from three associahedra $K_3$ to the moduli space $\overline{M}_{0,4}(\R)$ and to
the ${\rm BHV}_3^+$ moduli space and the embedding map $\cI$ of syntactic trees to semantic space $\cS$ seen
from the point of view of moduli spaces.\label{M04Fig}}
\end{center}
\end{figure}

We have described the construction here in terms of the simple model of assignment
of semantic values to syntactic objects described in \S \ref{PvalPhiSec}. This can be
adapted to other models, so that we can incorporate, as part of the modeling of
the syntax-semantics interface, the construction of a partially defined section
\begin{equation}\label{semasec}
 \sigma_{\cS,n} : {\rm Dom}(\sigma_\cS)\subset {\rm BHV}_n \to \overline{M}^{or}_{0,n+1}(\R) 
\end{equation} 
which depends on the model of semantic space $\cS$ used and on its  properties.
Similarly, the map \eqref{sPhmap} can be generalized as a map
\begin{equation}\label{sShmap}
s_{\cS, h,n}: \fT_{\cS\cO_0} \to  \cL_n\cap {\rm Dom}(\sigma_{\cS,n})\, . 
\end{equation}

\smallskip
\subsection{Origami folding and Externalization}\label{OrigamiSec}

In \cite{MCB} we gave an account of Externalization as a section of
the projection from planar to abstract binary rooted trees, where the
section is language dependent and is chosen so that the resulting
planar structure is compatible with certain syntactic parameters, through
the effect these have on word order. 

\smallskip

In terms of the geometry of moduli spaces described here, one can
similarly view Externalization as the choice of a {\em section}
(depending on a specified language $L$ through its syntactic parameters)
\begin{equation}\label{sigmaExt}
\sigma_{L,n} : {\rm BHV}_n \to \overline{M}^{or}_{0,n+1}(\R) \, .
\end{equation}
This section is defined at the level of the combinatorial trees,
as a choice of a section $\sigma_{L,n}: \fT_{\cS\cO_0,n} \to \fT_{\cS\cO_0,n}^{pl}$
that assigns a planar structure, as discussed in \cite{MCB}, and extended
to metric trees as the identity on the metric datum $\underline{\ell}$, since 
Externalization is decoupled from the metric structure, reflecting our initial
assumption on independence of semantic values from Externalization.
This independence assumption only affects this independence of $\sigma_{L,n}$
on the metric structure. It does not mean that there would be {\em no} interaction
with the semantics channel. One way to model such interaction is by comparing
the two sections $\sigma_{L,n}$ and $\sigma_{\cS,n}$ on the subdomain 
${\rm Dom}(\sigma_{\cS,n})\subset {\rm BHV}_n$ where both are defined
and in particular on the target of the map $s_{\cS, h,n}$ of \eqref{sShmap}.

\smallskip
\subsection{An example}

All the above discussion on the relation between Externalization and
the syntax-semantics interface in terms of moduli spaces is quite
abstract, so let us illustrate what is happening with a very simple
example. Consider a sentence such as
$$ \Tree[ [ yellow flowers ] [ bloom early ] ] =\Tree[ [ $\alpha$ $\beta$ ] [ $\gamma$ $\delta$ ] ] = ((\alpha \,\, \beta)\,(\gamma \,\, \delta)) $$
In the form depicted, this is represented by a planar binary rooted tree on four leaves
labeled by the lexical items in the set $\Lambda=\{ \alpha, \beta, \gamma, \delta \}$. 
The tree does not contain exocentric constructions and has a well defined syntactic head.
Thus, we have the associated data $(\Lambda, h)$ as above. The underlying syntactic
object, as produced by a free symmetric Merge. is the {\em non-planar} abstract binary rooted tree
$$ \Tree[ [ $\alpha$ $\beta$ ] [ $\gamma$ $\delta$ ] ] =\Tree[ [ $\beta$ $\alpha$  ] [ $\gamma$ $\delta$ ] ] =\Tree[ [ $\alpha$ $\beta$ ] [ $\delta$  $\gamma$ ] ] =
\Tree[ [ $\gamma$ $\delta$ ] [ $\alpha$ $\beta$ ]] = \{ \{ \alpha, \beta\} , \{ \gamma, \delta \} \} \, . $$
The planar tree $((\alpha \,\, \beta)\,(\gamma \,\, \delta))$ corresponds to a vertex of the associahedron $K_4$, as in Figure~\ref{AssocK4VFig}.
The associahedron considered is one of the $4!=24$ associahedra that correspond to the $4!$ permutations of the leaves' labels.
This assignment of a vertex on one of the $24$ associahedra corresponds to left arrow (free symmetric Merge to Externalization) 
in the top part of Figure~\ref{SemExtFig2}, for this example. 

\begin{figure}[h]
\begin{center}
\includegraphics[scale=0.3]{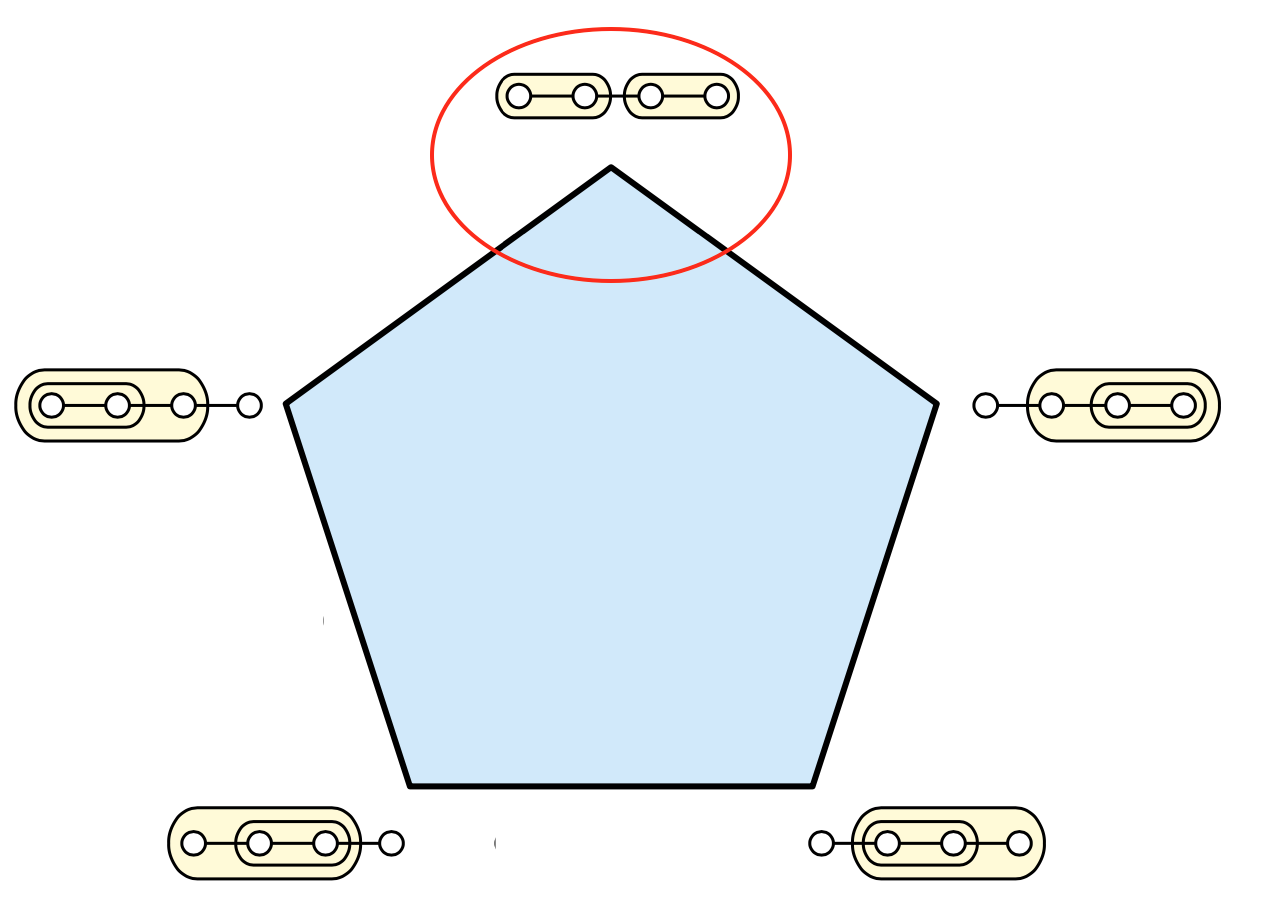}  
\caption{The selected vertex of the associahedron $K_4$ corresponding to the planar tree $((\alpha \,\, \beta)\,(\gamma \,\, \delta))$ (modified 
figure by Satyan Devadoss from \cite{DevMor}). \label{AssocK4VFig}}
\end{center}
\end{figure}

The abstract tree $\{ \{ \alpha, \beta\} , \{ \gamma, \delta \} \}$ produced by the free symmetric Merge,  on the other hand, 
is one of the $15 =(2n-3)!!$, for $n=4$, possible abstract binary rooted trees on four labeled leaves. These $15$ possible
trees correspond to the $15$ edges of the link $\cL_4$ of the origin in the moduli space $BHV_4$. Thus, the syntactic
object $\{ \{ \alpha, \beta\} , \{ \gamma, \delta \} \}$ selects one of these edges, see Figure~\ref{L4graphEFig}.

\begin{figure}[h]
\begin{center}
\includegraphics[scale=0.36]{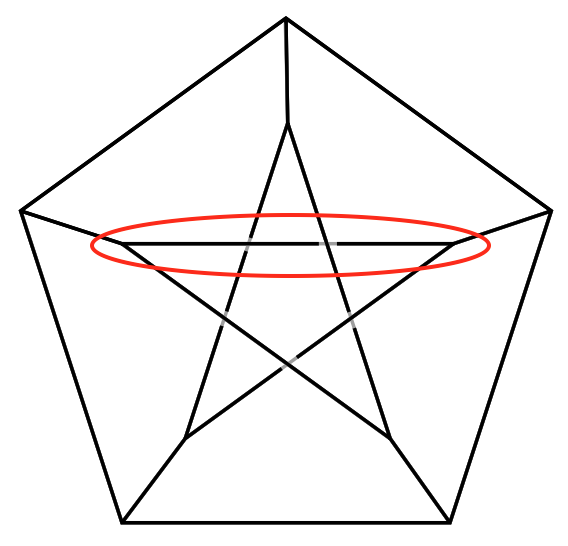}  
\caption{The selected edge in the link $\cL_4$ of the origin in the moduli space BHV$_4$ corresponding to the abstract tree $\{ \{ \alpha, \beta\} , \{ \gamma, \delta \} \}$. \label{L4graphEFig}}
\end{center}
\end{figure}

Now suppose we have chosen a semantic space $\cS$ (for simplicity of
discussion, consider using a vector space model, though
it is not necessary for $\cS$ to be of this kind).  Each of the four
lexical items has a representation
$s(\alpha), s(\beta), s(\gamma), s(\delta) \in \cS$. The two semantic
relatedness measures $u_1=\bP(s(\alpha), s(\beta))$ (relating
``yellow" and ``flower") and $u_2=\bP(s(\gamma), s(\delta))$ (relating
``blooming" and ``early") in $\cS$ provide two real coordinates
associated with the accessible terms $\{ \alpha, \beta\}$ and
$\{ \gamma, \delta \}$, respectively. These two coordinates fix
a point $(u_1,u_2)\in [0,1]^2$ in a square (see
Figure~\ref{L4graphSqFig}).  The selected edge of
Figure~\ref{L4graphEFig} corresponds to the diagonal of the square
given by $u_1+u_2=1$. Thus, the mapping of the result of the free
symmetric Merge to semantic space determines a point in the moduli
space $BHV_4^+$.  This completes the right arrow (free symmetric Merge
to Semantic Spaces) in the top part of Figure~\ref{SemExtFig2}, for
the example of this simple sentence.

\begin{figure}[h]
\begin{center}
\includegraphics[scale=0.36]{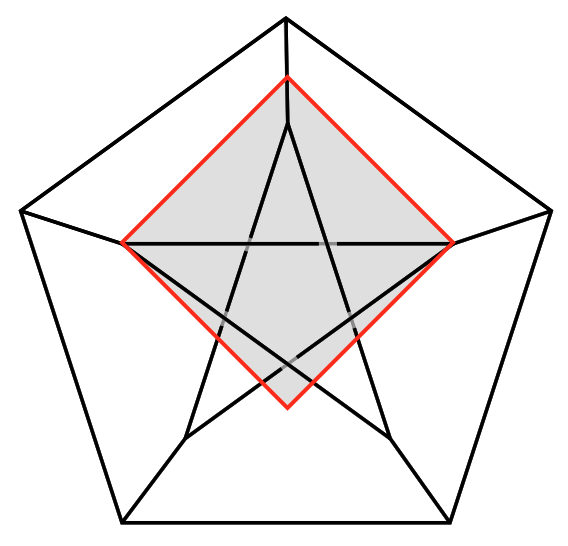}  
\caption{The square and the selected edge in the link $\cL_4$ for corresponding to the abstract tree $T=\{ \{ \alpha, \beta\} , \{ \gamma, \delta \} \}$: 
the mapping of $T$ to semantic space selects a point in this square. \label{L4graphSqFig}}
\end{center}
\end{figure}

\smallskip

We next  see in this example the bottom part of Figure~\ref{SemExtFig2}, that describes the compatibility between Externalization
and the syntax-semantics interface. First note that the associahedra $K_4$ are tiled with squares (quadrangles), as in Figure~\ref{AssocSqFig}.
The two vertices of the square adjacent to the marked vertex of the pentagon corresponds to degenerate trees where one or the
other of the internal edges as shrunk to zero length, while the other has normalized length one. Thus, we see that we can map this
square to the square of Figure~\ref{L4graphSqFig} through the same coordinates $(u_1,u_2)$ describing the lengths of the
two internal edges (compare with Figure~\ref{M04Fig} for the case $n=3$). 

\begin{figure}[h]
\begin{center}
\includegraphics[scale=0.3]{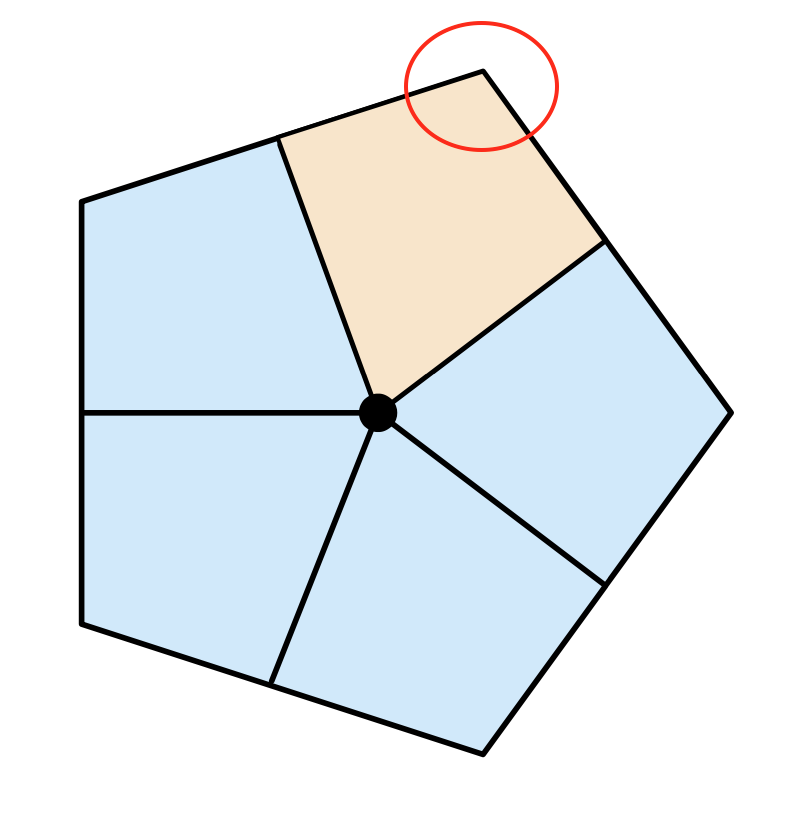}  
\caption{The associahedron $K_4$ tiled with squares (quadrangles), with the selected vertex associated to $((\alpha \,\, \beta)\,(\gamma \,\, \delta))$ (modified
figure by Satyan Devadoss from \cite{DevMor}). \label{AssocSqFig}}
\end{center}
\end{figure}

This lifting of the point associated to $T$ in the square of 
Figure~\ref{L4graphSqFig} to a corresponding point in the square of Figure~\ref{L4graphSqFig} is the effect of the section $\sigma_{L,4}$
described in \eqref{sigmaExt}. To see this, we need to take into consideration the fact that the $24$ associahedra combine
together into a single geometric space, obtained by gluing them along their boundaries. This is done in two steps: first $12$
associahedra are glued along their boundaries as in the left-hand-side of Figure~\ref{AssocGlueFig}, 
forming the space $\bar M_{0,5}(\R)$. Then the orientation double cover is formed: in a self-intersecting $3$-dimensional
visualization, this resulting space $\bar M^{or}_{0,5}(\R)$ can be identified with the {\em great dodecahedron} in the right-hand-side 
of Figure~\ref{AssocGlueFig}. It is not easy to see from its $3D$ representation as great dodecahedron, but the space $\bar M^{or}_{0,5}(\R)$
is a genus $4$ hyperbolic surface, and can be seen more directly from its description in terms of fundamental domain given in \cite{Apery}, 
as in Figure~\ref{HypDodecaFig} below.

\begin{figure}[h]
\begin{center}
\includegraphics[scale=0.2]{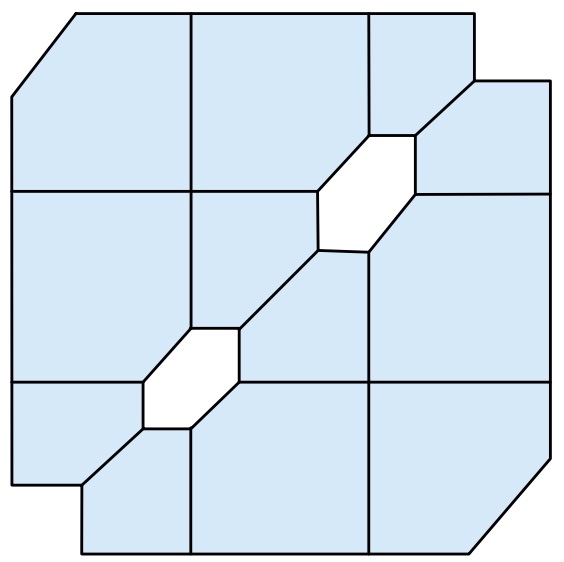}  \ \ \ \ \ \ \ \ 
\includegraphics[scale=0.2]{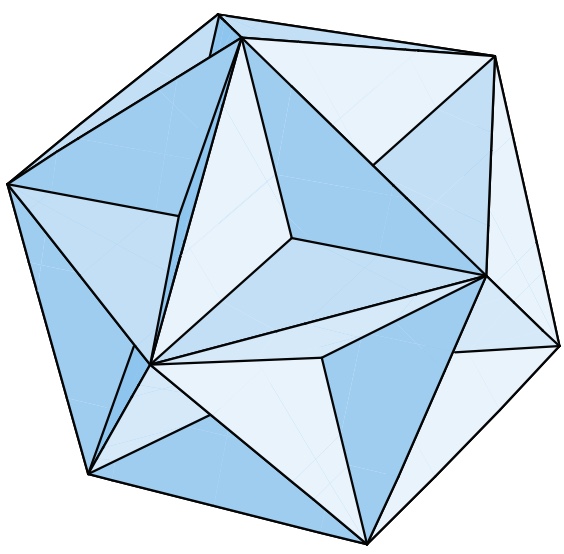} 
\caption{Twelve associahedra $K_4$ assemble into the space $\bar M_{0,5}(\R)$ and its orientation double cover gives $24$ associahedra assembled into the space
$\bar M^{or}_{0,5}(\R)$ identified with the great dodecahedron (figure by Satyan Devadoss from \cite{DevMor}). \label{AssocGlueFig}}
\end{center}
\end{figure}

\begin{figure}[h]
\begin{center}
\includegraphics[scale=0.2]{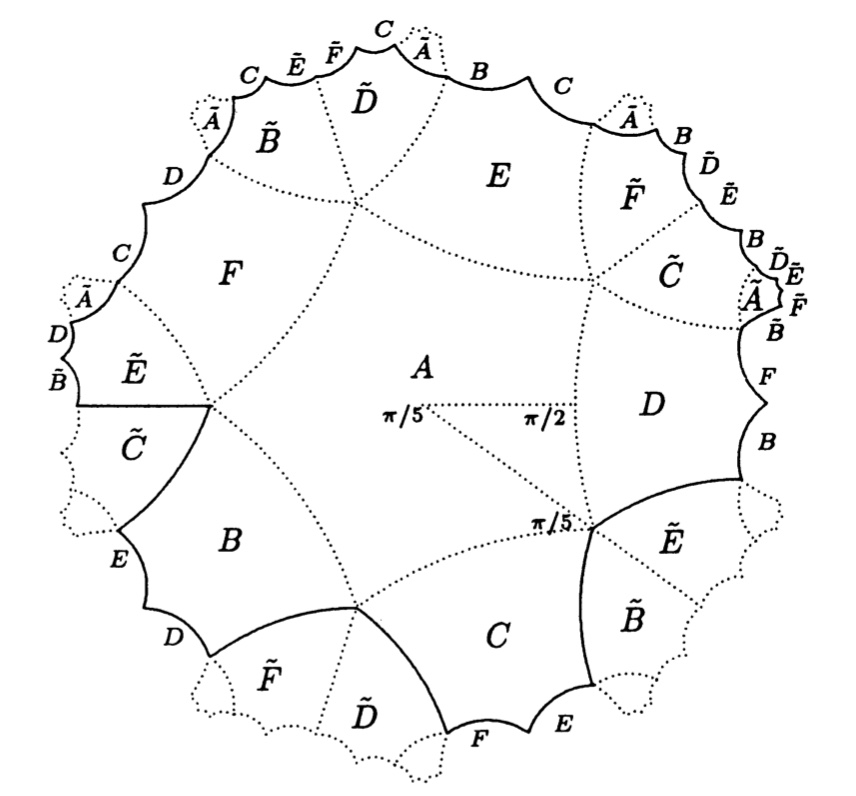}  \ \ \ \ 
\includegraphics[scale=0.3]{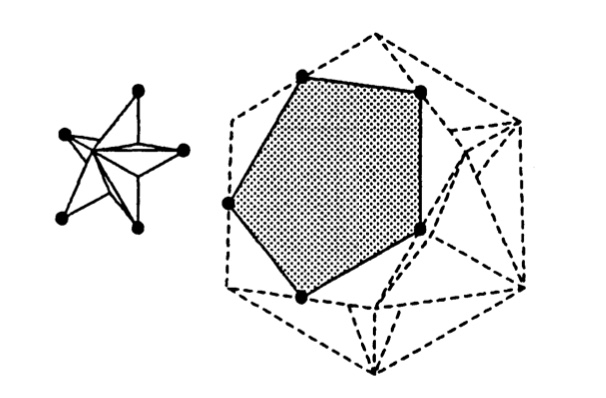} 
\caption{The great dodecahedron $\bar M^{or}_{0,5}(\R)$ as a hyperbolic genus $4$ surface, and the two different forms 
of the $24$ associahedron tiles (figure from \cite{Apery}). \label{HypDodecaFig}}
\end{center}
\end{figure}

The projection map $\Pi_4: \overline{M}^{or}_{0,5}(\R) \twoheadrightarrow {\rm BHV}_4^+$
of \eqref{projPin} folds together and identifies $8$ squares in $\overline{M}^{or}_{0,5}(\R)$ to each
square in ${\rm BHV}_4^+$. Thus, when we lift to $\overline{M}^{or}_{0,5}(\R)$ the point assigned 
to the tree $T$ in one of the squares of ${\rm BHV}_4^+$ by the mapping of $T$ to semantic space,
the lifted point lies on one of the $8$ preimages of the given square of ${\rm BHV}_4^+$. This
choice of one ut of the $8$ preimages is the choice of planar structure of the syntactic object
determined by externalization and this gives indeed the section $\sigma_{L,4}$
described in \eqref{sigmaExt}, where here $L=\, $English. 

\begin{figure}[h]
\begin{center}
\includegraphics[scale=0.2]{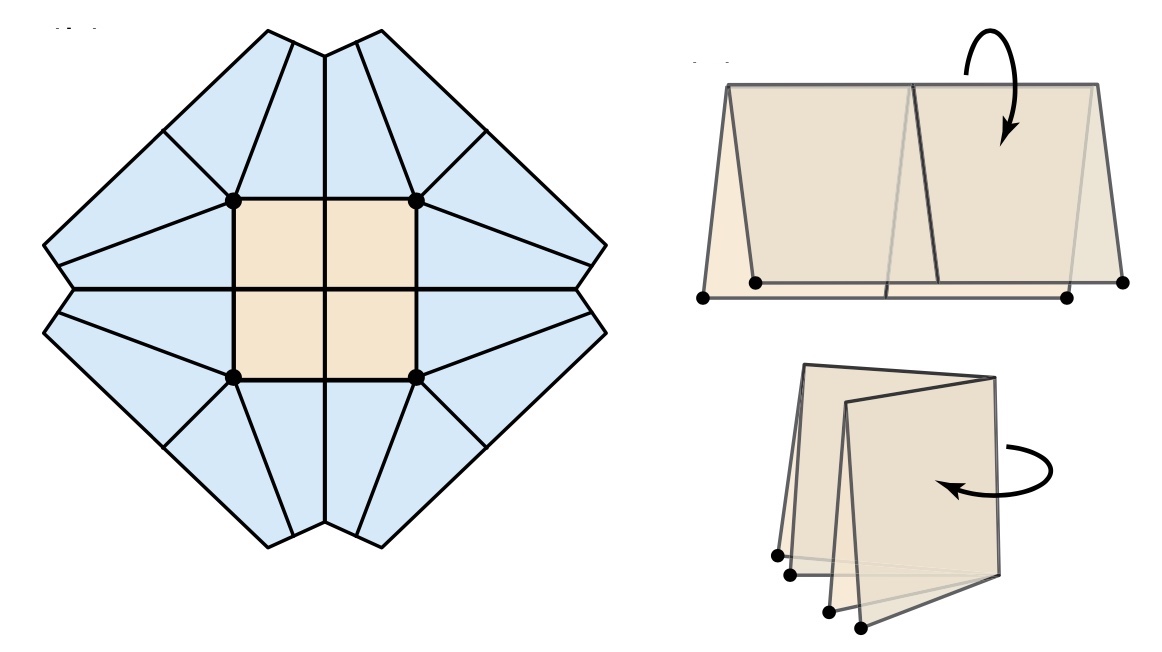}  
\caption{Four squares in adjacent associahedra $K_4$ are folded together (origami folding) in the projection
to ${\rm BHV}_4^+$ so that $8$ squares in the double cover $\overline{M}^{or}_{0,5}(\R)$ are identified in the projection
$\Pi_4: \overline{M}^{or}_{0,5}(\R) \twoheadrightarrow {\rm BHV}_4^+$ 
 (figure by Satyan Devadoss from \cite{DevMor}). \label{OriFoldFig}}
\end{center}
\end{figure}

One can then see in this same simple example, that if instead of taking the planarization
$T^{\pi_L}=((\alpha \,\, \beta)\,(\gamma \,\, \delta))$ of the syntactic object
$T=\{\{ \alpha, \beta\},\,\{ \gamma , \delta\}\}$, one would take the planarization $T^{\pi_h}$
determined by the head function, one would end up with the {\em different} planar tree
$$ T^{\pi_h}= \Tree[ [ $\gamma$ $\delta$ ]  [$\beta$ $\alpha$ ] ]     =(( \gamma \delta )(\beta \alpha)) \, .$$
This means that one ends up on a different one of the $24$ associahedra and a square inside that
associahedron, that is still one of the $8$ squares that project to the same (unchanged) square in 
${\rm BHV}_4^+$. The same point in this square in ${\rm BHV}_4^+$ determined by mapping $T$
to semantic space is then lifted to a corresponding point in a different square inside $\overline{M}^{or}_{0,5}(\R)$.
This means that we are considering a different section of the projection $\Pi_4$. This is the section
$\sigma_{\cS,4}$ described in \eqref{semasec}. The difference between these two sections is
measured by a transformation $\gamma_{L,4} \circ \sigma_{L,4} = \sigma_{\cS,4}$, where
applied to our syntactic object $T$ this gives the permutation 
$\gamma_{L,4}(T) = (3421)$.
As in \eqref{gammaLn}, this transformation $\gamma_{L,4}$ is Kayne's LCA algorithm
for this very simple example.

\smallskip

We have considered a very simple example with $n=4$ where the geometry
is straightforward to visualize. The spaces
$\overline{M}^{or}_{0,n+1}(\R)$ and ${\rm BHV}_n^+$ grow significantly
in combinatorial complexity as $n$ becomes larger, but they are still
very well understood and widely studied geometric spaces. Other more
complicated geometries are likely to arise if the mapping of syntactic
objects to semantic spaces is done in a more sophisticated and
informative way than the very simple type of mappings we considered in
this paper as illustrative examples.


\subsection{Geometric view of some planarization questions}

We conclude this section by briefly commenting on how certain
frameworks where the question of planarization of syntactic
objects arises can be also seen in terms of the geometry
described above. We discuss briefly 
Kayne's Linear Correspondence Axiom and Cinque's
 Abstract Functional Lexicon, and we also
 outline how one can describe the role of syntactic parameters
 in this geometric setting.

\smallskip
\subsubsection{Kayne's LCA algorithm} \label{LCAsec}

When the planar structure assigned by the section $\sigma_{\cS,n}$ is
the planar structure $\pi_h$ determined by a head function, as in the
case of \eqref{sigmaSsec}, 
this means comparing the planar structure $\pi_h$, for 
$h$ defined on ${\rm Dom}(h)\subset \fT_{\cS\cO_0}$,
with the planar structure $\pi_L$ defined by the section $\sigma_{L,n}$
through the constraints imposed by the syntactic parameters of the
language $L$. This comparison can be seen as a version of 
Richard Kayne's LCA (Linear Correspondence Axiom), \cite{Kayne1},
\cite{Kayne2}. As we observed in \cite{MBC}, Kayne's LCA cannot be
defined {\em globally} on $\fT_{\cS\cO_0}$, but is only partially defined on
the domain ${\rm Dom}(h)$ of a head function (the syntactic head),
hence it does not play the same role as Externalization, which
is a choice of a globally defined (non-canonical) section $\sigma_{L,n}$.
However, on the domain ${\rm Range}(s_{\cS, h,n})\subset \cL_n\cap {\rm Dom}(h)$ where
both $\sigma_{L,n}$ and $\sigma_{\cS,n}$ are defined there exists
a covering transformation $\gamma_{L,n}$ of the projection map
$$   \overline{M}^{or}_{0,n+1}(\R) \to {\rm BHV}_n^+ $$
that satisfies
\begin{equation}\label{gammaLn}
 \gamma_{L,n} \circ \sigma_{L,n} = \sigma_{\cS,n} 
\end{equation} 
at all points in ${\rm Range}(s_{\cS, h,n})$. This covering transformation $\gamma_{L,n}$
plays the role of the (partially defined) LCA algorithm. 

\smallskip
\subsubsection{Cinque's abstract functional lexicon} \label{5sec}

There are other constructions that one can fit into this geometric picture
with the projection map $\overline{M}^{or}_{0,n+1}(\R) \to {\rm
  BHV}_n^+$--partially defined sections, and covering transformations permuting the
$2^{n-1}$ points of the fibers of the projection map.  For example, 
one can view this as the abstract functional lexicon 
described by Cinque \cite{Cinque}.

\smallskip

In \cite{Cinque}, Cinque considers the problem of
comparing word order relations imposed on individual languages by
particular syntactic parameters, with a certain base ordering relation
of proximity to the verbal properties of different morphemes (in a
structural sense, rather than in terms of linear ordering), such as
mood, tense, modality aspect, and voice. In \cite{Cinque}, a general
hierarchy of functional morphemes and of adverbial classes is
identified (see (6) and (7) of \cite{Cinque}).  As observed in
\cite{Cinque}, with verbal morphemes as heads and corresponding
classes of adverbs as phrases in so-called specifier position, this hierarchy
determines a planar embedding (in the way that it appears in (6) and
(7) of \cite{Cinque}).  Syntactic parameters, on the other hand, also
determine a planar embedding.  While this could be {\em a priori} arbitrary,
the variability across languages is far less than the space of
combinatorial possibilities would allow. Also, different word order
constraints appear not to be independent, but to exhibit a significant
degree of relatedness. This can be seen both at the theoretical level
(see \cite{Greenberg}) and at the level of database analysis
of syntactic parameters (see \cite{Ort}, \cite{Park}, \cite{Port}).

\smallskip

In terms of the geometry of moduli
spaces described above, one can describe the difference between
the ordering (planar structure) described by Cinque in \cite{Cinque}
and the deviation from it in the word order of specific languages  
in terms of {\em covering tranformations} $\gamma_{L,n}$ of the projection
map $\overline{M}^{or}_{0,n+1}(\R) \to {\rm BHV}_n^+$ that act as
permutations of the planar structures, and are language specific. 
The degree to which word order constraints deviate from the
base structural hierarchy described in Cinque can then be
measured in terms of how far the $\gamma_{L,n}$ are from the identity
in the group of covering transformations of the projection map.

\smallskip
\subsubsection{A geometric view of syntactic parameters} \label{SyntParamSec} 

Syntactic parameters fix constraints on the planar structure of
Externalization. For an extensive recent account of syntactic
parameters see \cite{Roberts}.  In \cite{MCB} we interpreted the role
of syntactic parameters as constraints on the choice of a
language-dependent section $\sigma_{L,n}$ for the Externalization of free
symmetric Merge.

\smallskip

The discussion above shows that in our setting we can also interpret
the role of syntactic parameters in a geometric way, as the choice,
for a given language $L$, of a collection
$L \mapsto \{ \gamma_{L,n} \}_n$ of covering transformations of the
projections $\overline{M}^{or}_{0,n+1}(\R) \to {\rm BHV}_n^+$, as in
\eqref{gammaLn}. Comparison of syntactic parameters across languages
can be formulated in various computational forms.  This includes the
difficult problem of understanding the relation among parameters, as
well as the much lower dimensional space occupied by actual languages
inside the high-dimensional space of possible values of the hundreds
of parameters currently studied (see for example \cite{GuaLong},
\cite{LongTre}, \cite{Mar}, \cite{KKohl}, \cite{Port}, \cite{ShuMar}). In
particular, one can focus on the effect of syntactic parameters on
word order constraints.  In this case, using the framework we consider
here, one can view this comparison across languages as the comparison
between sections $\sigma_{L,n}$ for different languages $L$, or
equivalently as the properties of the collection of elements
$\gamma_{L,n} \gamma_{L',n}^{-1}$, for $L\neq L'$ in the group of
covering transformations of
$\overline{M}^{or}_{0,n+1}(\R) \to {\rm BHV}_n^+$.

\section{Birkhoff factorization and (semi)ring parsing}\label{BirkSemiParseSec}

We now extend the setting introduced above to more refined descriptions of
the characters and factorization, that incorporate more detailed
properties of semantic parsing and compositionality.  We first focus on
{\em semiring parsing}, as in \cite{Goodman}, while in \S \ref{PietSec} we analyze
how our model relates to Pietroski's theory of {\em minimalist meaning}, \cite{Pietro}. 
These two settings represent very
different models, where semiring semantics incorporates the idea of
truth-values and generalizes it to values in arbitrary (semi)rings, not
necessarily Boolean, while Pietroski's approach provides an alternative that bypasses the
idea of truth-values entirely and is based on a compositional structure
modeled on the Minimalism's Merge operation.

\smallskip

In this section we analyze (semi)ring parsing, introducing a version
that is adapted to Minimalism formulated in terms of free symmetric
Merge. 

\smallskip

Since this is the most mathematically-heavy section in this paper,
we provide a preliminary outline of the content and a more heuristic
explanation of what is covered in the various subsections, before
starting to discuss the more precise details.

\smallskip
\subsection{Preliminary discussion}\label{ParsePrelimSec}

The relation between grammars and semirings was first 
observed by Chomsky--Sch\"utzenberger in \cite{ChoSchu}.
Semiring parsing (see for instance \cite{Goodman}),
when formulated in the setting of context-free grammars,
considers deduction rules of the form
$$ \frac{A_1 \ldots A_k}{B} C_1\ldots C_\ell \, , $$
where the terms $A_i$ (main conditions) are rules $R$ of the grammar
or input nonterminals and the $C_i$ are (non-probabilistic) Boolean
side conditions and the fraction notation means that if the numerator
terms hold then the denominator term also does. To the main condition
terms one assigns values in a semiring, combined with the semiring
operations, to obtain a value for the deduced output. The target
semiring varies according to the parsing algorithm considered. The
main choices include the Boolean semiring, the tropical semiring, the
probability semiring (that is, the familiar case of Viterbi parsing),
as well as the non-commutative derivation forest semiring, that
collects all the possible derivations, with concatenation as
multiplication and union as semiring addition.  Often, parsing with
values in other semirings factor through the derivations
semiring. This setting is specifically constructed in the formal
language context, and specifically for context-free grammars, though
some generalizations exist in mildly context-sensitive classes like
those produced by tree-adjoining grammars (TAGs).

\smallskip

A natural question arises about what type of algebraic
structure replaces this form of semiring parsing in the
setting of Minimalism, and more specifically the form of
Minimalism based on free symmetric Merge.

\smallskip

The main goal of this section is to provide an answer to
this question, in a form that is again based on the
Birkhoff factorization procedure, that we present 
throughout this paper as a natural formalism for 
different forms of assignments of semantic values
in the context of a free symmetric Merge model of syntax.

\smallskip

Developing this form of ``semiring parsing" (where semirings
will in fact be replaced by more general algebraic objects) 
requires several steps, that we now briefly summarize.

\smallskip

In \S \ref{MSsec} we introduce a ring of Merge derivations, formed by
considering chains of Merge operations, given by the action of Merge
on workspaces. These are assembled into a ring structure, where the
linear structure is obtained by taking the vector space spanned by the
derivations (that is, including formal linear combinations) and the
multiplication operation is the union of the workspaces with the
corresponding Merge actions. These are the same operations on the
algebra part of the Hopf algebra of workspaces that we introduced in
\cite{MCB} and used earlier in this paper. We only
consider the product structure on this ring and not the coproduct,
as we have on the Hopf algebra of workspaces.  However, this does not
lead to a loss of structure in this case, because the coproduct is
built into the Merge operation on workspaces, so it it still encoded
into the data of this ring of Merge derivations.

\smallskip

In order to illustrate more clearly the properties of this
ring of Merge derivations, in \S \ref{MSsec} we return
to discuss the notion of Minimal Search in the Merge 
model of syntax. In \cite{MCB} we gave an account of
how Minimal Search is implemented as extraction of
leading order term in the action of Merge on workspaces.
This leading order term contains Internal and External
Merge, while it excludes other presumptively unwanted forms of 
Merge (Sideward and Countercyclic). The idea of
extraction of the leading order term is closely related to
Birkhoff factorization, as originally observed in the
context of the renormalization in physics. 

Here we show that in fact, after extending the ring of Merge
derivations to a ring of Laurent series with coefficients in this
ring, one can indeed show that Minimal Search is exactly a Birkhoff
factorization in this ring--for a character from the Hopf algebra of
workspaces that assigns to a workspace its Merge derivation and a
power that counts the effect of that derivation on the size of the
workspace. The Birkhoff factorization separates out, on one side, the
unwanted forms of Merge, while retaining on the other side only
fundamental ones, namely External and Internal Merge.  This case of
Birkhoff factorization happens to be the one that is closest to the
original form used in physics.

\smallskip

This result on Minimal Search as Birkhoff factorization
in \S \ref{MSsec} is not required for the following
parts of this section, and is included to provide
some more direct understanding of the ring of Merge
derivations and to connect it to our original formulation in \cite{MCB}. 
This section can be skipped (except for Definition~\ref{DMring})
by the readers interested in directly accessing the discussion
of how to extend the semiring parsing framework. 

\smallskip

The main construction for the generalization of semiring parsing
starts in \S \ref{OidSec}. The main viewpoint here is that, in order
to formulate semiring parsing for Merge derivations based on the
action of Merge on workspaces, one needs to replace the
setting of Hopf algebras and semirings, that we used in the previous
sections of this paper to describe simple models of syntax-semantics
interface, with a slightly more flexible form, where the algebraic
structures of Hopf algebra and semiring are replaced by their
``categorified" form, which we refer to, respectively, as 
Hopf algebroids and semiringoids. The reason for this
extension is very simple. Merge derivations given by actions 
of Merge on workspaces only compose when the target workspace
of one derivation agrees with the source workspace of the next. 
This differs from the situation we considered in the previous sections
where we only considered the Hopf algebra of workspaces,
where the product is the disjoint union (combination of workspaces)
which is always defined without source/target matching conditions. 

\smallskip

Thus, just as in passing from the multiplication in a group
to the multiplication in a groupoid, that precisely accounts for the
fact that arrows compose only when the target of the first is the
source of the second, one can obtain similar generalizations
of the structures of Hopf algebra and Rota--Baxter algebra 
(or semiring) that we used in the formulation of mapping
from syntax to semantics as Birkhoff factorization in the 
previous sections. 

\smallskip

An extension of the notion of Hopf algebra that
accommodates for the need for source/target matching conditions
in the product was developed in the context of algebraic
topology with the notion of (commutative) Hopf algebroid
and bialgebroid.
We take that as the starting point in \S \ref{OidSec}, by 
constructing a bialgebroid associated to the ring
of Merge derivations introduced in \S \ref{MSsec}. 
This bialgebroid replaces the Hopf algebra of workspaces
on the syntax side, by encoding the Merge derivations
in syntax. 

\smallskip

In \S \ref{RBoidSec} we then consider the other side of
the Birkhoff factorization, namely the semantics side, where
we wish to replace the algebraic datum of a Rota-Baxter algebra
or Rota-Baxter semiring with an analogous categorified version.
We use a notion of algebroid that is compatible with the
notion of Hopf algebroids and bialgebroids introduced in \S \ref{OidSec}
and we show that the notion of algebroid we consider is dual to
directed graphs, with the cases of bialgebroids being
dual to directed graphs that are
reflexive and transitive (small categories) and Hopf
algebroids being dual to groupoids. 

\smallskip

We also extend the generalization of algebras to algeboids
to an analogous generalization of semirings to a similar
categorified structure of semiringoid. (Note that other different
notions of algebroids and semiringoids exist in the mathematical
literature that should not be confused with the version adopted here.)

\smallskip

In \S \ref{RBoidSec2} we describe how the notion of Rota--Baxter operator
of weight $-1$ on an algebra can be generalized to the case of an
algebroids and similarly in \S \ref{RBoidSec3} we show the analogous
generalization of Rota--Baxter semirings of weight $\pm 1$ to semiringoids.

\smallskip

With this, we have both sides of the mapping ready for the case of Merge 
derivations with their composition structure. We prove in \S \ref{BfactOidSec}
the existence of Birkhoff factorizations of characters from Hopf algebroids
to Rota--Baxter algebroids and from bialgebroids to Rota--Baxter semiringoids. 
The characters and the factorization can here be described dually in terms
of maps of directed graphs. 

\smallskip

We conclude in \S \ref{ParSemSec}
by showing that, with the algebraic setting constructed in the previous
subsections, one obtains a form of semiring(oid) parsing that 
simultaneously generalizes the various semiring parsings of \cite{Goodman}
and the Birkhoff factorizations that we described in \S \ref{HeadIdRenSec}.


\smallskip
\subsection{Minimal Search as Birkhoff factorization}\label{MSsec}

In \cite{MCB} we presented a way to implement Minimal Search and
eliminate unwanted forms of Merge (Sideward and Countercyclic
Merge) and retain only the Internal and External forms of Merge.
In the formulation we presented in \cite{MCB}, Minimal Search is
implemented by extracting the leading order term with respect to
a specific grading function imposed on the terms of the coproduct of the
Hopf algebra. We show here that there is another natural way of
thinking about Minimal Search, by formulating it as a Birkhoff 
factorization, very similar in form to the one used in quantum
field theory, with respect to a character with values in a Laurent series.

\smallskip
\subsubsection{Effect of Merge on workspaces}\label{MergeWSsize}

For consistency with \cite{MCB}, and since here it is not important
to keep track of traces in the effect of Internal Merge, we consider
the quotients $T/F_{\underline{v}}$ in the coproduct as in \cite{MCB}
rather than as in \S \ref{QuotSec}. As a result,
we have the same counting of the effect of Merge on the
various measures of workspace size (number of components, number
of accessible terms, number of vertices, etc). as described in \cite{MCB}.

\smallskip

The different cases of Merge are given by External Merge (EM), 
Internal Merge (IM), Sideward Merge (SM), and
Countercyclic Merge (CM):
\begin{itemize}
\item[EM:] $F=T\sqcup T' \sqcup \hat F\mapsto F'= \fM(T,T')\sqcup \hat F$
\item[IM:] $F=T\sqcup \hat F\mapsto F'=\fM(T_v,T/T_v) \sqcup \hat F$
\item[SM(i):] $F=T\sqcup T' \sqcup \hat F\mapsto F'= \fM(T_v,T'_w)\sqcup T/T_v \sqcup T'/T'_w \sqcup \hat F$
\item[SM(ii):] $F=T\sqcup T' \sqcup \hat F\mapsto F'= \fM(T,T'_w) \sqcup T'/T'_w \sqcup \hat F$
\item[CM(i):] $F=T\sqcup \hat F\mapsto F'=\fM(T_v, T_w) \sqcup T/T_v \sqcup \hat F$
\item[CM(ii):] $F=T\sqcup \hat F\mapsto F'=\fM(T_v, T_w) \sqcup T/T_w \sqcup \hat F$
\item[CM(iii):] $F=T\sqcup \hat F\mapsto F'=\fM(T_v, T_w) \sqcup T/(T_v\sqcup T_w) \sqcup \hat F$
\end{itemize}
where $\hat F$ denotes the part of the workspace that is not affected, and
$$ \fM(T,T')= \Tree [ T  T' ] \, .$$
The effect of these Merge operations on size counting is summarized in
the following table from \cite{MCB}, with $b_0(F)$ the number of
connected components of a workspace $F\in \fF_{\cS\cO_0}$ (number of
syntactic objects), $\alpha(F)$ the number of accessible terms in $F$
(the total number of non-root vertices), $\sigma(F)=b_0(F)+\alpha(F)=\# V(F)$
the total number of vertices, and $\hat\sigma(F)=b_0(F)+\sigma(F)$.
We introduce here a combined variable
$$ \delta = -\Delta (3 b_0 + \alpha) = -\Delta (2 b_0 + \# V) \, , $$
as this will be used in the construction of \S \ref{MergeLaurentSec} below.

\begin{center}
\begin{tabular}{| l ||c|c|c|c|c|}
\hline 
& $\Delta b_0$ & $\Delta \alpha$ & $\Delta \sigma$ & $\Delta \hat\sigma$ & $\delta$ \\
\hline 
EM & $-1$ & $+2$ & $+1$ & $0$ & $1$ \\
\hline
IM & $0$ & $0$ & $0$ & $0$ & $0$ \\
\hline
SM(i)  &  $+1$ & $0$  & $+1$  & $+2$ & $-3$  \\
\hline
SM(ii) &  $0$ & $+1$  & $+1$   & $+1$ & $-1$  \\
\hline
CM(i) &  $+1$ & $\# Acc(T_{a,w_a})$   & $\sigma( T_{a,w_a} )$ & $ \sigma( T_{a,w_a} ) +1$  & $\leq -2$ \\
\hline
CM(ii) & $+1$  & $\# Acc(T_{a,v_a})$   & $\sigma( T_{a,v_a})$  &  $\sigma( T_{a,v_a})  + 1$ & $\leq -2$ \\
\hline
CM(iii) & $+1$  & $-2$  & $-1$  & $0$  & $\leq -1$ \\
\hline
\end{tabular} 
\end{center}
Note that values of $\delta \geq 0$ eliminate all the ``undesirable" forms of Merge (Sideward
and Countercyclic), leaving only Internal and External Merge. (We put
aside here the question as to whether these excluded forms of Merge
are indeed undesirable, and simply assume that this is so.)

We show here that the elimination of the forms of Merge, described in
terms of Minimal Search, can also be formulated as a Birkhoff factorization where
one eliminates divergences as in the physical setting.

\smallskip
\subsubsection{Laurent series ring of Merge derivations}\label{MergeLaurentSec}

We introduce a ring that organizes derivations in the Minimalist generative
grammar defined by free symmetric Merge, weighted by their effect on the
workspace. 

\begin{defn}\label{DMring}  {\rm 
The algebra of free Merge derivations $\cD\cM$ is the commutative associative $\Q$-algebra 
with the underlying $\Q$-vector space spanned by elements of the form 
$\varphi_A$ where $A \subset \cS\cO\times \cS\cO$ 
is a set of pairs $(S,S')$ of syntactic objects, and
\begin{equation}\label{FMF}
 \varphi_A = ( F \stackrel{\fM_A}{\rightarrow} F' ) 
\end{equation} 
consists of all possible chains of Merge operations 
\begin{equation}\label{FMFchain}
 F \stackrel{\fM_{S_1,S_1'}}{\rightarrow} F_1 \rightarrow \cdots F_{N-1}
\stackrel{\fM_{S_N,S_N'}}{\rightarrow} F' 
\end{equation}
with $(S_i,S_i')\in A$. Since the source and target workspaces are assigned,
there are finitely many such possible chains. The algebra multiplication is given by the operation 
\begin{equation}\label{DMprod}
 \varphi_A \sqcup \varphi_B =( F \sqcup \tilde F \stackrel{\fM_{A\sqcup B}}{\rightarrow} F' \sqcup \tilde F'  ) \, , 
\end{equation} 
for $ \varphi_A = ( F \stackrel{\fM_A}{\rightarrow} F' ) $ and $ \varphi_B = ( \tilde F \stackrel{\fM_B}{\rightarrow} \tilde F' ) $, with unit given by the empty forest mapped to itself. 
Let $\cD\cM[ t^{-1} ][[t]]$ denote the associative commutative $\Q$-algebra  
of Laurent power series with coefficients in $\cD\cM$.}
\end{defn}

\smallskip

The meaning of the product \eqref{DMprod} is to perform in parallel different Merge 
operations that affect different parts of a workspace. Such operations, if
conducted sequentially, would commute with each other hence would be independent
of the order of execution (unlike operations
that affect the same components of the workspace), so that composition can
be regarded as simultaneous and parallel rather than sequential, and can be
grouped together as a single operation.

\smallskip

The following fact is well known (see \cite{CoKr}, \cite{CoMa}, \cite{EbFKr}, \cite{EbFM}).

\begin{prop}\label{RBLaurent}
Given a commutative associative algebra $\cA$ and the algebra of Laurent series
$\cA[t^{-1}][[t]]$, the linear operator $R:\cA[t^{-1}][[t]] \to \cA[t^{-1}][[t]]$ that projects
onto the polar part, 
\begin{equation}\label{RBpolar}
 R( \sum_{i=-N}^\infty a_i t^i) =\sum_{i=-N}^{-1} a_i t^i \, , 
\end{equation} 
makes $(\cA[t^{-1}][[t]],R)$ a Rota--Baxter algebra of weight $-1$. 
\end{prop}

\smallskip

\begin{prop}\label{lemDMchar}
Consider the map $\phi: \cH \to  \cD\cM$, 
\begin{equation}\label{DMchar}
\phi(F)=(L(F) \stackrel{\fM_{A(L(F),F)}}{\longrightarrow} F )\, ,
\end{equation}
that assigns to a forest $F$ the set $A(L(F),F)$ of all Merge derivations from the
(multi)set of individual lexical items and syntactic features that form the
set of leaves $L(F)$, to the forest $F$ (the generative process for $F$).
This defines a character (a morphism of commutative algebras)
from the Merge Hopf algebra $\cH$ of non-planar binary rooted forests to 
the algebra of free Merge derivations $\cD\cM$.
The assignment
\begin{equation}\label{DMLchar}
\phi_t(F)=(L(F) \stackrel{\fM_{A(L(F),F)}}{\longrightarrow} F )\,\, t^{\delta(\fM_{A(L(F),F)})}\, ,
\end{equation}
then defines a morphism of commutative algebras 
$\phi_t: \cH \to \cD\cM[t^{-1}][[t]]$.
\end{prop}

\proof It suffices to check that $\phi(F\sqcup F')=\phi(F)\sqcup \phi(F')$,
namely that
$$ (L(F) \sqcup L(F') \stackrel{\fM_{A(L(F) \sqcup L(F') ,F\sqcup F')}}{\longrightarrow} F\sqcup F' )=
(L(F) \sqcup L(F') \stackrel{\fM_{A(L(F),F)\sqcup A(L(F'),F')}}{\longrightarrow} F\sqcup F' ) \, . $$
This is the case since, if
the end result of a chain of Merge operations contains
a disjoint union $F \sqcup F'$ of two trees, then all the individual Merge operations
$\fM_{T_v,T_w}$ in the chain will use syntactic objects $T_v$ $T_w$ where both
sets of leaves $L(T_v)$ and $L(T_w)$ are subsets of $L(F)$ or where both are
subsets of $L(F')$ as otherwise $\fM_{T_v,T_w}$ would create a connected component
$T$ with $L(T)\cap L(F)\neq \emptyset$ and $T\cap L(F')\neq \emptyset$ so that the
end result would not contain $F \sqcup F'$. Moreover, 
$$ \delta( F\stackrel{\fM_A}{\longrightarrow} F' ) = (3b_0+\alpha)(F)-(3b_0+\alpha)(F') $$
and $(3b_0+\alpha)(F\sqcup \tilde F)=(3b_0+\alpha)(F)+(3b_0+\alpha)(\tilde F)$ so that
\eqref{DMLchar} is also an algebra homomorphism.
\endproof

\smallskip

As we will see in Lemma~\ref{phitreg}, the character $\phi_t: \cH \to \cD\cM[t^{-1}][[t]]$
of Proposition~\ref{lemDMchar} is not good enough to detect the difference between
Internal/External Merge and Sideward/Countercyclic Merge. However,
one can consider similar characters more suitable for this purpose. A simple modification
of $\phi_t$ that works can be obtained in the following way, where the statement follows
as in Proposition~\ref{lemDMchar}.

\begin{cor}\label{lemDMchar2}
For $T\in \fT_{\cS\cO_0}$ let $\cF_T\subset \fF_{\cS\cO_0}\times \fF_{\cS\cO_0}$ denote the set of
pairs $(F,F')$ of forests $F$ with $L(F)=L(F')=L(T)$ that are intermediate derivations for $T$, namely such that
there exists a chain of free symmetric Merge derivations
$$ L(T) \stackrel{\fM_{S_1,S_1'}}{\longrightarrow} \cdots \stackrel{\fM_{S_i,S_i'}}{\longrightarrow} F 
\stackrel{\fM_{S_{i+1},S_{i+1}'}}{\longrightarrow} \cdots \stackrel{\fM_{S_j,S_j'}}{\longrightarrow} F'
\stackrel{\fM_{S_{j+1},S_{j+1}'}}{\longrightarrow} \cdots \stackrel{\fM_{S_m,S_m'}}{\longrightarrow} T\, ,$$
for some $m\geq 1$, including the case with $F=L(T)$ and $F'=T$.
Consider the assignment
\begin{equation}\label{DMLchar2}
\psi_t(T)=\sum_{(F,F')\in \cF_T}  (F \stackrel{\fM_{A(F,F')}}{\longrightarrow} F' )\,\, t^{\delta(\fM_{A(F,F')})}\, ,
\end{equation}
where $\fM_{A(F,F')}$ is the set of all possible Merge derivations from $F$ to $F'$. 
This determines a morphism of commutative algebras 
$\psi_t: \cH \to \cD\cM[t^{-1}][[t]]$.
\end{cor}

\smallskip

The reason why the choice of the character $\psi_t$ of \eqref{DMLchar2} is preferable to
the choice of $\phi_t$ of \eqref{DMLchar} is explained by the following simple property.

\begin{lem}\label{phitreg}
The character $\phi_t: \cH \to \cD\cM[t^{-1}][[t]]$ takes values in the subring
$$\cD\cM[[t]]=(1-R)\, \cD\cM[t^{-1}][[t]]$$ of formal power series. 
\end{lem}

\proof
Consider the case of a tree $T\in \fT_{\cS\cO_0}$. The value
$$ \phi_t(T) =(L(T)\stackrel{\fM_{A(L(T),T)}}{\longrightarrow} T)\,\, t^{\delta(T)} $$
represents the complete set of all possible chains of free symmetric Merge derivations 
that construct the syntactic object $T$ starting from a (multi)set $L=L(T)$ of
lexical items and syntactic features. If $\# L=\ell$ then $\# V(T)=2\ell -1$ 
so we have $\delta(T)=(2b_0+\# V)(L)-(2b_0+\#V)(T)=3 \ell - 2 -(2\ell -1) =\ell -1 \geq 0$.
Thus, notice that $\phi_t(T)$ is always in the non-polar part $\cD\cM[[t]]$ for any tree $T$,
regardless of whether some Sideward or Countercyclic Merge operations have been used 
along the chain of derivations. This means that $(1-R)\phi_t(T)=\phi_t(T)$ for all $T$.
The case of forests is then immediate since $\phi_t(F)=\prod_a \phi_t(T_a)$, for
$F=\sqcup_a T_a$ and $\delta(F)=\sum_a \delta(T_a)\geq 0$
\endproof

Thus, the character $\phi_t$  does not suffice to separate Internal/External Merge 
from Sideward and Counter\-cyclic Merge operations on the basis of the counting
given by $\delta$. On the other hand, the character $\psi_t$, that also considers all the
intermediate derivations from $L(T)$ to $T$, each weighted according to the corresponding
value of $\delta$ will have a non-trivial polar part, when Sideward/Countercyclic Merge operations 
are present somewhere in the chain of derivations. 

\smallskip

However, even when using the character $\psi_t$ that detects the presence of the
so-called undesirable forms of Merge in a derivation, simply applying the projection onto the regular part 
$$ \psi_t(F) \mapsto (1-R) \psi_t(F)  $$
does not suffice to eliminate those Sideward/Countercyclic Merge operations and
only retain Internal/External Merge. This is a consequence of the fact that the projection $R$
onto the polar part is not an algebra homomorphism but a Rota--Baxter operator.
The failure of the Rota--Baxter operator $R$ of \eqref{RBpolar} 
to be an algebra homomorphism
$$ R( (\sum_{i=-N}^\infty a_i t^i)(\sum_{j=-M}^\infty b_j t^j) )= R(\sum_{n=-(N+M)}^\infty \sum_{i+j=n} a_i b_j\,\, t^n ) =\sum_{n=-(N+M)}^{-1} \sum_{i+j=n} a_i b_j\,\, t^n $$
$$ \neq R(\sum_{i=-N}^\infty a_i t^i) R(\sum_{j=-M}^\infty b_j t^j) = 
\sum_{n=-(N+M)}^{-1} \sum_{i+j=n, i<0, j<0} a_i b_j\,\, t^n $$
reflects the fact that terms in a product of series can end up in the polar (respectively, non-polar) part
of the product without being in the polar (respectively, non-polar) part of the individual factor, because
of the sum $t^{i+j}$ of the exponents. This means that simply applying $(1-R)$ to $\phi(F)$
will not suffice to get rid of the free Merge derivations that contain 
Sideward and Countercyclic Merge. However, Birkhoff factorization achieves that result.

\smallskip

\begin{thm}\label{MSBirkhoff}
The inductively constructed Birkhoff factorization \eqref{AlgBfact} of the
character $\psi_t$ of \eqref{DMLchar2} implements Minimal Search, in
the sense that it inductively eliminates all
Sideward and Countercyclic Merge forms from the derivations and only
retains compositions of Internal and External Merge.
\end{thm}

\proof This is a direct consequence of Proposition~\ref{RBfact}. 
Taking $\psi_{t,+}(T)=(1-R)\tilde\psi_t(T)$, with $\tilde\psi_t$ the
Bogolyubov preparation of $\psi_t(T)$ gives an algebra homomorphism
$$ \psi_{t,+}: \cH \to \cD\cM[[ t ]] \, , $$
where in the inductive construction of 
$$ \tilde\psi_t(T) =\psi_t(T) +\sum \psi_{t,-}(F_{\underline{v}}) \psi_t(T/F_{\underline{v}}) $$  
one analyzes in parallel the Merge derivations of accessible terms of $T$, ensuring that
the so-called undesirable forms of Merge are progressively removed from all the 
accessible terms of $T$ and only derivations containing Internal and External Merge
(that is, with $\delta \geq 0$) are retained at each step. More precisely, if there
is a term in $\psi_t(T)$ of the form $(F \to F')t^\delta $ where the derivation is a 
Sideward or Countercyclic Merge, the forest $F'$ will occur as a collection of accessible
terms $F'=F_{\underline{v}}$ in $T$, hence in $\tilde\psi_t(T)$ the term
$\psi_{t,-}(F_{\underline{v}})\psi_t(T/F_{\underline{v}})$ will contain a term
$R(\psi_t(F'))\psi_t(T/F_{\underline{v}})$ which will contain a summand equal to
$-(F \to F')t^\delta $ that has the effect of removing the unwanted derivation,
while any term $(F \to F')t^\delta $ in $\psi_t(T)$ that only contains derivations
using Internal/External Merge is not cancelled by anything coming from  
the terms $\psi_{t,-}(F_{\underline{v}})\psi_t(T/F_{\underline{v}})$, because
such terms are eliminated when applying $R$ in the inductive construction of
$\psi_{t,-}(F_{\underline{v}})$. 
\endproof

\smallskip
\subsection{Birkhoff factorization in algebroids} \label{OidSec}

The construction of the ring (algebra) $\cD\cM$ of Merge derivations in
the previous sections can be seen as an adaptation to the case of
the free symmetric Merge (in the form presented in \cite{MCB}) 
of the idea of  the {\em derivation forest semirings} of
\cite{Goodman}, where the original case treated in
\cite{Goodman} is based on derivations in context-free grammars. We
now show how to extend this notion from the setting of context-free semiring parsing to
the Minimalist account.
To see the analogy more directly, instead of the algebra we used in \S \ref{MergeLaurentSec},
one can construct a slightly different algebraic object encoding the same set of
free symmetric Merge derivations. This will include the data of the Hopf
algebra $\cH$, while incorporating not just the workspaces
but also the explicit Merge derivations acting on them.

\smallskip

We recall the notion of commutative bialgebroid and Hopf algebroid, 
originally introduced in the context to algebraic topology 
(see Appendix~A1 of \cite{Rav}). We will assume here that all
algebras and vector spaces are over the field $\Q$ of rational
numbers, unless otherwise stated.

\begin{defn}\label{HopfOid} {\rm
A commutative Hopf algebroid is a semigroupoid scheme, 
namely a pair of commutative algebras
$\cA^{(0)}$ and $\cH^{(1)}$ with the property that, for any other commutative algebra 
$\cR$, the sets $\cG^{(0)}(\cR)=\Hom(\cA^{(0)},\cR)$ and $\cG^{(1)}(\cR)=\Hom(\cH^{(1)},\cR)$
are the objects and morphisms of a groupoid $\cG$. Equivalently, the pair of algebras 
$(\cA^{(0)}, \cH^{(1)})$ is endowed with homomorphisms $\eta_s,\eta_t: \cA^{(0)} \to \cH^{(1)}$
that give $\cH^{(1)}$ the structure of a $\cA^{(0)}$-bimodule (dual to source and target maps
of the groupoid), a coproduct (dual to composition of arrows in the groupoid) given by a
morphism of $\cA^{(0}$-bimodules
$$ \Delta: \cH^{(1)} \to \cH^{(1)}\otimes_{\cA^{(0)}}   \cH^{(1)}\, , $$
a counit $\epsilon: \cH^{(1)} \to \cA^{(0)}$, which is also a morphism of $\cA^{(0}$-bimodules 
(dual to the inclusion of identity morphisms),
and a conjugation $S: \cH^{(1)} \to \cH^{(1)}$ (dual to the inverse of morphisms in the
groupoid). These maps satisfy $\epsilon \eta_s=\epsilon \eta_t =1$ (identity
morphisms have same source and target), $(1\otimes \epsilon)\Delta = (\epsilon \otimes 1)\Delta= 1$
(composition with the identity morphism), $(1\otimes \Delta)\Delta = (\Delta \otimes 1)\Delta$
(associativity of composition of morphisms), $S^2=1$ and $S \eta_s=\eta_t$ (inversion is an
involution and exchanges source and target of morphisms), and the property that
composition of a morphism with its inverse gives the identity morphism, namely that
$$ \eta_t  \epsilon = \mu (S\otimes 1) \Delta \ \ \ \text{ and } \ \ \  \eta_s \epsilon = \mu (1\otimes S) \Delta, $$
with $\mu: \cH^{(1)}\otimes_{\cA^{(0)}}   \cH^{(1)}\to \cH^{(1)}$ extending the algebra
multiplication $\mu: \cH^{(1)}\otimes_\Q \cH^{(1)}\to \cH^{(1)}$. Also one has $\Delta \eta_s=1\otimes \eta_s$,
$\Delta \eta_t = \eta_t \otimes 1$ (the source of the composition of arrows is the source of the first
and the target of the composition is the target of the second).
 A morphism of Hopf algebroids $$f: (\cA^{(0)}_1, \cH^{(1)}_1)\to (\cA^{(0)}_2, \cH^{(1)}_2)$$
is a pair of algebra homomorphisms $f^{(0)}: \cA^{(0)}_1 \to \cA^{(0)}_2$ and
$f^{(1)}: \cH^{(1)}_1 \to \cH^{(1)}_2$ with $f^{(0)}\circ \epsilon_1=\epsilon_2\circ  f^{(1)}$,
$f^{(1)}\circ \eta_{s,1}=\eta_{s,2} \circ f^{(0)}$, $f^{(1)}\circ \eta_{t,1}=\eta_{t,2} \circ f^{(0)}$, 
$f^{(1)}\circ S_1=S_2\circ  f^{(1)}$, $\Delta_2 \circ f^{(1)}= (f^{(1)}\otimes f^{(1)})\circ \Delta_1$.

A commutative bialgebroid is a structure as above, where one does not
assume invertibiliy of morphisms, namely where 
$\cC^{(0)}(\cR)=\Hom(\cA^{(0)},\cR)$ and $\cC^{(1)}(\cR)=\Hom(\cH^{(1)},\cR)$
are the objects and morphisms of a (small) category $\cC$ (a semigroupoid) 
instead of a groupoid, so that one has the same structure above but 
without the conjugation map $S$. 
}\end{defn}

\smallskip

Examples of Hopf algebroids arise, for instance, when the field of
definition of a Hopf algebra $\cH$ is replaced by the ring of functions $\cA$ 
of some underlying space. In our setting, the natural modification of the
Hopf algebra $\cH$ of workspaces is a version where arrows corresponding
to the action of Merge are also incorporated as part of the same algebraic
structure.  Since these will in general not necessarily be invertible arrows, 
the resulting structure will be a bialgebroid rather than a Hopf algebroid.

\smallskip

\begin{rem}\label{H1grading}{\rm
We assign a grading to a bialgebroid $(\cA^{(0)},\cH^{(1)})$ by defining, for an arrow $\gamma$ 
in the semigroupoid the degree as the maximal length of a factorization of $\gamma$, 
$\deg(\gamma)=\max\{ n\geq 1\,|\, \exists \gamma=\gamma_1\circ\cdots \circ \gamma_n \}$.
In the dual algebra we assign $\deg(\delta_\gamma)=\deg(\gamma)$, with $\delta_\gamma$ the
Kronecker delta, and
$\deg(\prod_i \delta_{\gamma_i})=\sum_i \deg(\delta_{\gamma_i})$. The
coproduct $\Delta(\delta_\gamma)=\delta_\gamma \otimes 1 + 1 \otimes \delta_\gamma +
\sum_{\gamma=\gamma_1\circ \gamma_2} \delta_{\gamma_1}\otimes \delta_{\gamma_2}$
has the terms $\delta_{\gamma_1}$, $\delta_{\gamma_2}$ of lower degrees. So
we set $\cH^{(1)}=\oplus_{n\geq 0} \cH^{(1)}_n$ with $\cH^{(1)}_0=\Q$ and
$\cH^{(1)}_n$ spanned by the elements of degree $n$, compatibly with
product and coproduct operations. }\end{rem}

\smallskip

\begin{lem}\label{MergeOid}
The data $\cA^{(0)}=(\cV(\fF_{\cS\cO_0}), \sqcup)$ and $\cH^{(1)}=(\cD\cM,\sqcup)$, 
define a bialgebroid. 
\end{lem} 

\proof The algebra $\cH^{(1)}=(\cD\cM,\sqcup)$ dual to the arrows $\cC^{(1)}$ 
 is the same algebra of Merge derivations introduced in Definition~\ref{DMring}.
 We can identify elements $X=\sum_i a_i \varphi_{A_i}$ in $\cD\cM$ with finitely
 supported functions $X=\sum_i a_i \delta_{\varphi_{A_i}}$ on the set of derivations of
 the form \eqref{FMF}, \eqref{FMFchain}, with $\delta_{\varphi_{A_i}}$ the Kronecker 
 delta. The left and right $\cA^{(0)}$-module structures that correspond to the source and target 
 maps are determined by 
 $$ \eta_s(F) \varphi_A=\left\{ \begin{array}{ll} \varphi_A & s(\varphi_A)=F \\ 0 & \text{otherwise}
  \end{array}\right. 
  \ \ \ \ \  \eta_t(F) \varphi_A=\left\{ \begin{array}{ll} \varphi_A & t(\varphi_A)=F \\ 0 & \text{otherwise}  \end{array}\right. $$
  The coproduct $\Delta: \cH^{(1)} \to \cH^{(1)}\otimes_{\cA^{(0)}}   \cH^{(1)}$ is given by
 $$ \Delta(\delta_{\phi_A})=\delta_{\phi_A}\otimes 1 + 1 \otimes \delta_{\phi_A} +\sum_{\phi_A=\phi_{A_1}\circ \phi_{A_2}} \delta_{\phi_{A_1}}
 \otimes \delta_{\phi_{A_2}}\, , $$
 where for $\phi_{A_2}=(F \stackrel{\fM_{A_2}}{\to} F')$ and $\phi_{A_1}=(F' \stackrel{\fM_{A_1}}{\to} F'')$
 the composition is given by 
 $$ \phi_{A_1}\circ \phi_{A_2} =(F \stackrel{\fM_{A_1\circ A_2}}{\to} F'')\, , $$
 where $\fM_{A_1\circ A_2}=\fM_{A_1}\circ \fM_{A_2}$ denotes 
 the set of all compositions of a chain of Merge derivations
 in the set $A_2$ followed by one in $A_1$.
\endproof

\smallskip

\begin{rem}\label{HandH1}{\rm
Note that the bialgebroid of Lemma~\ref{MergeOid} only uses the multiplication 
$(\cV(\fF_{\cS\cO_0}), \sqcup)$ of the Hopf algebra $\cH$ of workspaces, and the
comultiplication of $\cH$ does not appear in the expression for the 
coproduct on $\cH^{(1)}$. The
coproduct of $\cH$, however, is also encoded in the bialgebroid, 
as it is built into the arrows of $\cH^{(1)}$, since the
Merge operations $\fM_{S,S'}$ that occur in the arrows are of the form \eqref{MergeSS}, so
that terms of the coproduct of $\cH$ will contribute to arrows.
}\end{rem}

\smallskip
\subsection{Bialgeroids and Rota-Baxter algebroids}\label{RBoidSec}

In order to simultaneously extend our setting with Rota--Baxter algebras (and semirings)
and Birkhoff factorization of maps from Hopf algebras, and the setting of semiring parsing
in semantics, we introduce a version of Birkhoff factorization for algebroids.

\subsubsection{Algebroids and directed graph schemes}
In our setting, we will take a different viewpoint on the notion of {\em algebroid}
than what is more commonly used in mathematics. The common definition
of an algebroid (over a field $K$) is just a $K$-linear category, where the
operation of morphism composition is the multiplication part of the algebroid
and the linear structure on the spaces of morphisms provides the addition part. 
However, in view of our use above of the notions of Hopf algebroid and
bialgebroid, of Definition~\ref{HopfOid}, it is natural to think of a  commutative
algebroid simply in the following way.

\begin{defn}\label{AlgOid}{\rm An algebroid is 
a pair of commutative algebras $(\cA,\cE)$ with two
morphisms $\eta_s,\eta_t: \cA\to \cE$ that give 
$\cE$ the structure of bimodule over $\cA$ and a morphism of $\cA$-bimodules
$\epsilon: \cE \to \cA$ with $\epsilon \eta_s=\epsilon \eta_t =1_{\cA}$.
A morphism $f: (\cA_1,\cE_1) \to (\cA_2,\cE_2)$ is a pair of morphisms
of commutative algebras $f_V: \cA_1\to \cA_2$ and $f_E: \cE_1\to \cE_2$
with $\eta_{s,2}\circ f_V =f_E\circ \eta_{s,1}$, $\eta_{t,2}\circ f_V =f_E\circ \eta_{t,1}$
and $f_V \circ \epsilon_1 = \epsilon_2 \circ f_E$.
}\end{defn}

A way of thinking of
this notion of algebroid is as the notion of a dual to directed graphs.
In other words our algebroids are directed graph schemes, as
can be seen immediately in the following way.

\begin{lem}\label{DualG}
Let $(\cA,\cE)$ be a commutative algebroid in the sense of Definition~\ref{AlgOid}. 
Then for every other commutative algebra $\cR$ the sets 
$V(\cR)=\Hom(\cA,\cR)$ and $E(\cR)=\Hom(\cE,\cR)$ are the sets
of vertices and edges of a directed graph $G(\cR)$ with source and target maps
$s,t: E(\cR)\to V(\cR)$ determined by the morphisms $\eta_s,\eta_t: 
\cA\to \cE$, and where each vertex $v\in V(\cR)$ has a looping edge
$e_v \in E(\cR)$ with $s(e_v)=t(e_v)=v$. A morphism of
algebroids induces a morphism of directed graphs.
\end{lem}

\proof A directed graph $G$ is a functor from the category ${\bf 2}$ to Sets,
with two objects $V,E$ and two non-identity morphisms $s,t: E \to V$.
The assignment $G(\cR): V \mapsto \Hom(\cA,\cR)$ and
$G(\cR): E \mapsto \Hom(\cE,\cR)$ and $G(\cR): s \mapsto \eta_s^*$
$G(\cR): t \mapsto \eta_t^*$, with $\eta_i^*(\phi)=\phi\circ \eta_i$, for
$\phi\in \Hom(\cE,\cR)$, determine such a functor. The inclusion
of the looping edges $e_v$ in $\Hom(\cE,\cR)$ is given by 
$e_v=v\circ \epsilon$, with $v\in \Hom(\cA,\cR)$. A morphism
of directed graph $\alpha: G_2\to G_1$ is a natural transformation of the functors 
from ${\bf 2}$ to Sets, that is a pair of maps $\alpha_V: \Hom(\cA_2,\cR)\to \Hom(\cA_1,\cR)$
and $\alpha_E: \Hom(\cE_2,\cR)\to \Hom(\cE_1,\cR)$ such that $s\circ \alpha_E=\alpha_V \circ s$
and $t\circ \alpha_E=\alpha_V \circ t$. A morphism $f: (\cA_1,\cE_1) \to (\cA_2,\cE_2)$ of
algebroids determines such a natural transformation with $\alpha_V=f_V^*$ and
$\alpha_E=f_E^*$. The additional property $f_V \circ \epsilon_1 = \epsilon_2 \circ f_E$
ensures that a looping edge $e_v$ in $\Hom(\cE_2,\cR)$ is mapped to $\alpha_E(e_v)=e_{\alpha_V(v)}$
in $\Hom(\cE_1,\cR)$. 
\endproof 

\smallskip

\begin{cor}
A bialgebroid $(\cA,\cE)$ is a commutative algebroid with the property that
the graphs $G(\cR)$ are categories (that is, they are directed graphs
satisfying reflexivity and transitivity).
\end{cor}

\proof
A directed graph $G$ is a category (with objects the vertices and
morphisms the directed edges) if and only if it is the directed graph
of a preorder, namely if it satisfies reflexivity and transitivity. In
other word, a directed graph where every vertex has a looping
edge attached to it, and if there is a pair of edges $e,e'$ with
$s(e)=v$, $t(e)=s(e')$ and $t(e')=v'$ then there exists an edge $\tilde e$
with $s(\tilde e)=v$ and $t(\tilde e)=v'$. The coproduct of the
bialgebroid ensures that the graphs $G(\cR)$ are transitive,
while reflexivity is already a property of directed graphs
determined by algebroids. 
\endproof

\smallskip
\subsubsection{Rota--Baxter algebroids}\label{RBoidSec2}

The notion generalizing the Rota--Baxter algebra structure in this
setting is given by the following.

\begin{defn}\label{RBOid} {\rm
A commutative Rota--Baxter algebroid of weight $-1$ is a commutative algebroid 
$(\cA,\cE)$ as in Definition~\ref{AlgOid}, together with a pair of
maps $R=(R_V, R_E)$ with $R_V\in {\rm End}(\cA)$ an algebra homomorphism
and $R_E: \cE \to \cE$ a linear map that satisfies 
\begin{equation}\label{RBbimod}
R_E (\eta_s(a)\cdot \xi)= \eta_s(R_V(a)) \cdot R_E(\xi) \, \ \ \ \  
R_E(\eta_t(a)\cdot \xi) =\eta_t(R_V(a))\cdot R_E(\xi) \, ,
\end{equation}
for all $a\in \cA$ and $\xi\in \cE$, with $\cdot$ the algebra product in $\cE$, 
and $\epsilon \circ R_E=R_E\circ \epsilon$, and that satisfies the Rota--Baxter relation of weight $-1$,
\begin{equation}\label{RBminusone}
R_E(\xi) \cdot R_E(\zeta) = R_E (R_E(\xi) \cdot \zeta) + R_E( \xi \cdot R_E(\zeta)) 
- R_E (\xi \cdot \zeta)\, .
\end{equation}
We moreover require a normalization condition, that $R_E(1_\cE)=0$ or $R_E(1_\cE)=1_\cE$,
for $1_\cE$ the unit of the algebra $\cE$.
}\end{defn}

\smallskip

\begin{lem}\label{compatRB}
The Rota--Baxter structure of Definition~\ref{RBOid}  has the following properties.
\begin{enumerate}
\item The condition \eqref{RBbimod} replaces the conditions $\eta_s R_V = R_E \eta_s$
and $\eta_t R_V = R_E \eta_t$ and is implies by these conditions in the case
where $R_E$ is an algebra homomorphism. 
\item The normalization condition that $R_E(1)\in \{ 0, 1\}$ together with
the conditions \eqref{RBbimod} and \eqref{RBminusone} imply that
$R_E$ also satisfies
\begin{equation}\label{bimodRE1}
R_E (R_V(\eta_s(a))\cdot \xi)= R_V(\eta_s(a)) \cdot R_E(\xi) \, \ \ \ \  
R_E(R_V(\eta_t(a)) \cdot \xi) =R_V(\eta_t(a))\cdot R_E(\xi) \,  ,
\end{equation}
for all $a\in \cA$ and $\xi \in \cE$, that is, $R_E$ is a bimodule
homomorphism when $\cE$ is viewed as a bimodule over the
subalgebra $R_V(\cA)$. 
\item If $R_V\in {\rm Aut}(\cA)$ is an algebra 
automorphism, then \eqref{RBbimod} and \eqref{RBminusone} with 
$R_E(1_\cE)\in \{ 0, 1\}$ imply that $R_E$ is a bimodule homomorphism
of $\cE$ as a $\cA$-bimodule.
\end{enumerate}
\end{lem}

\proof (1) If $R_E$ is an algebra homomorphism then the conditions
$\eta_s R_V = R_E \eta_s$ and $\eta_t R_V = R_E \eta_t$ imply
that
$$ R_E(\eta_s(a)\cdot \xi) =R_E(\eta_s(a)) \cdot R_E( \xi) =
\eta_s(R_V(a))\cdot R_E(\xi) $$ 
and similarly for $\eta_t$.

(2) If $R_E$ satisfies \eqref{RBminusone}, then the subspaces $R_E(\cE)$ 
and $(1-R_E)(\cE)$ of $\cE$ are (possibly non-unital) subalgebras. If $R_E(1)\in \{ 0, 1\}$ then
either $R_E(\cE)\subset \cE$ is unital and $(1-R_E)(\cE)$ is not, or viceversa.
If, moreover,  $R_E$ also satisfies \eqref{RBbimod}, then  $(\cA,R_E(\cE))$ and $(\cA,(1-R_E)(\cE))$ 
are subalgebroids of $(\cA,\cE)$ with the induced 
maps $\eta_s,\eta_t,\epsilon$. Indeed,
the Rota--Baxter identity \eqref{RBbimod} ensures that the product
$R_E(\xi)\cdot R_E(\zeta)$ is in the range $R_E(\cE)$ for all $\xi,\zeta\in \cE$,
hence $R_E(\cE)\subset \cE$ is a (possibly non-unital) subalgebra, and
similarly for $(1-R_E)\cE$. If $R_E(1)=0$ then $(1-R_E)\cE$ is unital and
$R_E(\cE)$ is not and vice-versa if $R_E(1)=1$. 
Note then that conditions $R_E(1_\cE)\in \{ 0, 1\}$ and \eqref{RBminusone} imply 
that the linear map $R_E$ is a projector, namely $R^2_E=R_E$. In fact
by \eqref{RBminusone} we have
$$ R_E(R_E(\xi))=R_E(R_E(\xi)\cdot 1)) = R_E(\xi)\cdot R_E(1) + R_E(\xi \cdot 1) - R_E(\xi \cdot R_E(1)) $$
$$ = R_E(\xi)\cdot (1+ R_E(1)) - R_E(\xi \cdot R_E(1))\, , $$
where if $R_E(1)=0$ or $R_E(1)=1$ we get $R_E^2(\xi)=R_E(\xi)$. 
Applying condition \eqref{RBminusone} to a pair with $\xi=\eta_s(a)$ gives (using condition \eqref{RBbimod} )
$$ R_E(\eta_s(a)) \cdot R_E(\zeta) = R_E (R_E(\eta_s(a)) \cdot \zeta) + R_E( \eta_s(a) \cdot R_E(\zeta)) 
- R_E (\eta_s(a) \cdot \zeta) $$
which gives 
$$ \eta_s(R_V(a)) \cdot R_E(\zeta) = R_E (\eta_s(R_V(a)) \cdot \zeta) + \eta_s(R_V(a)) R_E^2(\zeta) -
 \eta_s(R_V(a)) R_E(\zeta)\, . $$
Since we are also assuming that $R_E(1_\cE)\in \{ 0, 1\}$, we have $R_E^2(\zeta)=R_E(\zeta)$
so we obtain
$$ \eta_s(R_V(a)) \cdot R_E(\zeta) = R_E (\eta_s(R_V(a)) \cdot \zeta)\, , $$
and similarly with $\eta_t$, so that \eqref{bimodRE1} holds, for all $a\in \cA$ and $\zeta\in \cE$.

(4) If $R_V$ is an automorphism of $\cA$ rather than just an endomorphism, then this also implies
\begin{equation}\label{bimodRE}
R_E (\eta_s(a)\cdot \xi)= \eta_s(a) \cdot R_E(\xi) \, \ \ \ \  
R_E(\eta_t(a)\cdot \xi) =\eta_t(a)\cdot R_E(\xi) \,  ,
\end{equation}
for all $a\in \cA$ and $\zeta\in \cE$. 
\endproof

\smallskip

A simple source of examples of Rota--Baxter algebroids is obtained by considering
functions on the edges of a directed graph, with values in a Rota--Baxter algebra.
This means that, in these examples, the Rota--Baxter operator is acting only on the 
coefficients of functions. The following is a direct consequence of the definition of
Rota--Baxter algebroids.

\smallskip

\begin{lem}\label{classAE}
Let $G$ be a directed graph and let $(\cR,R)$ be a Rota--Baxter algebra of weight $-1$.
Consider pair of algebras $(\cA,\cE)$ with $\cA=\Q[V_G]$ (finitely supported $\Q$-valued functions on the
set $V_G$ of vertices of $G$) and $\cE= \Q[V_G]\otimes_\Q \cR$, with morphisms $\eta_s, \eta_t: \cA\to \cE$
given by precomposition with source and target maps $s,t: E_G \to V_G$. The maps $R_V={\rm id}$ on $\cA$
and $R_E=1\otimes R$ give $(\cA,\cE)$ the structure of a Rota--Baxter algebroid of weight $-1$.
\end{lem}

\smallskip
\subsubsection{Rota--Baxter semiringoids}\label{RBoidSec3}

There is a direct generalization of this notion of Rota--Baxter algebroids, and the class of
examples of Lemma~\ref{classAE} to the case where algebras are replaced by semirings.
We will refer to those as {\em Rota--Baxter semiringoids}. The definition and properties
are analogous to the algebroid case, in the same way in which we generalized from
Rota--Baxter algebras to Rota--Baxter semirings in \S \ref{AlgRenSec}. 
We will focus in particular on the analog of the 
examples of Lemma~\ref{classAE}.

\smallskip

The category of commutative semirings, with initial object the semiring $\Z_{\geq 0}$ of non-negative integers,
is dual to the category of semiring schemes, that is, affine schemes over ${\rm Spec}(\Z_{\geq 0})$. The full
subcategory of idempotent commutative semirings, with initial object $\cB$, the Boolean semiring of \eqref{BooleanB},
is dual to the category of affine schemes over ${\rm Spec}(\cB)$. 

\smallskip

\begin{defn}\label{RBsemiringoid} {\rm
A semiringoid is the datum $(\cA,\cE)$ of two commutative semirings 
with semiring homomorphisms $\eta_s,\eta_t: \cA \to \cE$ that give $\cE$ the structure of bi-semimodule
over the semiring $\cA$ and with a bi-semimodule homomorphism $\epsilon: \cE\to\cA$ with $\epsilon\eta_s=\epsilon\eta_t=1_{\cA}$.
A morphism $(\cA_1,\cE_1) \to (\cA_2,\cE_2)$ of semiringoids is a pair of semiring homomorphisms $f_V: \cA_1\to \cA_2$
and $f_E: \cE_1\to \cE_2$ with $\eta_{s,2}\circ f_V =f_E\circ \eta_{s,1}$, $\eta_{t,2}\circ f_V =f_E\circ \eta_{t,1}$
and $f_V \circ \epsilon_1 = \epsilon_2 \circ f_E$. A Rota--Baxter semiringoid of weight $+1$ is a semiringoid $(\cA,\cE)$
endowed with a semiring endomorphism $R_V: \cA\to \cA$ and an $R_E: \cE \to \cE$ a $\Z_{\geq 0}$-linear map (morphism
of  $\Z_{\geq 0}$-semimodules) satisfying
\begin{equation}\label{RBbimodSemi}
R_E (\eta_s(a)\odot \xi)= \eta_s(R_V(a)) \odot R_E(\xi) \, \ \ \ \  
R_E(\eta_t(a)\odot \xi) =\eta_t(R_V(a))\odot R_E(\xi) \, ,
\end{equation}
for all $a\in \cA$ and $\xi\in \cE$, with $\odot$ the semiring product in $\cE$, 
and $\epsilon \circ R_E=R_E\circ \epsilon$, and that satisfies the Rota--Baxter relation of weight $+1$,
\begin{equation}\label{RBplusoneSemi}
R_E(\xi) \odot R_E(\zeta) = R_E (R_E(\xi) \odot \zeta) \boxdot R_E( \xi \odot R_E(\zeta)) 
\boxdot R_E (\xi \odot \zeta)\, ,
\end{equation}
with $\boxdot$ and $\odot$ the semiring sum and product in $\cE$. The case of a Rota--Baxter structure
of weight $-1$ is similar, with \eqref{RBplusoneSemi} replaced by
\begin{equation}\label{RBminusoneSemi}
R_E(\xi) \odot R_E(\zeta) \boxdot R_E (\xi \odot \zeta)= R_E (R_E(\xi) \odot \zeta) \boxdot R_E( \xi \odot R_E(\zeta)) 
\, .
\end{equation}
We moreover require the normalization condition, that $R_E(1_\cE)=0_\cE$ or $R_E(1_\cE)=1_\cE$,
for $1_\cE$ the unit of the multiplicative monoid and $0_\cE$ the unit of the additive monoid of $\cE$. }
\end{defn}

\medskip

When considering semiringoids with commutative idempotent semirings, one can drop the 
$\Z_{\geq 0}$-linearity requirement for $R_E$ and only require that $R_E$ is a morphism of 
$\cB$-semimodules (Boolean semimodules). 

\smallskip

\begin{rem}\label{semioids}{\rm
Note the the notion of semiringoid we use in Definition~\ref{RBsemiringoid} differs from
another commonly used notion, where a semiringoid is a small category $\cC$ where 
all the Hom-sets $\Hom_\cC(X,Y)$, for
$X,Y\in {\rm Obj}(\cC)$, are commutative monoids with bilinear composition of
morphisms, and all the End-sets ${\rm End}_\cC(X)=\Hom_\cC(X,X)$ are semirings. 
}\end{rem}

\smallskip

We have then an analog for semiringoids of the class of Rota--Baxter algebroids
of Lemma~\ref{classAE}. Again this follows directly from Definition~\ref{RBsemiringoid}.

\begin{lem}\label{classAEsemi}
Let $G$ be a directed graph and let $(\cR,R)$ be a Rota--Baxter semiring of weight $+1$ (or $-1$).
Consider the pair of semirings $(\cA,\cE)$ with $\cA=\Z_{\geq 0}[V_G]$ (finitely supported $\Z_{\geq 0}$-valued functions on the
set $V_G$ of vertices of $G$) and $\cE= \Z_{\geq 0}[V_G]\otimes_{\Z_{\geq 0}} \cR$, with morphisms $\eta_s, \eta_t: \cA\to \cE$
given by precomposition with source and target maps $s,t: E_G \to V_G$. The maps $R_V={\rm id}$ on $\cA$
and $R_E=1\otimes R$ give $(\cA,\cE)$ the structure of a Rota--Baxter semiringoid of weight $+1$ (or $-1$). In the case
where $\cR$ is a commutative idempotent semiring, we can replace this construction with $\cA=\cB[V_G]$ (Boolean
functions on $V_G$) and $\cE= \cB[V_G]\otimes_{\cB} \cR$, to obtain a Boolean Rota--Baxter semiringoid (a semiringoid
over commutative idempotent semirings). 
\end{lem}

\smallskip
\subsubsection{Birkhoff factorization in algebroids and semiringoids}\label{BfactOidSec}

We then consider morphisms of algebroids $\Phi: (\cA^{(0)},\cH^{(1)}) \to (\cA,\cE)$
from a Hopf algebroid to an an algebroid 
with a Rota--Baxter structure $(R_V, R_E)$ of weight $-1$. 
The target algebroid $(\cA,\cE)$ does {\em not} have a compositional
structure, in the sense that the directed graph (graph scheme) $G$ dual
to the algebroid does not have, in general, the transitive property: 
given two directed edges where the target of the first is the source of
the second it is not necessarily the case that there is also a edge from
the source of the first to the target of the second. The source $(\cA^{(0)},\cH^{(1)})$
has the compositional structure, which is encoded in the coproduct as
bialgebroid, which is the convolution product of the groupoid algebra $\cH^{(1)}$. 
As in the case of algebras, the convolution structure on $\cH^{(1)}$ together
with the Rota--Baxter structure on $(\cA,\cE)$ will perform the factorization
of $\Phi: (\cA^{(0)},\cH^{(1)}) \to (\cA,\cE)$ which accounts for the induced
compositional structure on the image. 

\smallskip

\begin{lem}\label{PhiEpm}
Let $(\cA^{0)},\cH^{(1)})$ be a Hopf algebroid and let $(\cA,\cE)$ be an algebroid 
with a Rota--Baxter structure $(R_V, R_E)$ of weight $-1$.  Given a
morphism $\Phi: (\cA^{(0)},\cH^{(1)}) \to (\cA,\cE)$ of algebroids, 
there is a pair $\Phi_\pm$ with $\Phi_{\pm,V}=\Phi_V$ and
$\Phi_{+,E}(f)=(\Phi_{-,E}\star \Phi_E)(f)=(\Phi_{-,E}\otimes \Phi_E)(\Delta f)$ for
all $f\in \cH^{(1)}$, where we have
$$ \Phi_{-,E}(f)=- R_E ( \tilde\Phi_E(f)) \ \ \ \text{ with } \ \   \tilde\Phi_E(f)=   \Phi_E(f)  + \sum \Phi_{-,E}(f') \Phi_E(f'') \, , $$
for $\Delta(f)=f\otimes 1 + 1 \otimes f +\sum f' \otimes f''$, and with $\Phi_{+,E}(f)=(1-R_E) (\tilde\Phi_E(f))$.
\end{lem}

\proof The argument for showing that the maps $\Phi_{\pm,E}: (\cA^{0)},\cH^{(1)})\to (\cA,\cE_\pm)$ with
$\cE_+=(1-R_E)(\cE)$ and $\cE_-=R_E(\cE)$ are algebroid homomorphisms follows closely the same
argument for Rota--Baxter algebras of weight $-1$, as in Theorem~1.39 of \cite{CoMa}.  The factorization
identity $\Phi_{+,E}=\Phi_{-,E}\star \Phi_E$ follows from $\Phi_{+,E}=(1-R_E) \tilde\Phi_E$ and $\Phi_{-,E}=- R_E \tilde\Phi_E$
and the expression for $\tilde\Phi_E$ in terms of the coproduct $\Delta$.
\endproof

\smallskip

We consider in particular the case where the Rota--Baxter algebroids are as in
Lemma~\ref{classAE}.

\begin{lem}\label{BirkAE}
The Birkhoff factorization of an algebroid homorphism $\Phi: (\cA^{(0)},\cH^{(1)}) \to (\cA,\cE)$,
with $(\cA,\cE)$ a Rota--Baxter algebroid as in Lemma~\ref{classAE} and $(\cA^{(0)},\cH^{(1)})$ a bialgebroid, 
consists of a map of directed graphs (graph schemes) $\alpha: G \to \cG$, with $G$ dual to $(\cA,\cE)$ and $\cG$
dual to $(\cA^{(0)},\cH^{(1)})$, so that $\Phi_E(f)=f\circ \alpha$ for $f\in \cH^{(1)}$, with the factorization
$\Phi_{E,-}$ mapping $f=\delta_\gamma$ for $\gamma$ an arrow in $\cG$ to the function  
$\Phi_{E,-}(\delta_\gamma)$ that acts on a combination $\sum_i a_i e_i$ with $e_i\in E_G$ as
\begin{equation}\label{PhiminusRE}
 \Phi_{E,-}(\delta_\gamma) (\sum_i a_i e_i) =-( \sum_{\alpha(e)=\gamma} R_E(a_e)+ \sum_{\alpha(e_1)\circ \alpha(e_2)=\gamma}
R_E(R_E(a_{e_1}) a_{e_2}) + \cdots $$ $$ + \sum_{\alpha(e_1)\circ \cdots \circ \alpha(e_n)=\gamma} 
R_E(\cdots (R_E(a_{e_1}) \cdots )a_{e_n} ))\, . 
\end{equation}
\end{lem}

\proof
The algebroid $(\cA,\cE)$ is associated to a directed graph (graph scheme) $G$ and
the bialgebroid $(\cA^{0)},\cH^{(1)})$ is associated to a semigroupoid $\cG$ 
(equivalently a graph that is reflexive, symmetric, and transitive). 
A morphism $\Phi: (\cA^{(0)},\cH^{(1)}) \to (\cA,\cE)$ of algebroids is equivalent to the datum
of a map of directed graphs $\alpha: G \to \cG$. The map $\Phi_E: \cH^{(1)}\to \cE$ then
is given by $\Phi_E(f)=f\circ \alpha$. It suffices to consider the case of $f=\delta_\gamma$
for some $\gamma \in \cG^{(1)}$, as in general $f \in \Q[\cG]$ will be a product of
linear combinations of delta functions $\delta_\gamma$. In the case where the Rota
Baxter operator of weight $-1$ is the identity, the Bogolyubov preparation is of the form
$$ \tilde \Phi_E (\delta_\gamma) =\delta_\gamma \circ \alpha +\sum_{\gamma=\gamma_1\circ \gamma_2} \delta_{\gamma_1}\circ \alpha\,\, \cdot \,\delta_{\gamma_2}\circ \alpha +\cdots + \sum_{\gamma= \gamma_1\circ \cdots \circ \gamma_n} \delta_{\gamma_1}\circ \alpha \cdots \delta_{\gamma_n} \circ \alpha \, , $$
with $n=\deg(\gamma)$, which is then equal to
\begin{equation}\label{tildePEdelta}
\tilde \Phi_E (\delta_\gamma)=  \sum_{e\in E_G\,:\, \alpha(e)=\gamma} \delta_e + \cdots + \sum_{e_1,\ldots, e_n \in E_G\,:\, \gamma=\alpha(e_1)\circ \cdots \circ \alpha(e_n)}
\delta_{e_1}\cdots \delta_{e_n} \, , 
\end{equation}
so that we have, for a collection of edges $e_i\in E_G$, 
$$ \tilde \Phi_E (\delta_\gamma) (\sum_i a_i e_i) =\sum_{\alpha(e)=\gamma} a_e + \sum_{\alpha(e_1)\circ \alpha(e_2)=\gamma}
a_{e_1} a_{e_2} + \cdots + \sum_{\alpha(e_1)\circ \cdots \circ \alpha(e_n)=\gamma} a_{e_1} \cdots a_{e_n} \, . $$
In the case of a Rota Baxter operator  $R_E$ of weight $-1$ that is not the identity, we similarly get \eqref{PhiminusRE}.
\endproof

In the case of the bialgebroid $(\cA^{0)}=\cV(\fF_{\cS\cO_0}),\cH^{(1)}=\cD\cM)$  
of Merge derivations as in Lemma~\ref{MergeOid}, with $\cG$ the associated
semigroupoid, we can regard the choice of a map of directed graphs $\alpha: G \to \cG$
from some graph $G$ as a chosen {\em diagram of Merge derivations} modeled on $G$.
The algebroid homomorphism $\Phi_E(f)=f\circ \alpha$ describes all the ways of
obtaining a certain Merge derivation $\gamma$ in $\cD\cM$ as an arrow in $G$,
$\Phi_E(\delta_\gamma)=\sum_{e\,:\, \alpha(e)=\gamma}\delta_e$. The Bogolyubov
preparation with the identity Rota-Baxter operator lists all the possible ways
of obtaining $\gamma$ as a composition of Merge derivations through arrows in $G$,
as in \eqref{tildePEdelta}. Consider an element $\sum_i \lambda_i e_i$ as a weighted
combination of edges in the diagram $G$. For example, if the coefficients $\Lambda=(\lambda_e)_{e\in E}$
are a probability distribution on the edges of $G$, the value (using the identity as Rota--Baxter operator)
$$ \tilde \Phi_E (\delta_\gamma) (\sum_e \lambda_e \,\,  e)=\sum_{\alpha(e)=\gamma} \lambda_e +
 \cdots + \sum_{\alpha(e_1)\circ \cdots \circ \alpha(e_n)=\gamma} \lambda_{e_1} \cdots \lambda_{e_n} $$
is the total probability of realizing $\gamma$ through the diagram $E$, as a sum of 
the probabilities of all the possible ways of obtaining $\gamma$ as a composition 
of arrows in the image of edges of $E$ drawn the assigned probabilities $\lambda_e$. 

\smallskip

The setting for algebroids generalizes to semiringoids as in the case of the generalization
from Rota--Baxter algebras to Rota--Baxter semirings.

\begin{cor}\label{PhiEpmSemi}
The Birkhoff factorization of Lemma~\ref{PhiEpm} extends to the case of Rota--Baxter semiringoids
of weight $+1$, with a morphism of semiringoids $\Phi: (\cA^{(0)},\cH^{(1)})^{semi} \to (\cA,\cE)$
from a subdomain of a bialgebroid $(\cA^{(0)},\cH^{(1)})$ that has semiringoid structure and is
closed under coproduct $\Delta$. The terms of the factorization are as in the case of semirings
(Proposition~\ref{semiRBfact}) with
$\Phi_{E,-}(f)=  R(\tilde\Phi_E(f))  =R(\Phi_E(f) \boxdot \phi_-(f')\odot \phi(f''))$ with
$\Delta(f)=f\otimes 1 + 1 \otimes f +\sum f'\otimes f''$.
\end{cor}

\smallskip
\subsection{Parsing semirings and Merge derivations}\label{ParSemSec}

After this preparatory work, we can now formulate the analog of parsing semirings in the setting of Merge
derivations, replacing the usual formulation for context-free grammars, as in \cite{Goodman}.

Here we consider a map $\Phi: (\cA^{(0)},\cH^{(1)})^{semi} \to (\cA,\cE)$, where 
$(\cA,\cE)$ is a Rota--Baxter semiringoid and $(\cA^{(0)},\cH^{(1)})^{semi}$ is a subdomain
of the bialgebroid  $(\cA^{(0)}=\cV(\fF_{\cS\cO_0}),\cH^{(1)}=\cD\cM)$  of Merge
derivations that has a semiringoid structure, so that $\Phi$ is a morphism of semiringoids.
We assume that the target $(\cA,\cE)$ is of the form as in Lemma~\ref{classAEsemi}, with
$(\cR,R)$ a Rota--Baxter semiring, such as the max-plus semiring 
$(\R\cup \{ -\infty \}, \max, +)$ with $R$ given by the ReLU operator, or
the semiring $([0,1],\max, \cdot)$ with the threshold Rota--Baxter operators $c_\lambda$
that we considered before. Then the map $\Phi$ may be viewed as assigning
a diagram of Merge derivations, through a map $\alpha: G \to \cG$ as above,
and checking all the possible ways of realizing some chain of Merge derivations $\gamma$
through compositions coming from the chosen diagram, weighted by elements in the
given semiring and filtered by the Rota--Baxter operator that acts as a threshold. 

\smallskip

Thus, as above, we start with a chosen a diagram $\alpha: G \to \cG$
of Merge derivations where we have assigned weights $\lambda_e \in \cR$ with
values in the parsing semiring $\cR$, for each edge $e\in E_G$.  For example,
if $\cR=([0,1],\max, \cdot)$, we can think of $\lambda_e$ as a probability (or
frequency counting) of occurrence of $e$ in the diagram of derivations. If 
$\cR=(\R\cup \{ -\infty \}, \max, +)$ we can think of $\lambda_e$ as being 
real weights assigned to the edges $e$ of the diagram $G$. Then the resulting factorization
\begin{equation}\label{semiPhiEminus}
 \Phi_{E,-}(\delta_\gamma)(\sum_e \lambda_e \, e)=
\sum_{\alpha(e)=\gamma} R_E(a_e)+ \sum_{\alpha(e_1)\circ \alpha(e_2)=\gamma}
R_E(R_E(a_{e_1}) a_{e_2}) + \cdots $$ $$ + \sum_{\alpha(e_1)\circ \cdots \circ \alpha(e_n)=\gamma} 
R_E(\cdots (R_E(a_{e_1}) \cdots )a_{e_n} )
\end{equation}
measures all the possible ways of obtaining the Merge derivation $\gamma$ via
compositions in the chosen diagram with combined weights filtered by $R$.
\begin{itemize}
\item In the case of $\cR=(\R\cup \{ -\infty \}, \max, +)$ with $R=$ReLU, \eqref{semiPhiEminus} lists all the
possibilities with weights of the substructures involved that are above the ReLU threshold.
\item In the case of $\cR=([0,1],\max, \cdot)$ with the threshold $R=c_\lambda$, \eqref{semiPhiEminus}
lists all the possible realizations of the derivation $\gamma$ in the diagram that
have probabilities above the threshold $\lambda$ in the substructures involved. 
\item In the case of the Boolean semiring $\cB=(\{ 0, 1 \}, \max, \cdot)$ with $R={\rm id}$,
the factorization \eqref{semiPhiEminus} evaluates the truth value (truth conditions) for the realization of
a derivation $\gamma$ through the diagram $G$ given that the arrows of $G$ have
assigned truth values (truth conditions), in such a way that the composition of arrows in the
derivation corresponds to the AND operation on the respective truth values and the choices of different
paths of derivations to obtain the same $\gamma$ correspond to the OR operation on the
respective truth values.
\end{itemize}

\smallskip

With this we have shown that we can obtain in this way a form of semiring parsing for Merge derivations
that simultaneously generalizes the semiring parsings of \cite{Goodman}, for example
with values in the Boolean or the Viterbi semiring, and also the Birkhoff factorizations of
our initial toy models of syntax-semantics interface discussed in \S \ref{HeadIdRenSec}.

\medskip

\section{Pietroski's compositional semantics}\label{PietSec}

Among the different proposed models of semantics, Pietroski's
compositional model (see for instance \cite{Pietro}, \cite{Pietroski})
is closely linked to the structure of syntax as described by Merge. We
discuss how this approach relates to our model of the syntax-semantics
interface. Our main observation here is that, in our model, it is not
necessary to assume an independent existence {\em within} semantics of
what Pietroski refers to in \cite{Pietro} as the {\em Combine} binary
operation that mimics the functioning of Merge in syntax. The type of
compositional structure postulated by Pietroski in \cite{Pietro} for
semantics {\em follows} in our case from Merge itself acting on the
syntax side of the interface, along with the map $\phi: \cH \to \cR$
together with its Birkhoff factorization.

\smallskip

To see this, we recall briefly the setting of \cite{Pietro}, focusing in particular
on the discussion of the {\em Combine} operation, that is the aspect
more directly connected to our setting.  The general principles for
the compositional structure of semantics articulated in \cite{Pietro}
include the basic idea that ``meanings are instructions to build
concepts,''  that can be articulated in the following way, adapting the
arguments of \cite{Pietro} to the terminology we have been using in
this paper. Lexical items are seen as ``instructions to fetch
concepts." This corresponds to the assumption we made in various
examples discussed in the previous sections, of the existence of a map
$s: \cS\cO_0\to \cS$ from lexical items to a semantic space $\cS$.
One then considers i-expressions, generated by I-language, as building
instructions for the construction of i-concepts, with principles that
govern the combination of i-expressions.  

\smallskip

This fits nicely with our proposal of a syntax-driven syntax-semantics
interface, where the i-expressions are provided, in our setting, by
the syntactic objects $T\in {\rm Dom}(h) \subset \fT_{\cS\cO_0}$. The
corresponding i-concepts are provided in our setting by and their
images $s(T)\in \cS$, under an extension of the map
$s: \cS\cO_0\to \cS$ from $\cS\cO_0$ to ${\rm Dom}(h)\subset \cS\cO$
as discussed in previous sections, together with the corresponding
$\phi(T) \in \cR$, where $\cR$ is an algebraic structure of
Rota-Baxter type associated to the (topological/metric) space $\cS$.
\smallskip

On the side of syntax, the free commutative non-associative magma $\cS\cO={\rm Magma}_{nc,c}(\cS\cO_0,\fM)$
of \eqref{SOmagma} is the main {\em computational structure}, with Merge $\fM$ as the main {\em binary operation}
of structure formation. The resulting hierarchical structures are the syntactic objects $T\in \cS\cO=\fT_{\cS\cO_0}$,
identified with abstract (non-planar) binary rooted trees with leaves labeled by lexical items in $\cS\cO_0$. In
Pietroski's formulation of \cite{Pietro} one considers a parallel form of binary structure formation operation, acting
on the side of semantics. 

\smallskip

The compositional rules for the building of i-concepts via
i-expressions are described in \cite{Pietro} in terms of one basic
non-commutative binary operation, {\em Combine}. This in turn consists
of the composition of two operations, {\em Combine} $=$ {\em Label}
$\circ$ {\em Concatenate}, where, given two i-concepts $\alpha,\beta$
that can be combined in the I-language, one first forms a
concatenation
 $$ \text{{\em Concatenate}}(\alpha,\beta) =\{ \alpha, \beta \} =\Tree[ $\alpha$  $\beta$ ] $$
 of the two and then labels the resulting expression by one of the constituents $\alpha$
 or $\beta$ that plays the role of the head $h(\alpha,\beta)$ of the combined expression
 \begin{equation}\label{Combine}
  \text{{\em Combine}}(\alpha,\beta) =\text{{\em Label}}\circ \text{{\em Concatenate}}(\alpha,\beta) =
 \text{{\em Label}}(\Tree[ $\alpha$  $\beta$ ])=  
 \end{equation}
 $$ \Tree[ .$h(\alpha,\beta)$ $\alpha$ $\beta$ ] $$
 The binary operation {\em Combine} is not symmetric because of the head label.
 Note that this operation is closely modeled on the Merge operation where, given two
 syntactic objects $T_1$ and $T_2$, with the property that 
 $$ T=\fM(T_1,T_2)=\Tree[ $T_1$ $T_2$ ]\in {\rm Dom}(h) \subset \cS\cO=\fT_{\cS\cO_0} \, , $$
 where $\fM$ is the free symmetric Merge and $h$ is a head function, one can assign
 to the abstract tree $T$ a planar structure $T^{\pi_h}$ determined by the head function,
 resulting in a planar tree
 $$ T^{\pi_h}=\fM^{nc}(T_{h(T)}, T') \in \fT^{{\rm planar}}_{\cS\cO_0}\, , $$
 where $T'\in \{T_1, T_2\}$ is the one that does not contain $h(T)$.
 
 \smallskip

 In \cite{Pietro}, the operation \eqref{Combine} is presented, in principle, as a compositional
 operation that takes place in the semantic space $\cS$, hence requiring this space to be
 endowed with its own computational system (at least partially defined), analogous to the Merge
 operation in syntax. As a result, we would have two systems that each
 have a ``Merge'' type operation, one for syntax and one for
 semantics. Besides an issue here with parsimony (we can get by, given
 the model presented here, with just one), this would be different from the case of other types
 of conceptual spaces, such as the perceptual manifolds associated to vision (see for instance \cite{Chung}). 

\smallskip

The most widely studied conceptual spaces and perceptual manifolds are
in the context of vision.  It should be noted that there have been
significant attempts by mathematicians at formulating a compositional
computational model for vision: among these in particular Pattern
Theory, as developed by Grenander, Mumford, et al.~(see for instance
\cite{Gren}, \cite{Gren2}, \cite{Mum}), that has found various
applications, especially in computer vision.  The original approach to
Pattern Theory was based on importing ideas from the theory of formal
languages, especially from the case of probabilistic context-free
grammars. This was further articulated in a proposed ``mathematical
theory of semantics" in \cite{Gren3}. We will not be discussing this
viewpoint in the present paper, but it is important to stress here
that it is still {\em topological and geometric} properties of the
relevant ``semantic spaces" that play a fundamental role in that
setting and that there are serious limitation to the extent to which a
generative model can be adapted to vision in comparison to language.
 
 \smallskip
 \subsection{The Combine operation in Pietroski's semantics}\label{CombineSec}
 
 In terms of the syntax-semantics interface, using the terminology in this paper, a setting
 such as that proposed by Pietroski in \cite{Pietro} would seem to correspond to
 $\phi: \cH^{nc} \to \cH_\cS$, that maps a non-commutative version of the Hopf algebra structure of $\cH$ 
 (responsible for the action of Merge on syntax) to a (possibly partially defined) 
 non-commutative Hopf algebra structure $\cH_\cS$ on the side of the semantic space $\cS$, at least if one desires a {\em Combine}
 operation that fully mimics the Merge operation, including the Internal Merge action. 
 This would be a much stronger requirement than what is needed for the desired type of
 compositionality of i-concepts to take place.

 \smallskip
 
 Indeed, one can view the construction of i-concepts postulated by \cite{Pietro} not as
 the result of a compositional structure on the semantic space $\cS$ itself, but simply as the 
 extension of the map $s: \cS\cO_0\to \cS$ to a map $s: {\rm Dom}(h) \to \cS$, built along
 the lines of what we discussed in Lemma~\ref{phiP}. 
In other words, in this formulation, the i-concept $\text{{\em Concatenate}}(\alpha,\beta)$
where $\alpha=s(T_1)$ and $\beta=s(T_2)$ is well defined if $T=\fM(T_1,T_2)\in {\rm Dom}(h)$
and in that case is simply the image
\begin{equation}\label{CombineMerge}
 \text{{\em Combine}}(\alpha,\beta) := s( \fM(T_1,T_2) ) \in \cS \, , 
\end{equation} 
where the construction of the point $s(T)$ depends on
$s(T_1), s(T_2)$, and on whether the head function satisfies
$h(T)=h(T_1)$ or $h(T)=h(T_2)$.  In other words, it depends on $\cS$
only through the existence of a geodesically convex Riemannian
structure and a semantic proximity function $\bP$, without having to
require any Merge-like computational mechanism on $\cS$ itself. It
suffices that syntax has such an operation and that $\cS$ has a
topological proximity relation (expressed in the case of the
construction we presented in Lemma~\ref{phiP} in terms of a more
specific metric property of convexity).

\smallskip
\subsubsection{The role of idempotents}\label{IdemSec}

One may worry here that the {\em Combine} operation of Pietroski appears
to behave differently from Merge itself. A simple way in which this difference manifests
itself is in the possible presence of idempotent structures. For example, one 
expects that $\text{{\em Combine}}(\alpha,\alpha) =\alpha$, while at the level of Merge
$$ \fM(T,T)=\Tree[ $T$  $T$ ] \neq T \, . $$
This in itself may not constitute an example because we also need a head function
and a structure of the form $\fM(T,T)$ might not admit a head function.
However, by the same principle, one expects cases where
$\text{{\em Combine}}(\alpha,\beta) =\alpha$ (or $\beta$), where the head function is not
an issue, and again this seems to be at odds with the fact that at the 
level of Merge this never happens since $\cS\cO$ is a {\em free} magma, so that 
for all $T,T'$ one has $\fM(T,T')\neq T$ and $\fM(T,T')\neq T'$. This, however,
does not constitute a problem, as it is taken care of in \eqref{CombineMerge}  by 
the structure of the map $s: {\rm Dom}(h) \to \cS$ from $s: \cS\cO_0\to \cS$, other 
than the one described in Lemma~\ref{phiP}. 

\smallskip

In the construction of Lemma~\ref{phiP} we have assumed that the semantic space
we work with has the structure of a geodesically convex Riemannian manifold and
that, for a syntactic object $\fM(T,T')$ the image $s(\fM(T,T'))$ is obtained as a form
of convex interpolation between the images $s(T)$ and $s(T')$. In this setting, 
the location of the point $s(\fM(T,T'))$ on the geodesic arc between 
$s(T)$ and $s(T')$ depends on a function $\bP(s(T),s(T'))$ measuring syntactic
relatedness. Depending on the nature of this function $\bP$, one expects that
there will be points $s,s'\in\cS$ for which $\bP(s,s')=0$ or $\bP(s,s')=1$, so that 
the point $s(\fM(T,T'))$ coincides with one of the endpoints $s(T)$ and $s(T')$. 
This gives rise precisely to the type of situation where one obtains
$$ \text{{\em Combine}}(\alpha,\beta) =\alpha \ \  \  \text{ or } \ \ \  \text{{\em Combine}}(\alpha,\beta)=\beta \, , $$
even though $\fM(T,T')\neq T$ and $\fM(T,T')\neq T'$. Note that the function $\bP$, that
is responsible for this difference in behavior between Merge and Combine, does not implement
any computational process itself, but is only an evaluator of topological proximity in semantics.
The only computational process is implemented by syntactic Merge.

\smallskip

This case illustrates the situation where, contrary to the case described in \S \ref{SyntImageSec}
(or the possible situation discussed in \S \ref{TransfSec} below), the image of syntax inside
semantics is {\em not} an embedding. This non-embedding situation is generally expected
when one maps to a semantic model that has a discrete topology (a Boolean assignment for
example). In the case we describe here, where the function $s: {\rm Dom}(h) \to \cS$ is
based on geodesic convexity, one could in principle entirely avoid idempotent cases and assume 
as in in \S \ref{SyntImageSec} that the semantic relatedness $\bP(s(T),s(T'))$ may be very close
to either $0$ or $1$ but not exactly equal to either (see the discussion in \S \ref{SyntImageSec}).

\smallskip
\subsubsection{An example}\label{ExPietSec}

Pietroski's {\em Combine} operation is designed to rule out
improper inferences. We consider here an example to show
how it fits with the formulation we give above.\footnote{We thank
  Norbert Hornstein for this example.}

\smallskip 

Given the sentences
``John ate a sandwich in the basement" and
``John ate a sandwich at noon", 
these two sentences clearly do {\em not} imply 
that ``John ate a sandwich in the basement at noon".

\smallskip

In our setting, consider a sentence with a series of adjuncts to a verb, such as 
``John ate a sandwich in the basement with a spoon at noon.''
We have a Merge-based inductive construction of the map
$s: {\rm Dom}(h) \to \cS$ from $s: \cS\cO_0\to \cS$, of the type 
discussed in \S \ref{SyntImageSec}. This means that, if $\tilde T$ is
the syntactic object associated to the full sentence, we can view it as a structure of the form
$$  \Tree[ $T'$  [ $T$ $\quad T_{1,\ldots,k}$ ] ]$$ 
where a VP $T$ is modified by a series of adjuncts $T_{1,\ldots,k }=\{
T_1,\ldots, T_k \}$ (for simplicity, we do not
draw the full tree structure). 

\begin{figure}[h]
\begin{center}
\includegraphics[scale=0.12]{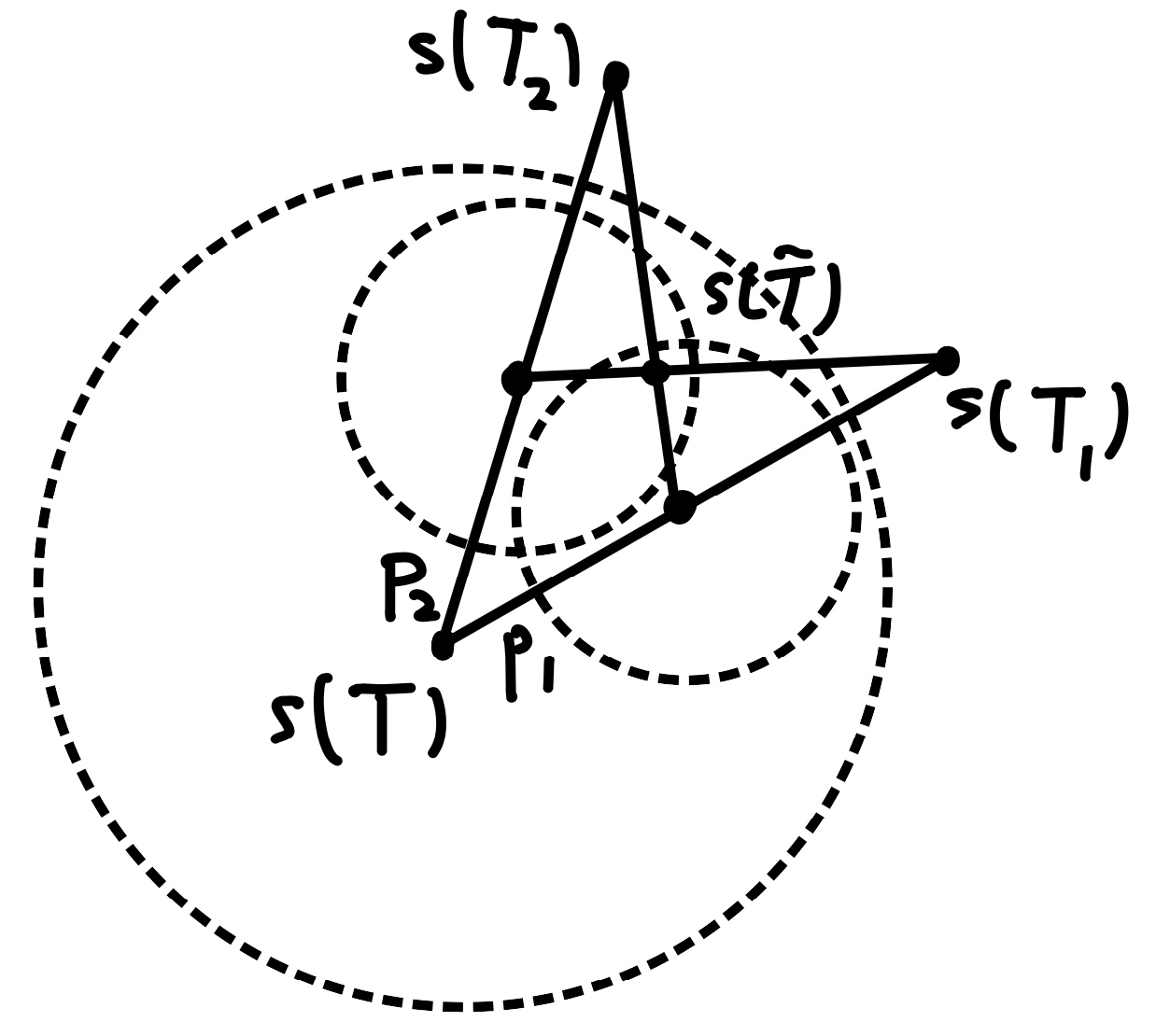}  
\caption{Example: adjuncts to verbs and semantic points. \label{PietroskiExFig}}
\end{center}
\end{figure}

With the construction of  \S \ref{HeadIdRenSec2} and \S \ref{SyntImageSec} of
the extension $s: {\rm Dom}(h) \to \cS$ of the map $s: \cS\cO_0\to \cS$ we obtain
points $s(T) \in \cS$ and $s(T_i) \in \cS$ for each $i=1,\ldots, k$. When we
consider each individual adjunct, the corresponding point
$$ s_i := s(\Tree[ $T$ $T_i$ ]) $$
lies on the geodesic arc in $\cS$ between $s(T)$ and $s(T_i)$, at a
distance $p_i=\bP_\sigma(s(T),s(T_i))$ from $s(T)$, where $\sigma$ is 
the adjunct syntactic relation. In particular, there is a convex geodesic neighborhood of 
the point $s(T)$ in $\cS$ that contains all the points $s_i$. 
When we consider the combinations $T_{i_1 , \ldots, i_r}$ of the adjuncts $T_i$,
this further determines points 
$$ s_{i_1\ldots i_r} =s(\Tree [ $T$  $\quad T_{i_1\ldots i_r}$ ]) \, . $$
These points are contained in the same neighborhood of $s(T)$ and they
are also contained in the intersection of neighborhoods around 
the points $s_i$ with $i\in \{ i_1, \ldots, i_k \}$. A sketch of this relation is
illustrated in Figure~\ref{PietroskiExFig}, with 
$$ \tilde T := \Tree [ $T$  $T_{12}$ ] \,  . $$

\smallskip

Dropping the
more refined metric/convexity structure, and
the fact that the more precise location of this
point depends on syntactic heads and 
evaluation of semantic proximity of the
lexical items involved, if we 
only retain the Boolean relations of these 
neighborhoods and their intersections,
we obtain a map to the Boolean semiring that
checks the fact that
``John ate a sandwich in the basement at noon"
implies that 
``John ate a sandwich in the basement" and that
``John ate a sandwich at noon", while the
opposite implications do not hold, as desired.

\smallskip

Here we can see, however, that the construction of the map 
$s: {\rm Dom}(h) \to \cS$ that we used in \S \ref{HeadIdRenSec2} 
and \S \ref{SyntImageSec} is only an oversimplified model, and
that it should be refined by directly including coverings by 
neighborhoods related by intersections; see the discussion in  \S \ref{GeodNeighSec}.

\subsection{Predicate saturation in Pietroski's semantics and operadic structure} \label{PredPietSec}

Other important parts of Pietroski's semantics, in addition to the {\em Combine} operation
discussed above, consists of predicate saturation and existential closure, see \cite{Pietro}.
We propose here a way to fit these aspects in our model, compatibly with the form of the
{\em Combine} operation that we just described, using the formulation of the magma $\cS\cO$
of syntactic objects in terms of {\em operads} (which we mentioned briefly in \S \ref{TAGsec}). 

\smallskip

We recall briefly the mathematical notion of an operad, introduced in \cite{May}, and
we describe how to view syntactic objects as an algebra over an operad. 

\smallskip
\subsubsection{Syntactic objects and operads}\label{SOoperadSec}

An operad (in Sets) is a collection $\cO=\{ \cO(n) \}_{n\geq 1}$ of sets of $n$-ary operations 
(with $n$ inputs and one output), with composition laws
\begin{equation}\label{operadgamma}
\gamma: \cO(n) \times \cO(k_1) \times \cdots \times \cO(k_n) \to \cO(k_1+\cdots + k_n)
\end{equation}
that plug the output of an operation in $\cO(k_i)$ into the $i$-th input of an operation in $\cO(n)$.
The composition of these operations $\gamma$ is subjects to requirements of associativity and unitarity,
which we do not write out explicitly here.  
An algebra $\cA$ over an operad $\cO$ (in Sets)
is a set $\cA$ on which the operations of $\cO$ act, namely there are maps
\begin{equation}\label{algoperad}
 \gamma_\cA: \cO(n) \times \cA^n \to \cA 
\end{equation} 
that satisfy compatibility with the operad composition, 
\begin{equation}\label{opactioncond}
\begin{array}{c}
\gamma_\cA (\gamma_\cO(T, T_1, \ldots, T_m), a_{1,1}, \ldots, a_{1,n_1}, \ldots, a_{m,1}, \ldots, a_{m,n_m}) = \\[3mm]
\gamma_\cA( T, \gamma_\cA(T_1, a_{1,1}, \ldots, a_{1,n_1}),\ldots, \gamma_\cA(T_m, a_{m,1}, \ldots, a_{m,n_m})  )\, .
\end{array}
\end{equation}
for $T\in \cO(m)$, $T_i\in \cO(n_i)$ and $\{ a_{i,j} \}_{j=1}^{n_i} \subset \cA$, and with $\gamma_\cO$ the composition
in the operad and $\gamma_\cA$ the operad action.  
This notion means that elements of 
the set $\cA$ can be used as inputs for the operations in $\cO$, resulting in
an output that is again an element in $\cA$. The category of Sets can be replaced by 
more general symmetric monoidal categories. In particular we can consider cases
where $\cA$ is a topological space, or a vector space, which are suitable for the
setting of semantic spaces. The description of the operadic composition laws that
we mentioned in \S \ref{TAGsec}, in terms of the compositions $\circ_i: \cO(n)\times \cO(m)\to \cO(n+m-1)$
is equivalent, for unitary operads, to the description in terms of the compositions \eqref{operadgamma}. 

\smallskip

In particular, we are interested here in the operad $\cM$ freely generated by a single commutative binary operation $\fM$, where
we have $\cM(1)=\{ {\rm id} \}$, $\cM(2)=\{ \fM \}$, $\cM(3)=\{ \fM\circ ({\rm id}\times \fM), \fM\circ (\fM\times {\rm id }) \}$, etc.
Consider again the set of syntactic objects $\cS\cO$. The magma structure of \eqref{SOmagma} can be reformulated as
the structure of algebra over this operad.

\begin{lem}\label{SOoperad}
The set $\cS\cO$ of syntactic objects is an algebra over the operad $\cM$ freely generated by the 
single commutative binary operation $\fM$.
\end{lem}

\proof We can identify the elements in $\cM(n)$ with the abstract binary rooted trees with $n$ leaves (with no labels on the leaves),
where each internal (non-leaf) vertex is labelled by an $\fM$ operation. 
The maps \eqref{algoperad} are simply given by taking
$\gamma(T , T_1,\ldots,T_n)$ with $T\in \cM(n)$ and $T_i \in \cS\cO$ for $i=1, \ldots , n$ to be the abstract binary rooted tree
in $\fT_{\cS\cO_0}=\cS\cO$ obtained by grafting the root of the syntactic object $T_i$ to the $i$-th leaf of $T\in \cM(n)$. 
If the syntactic objects $T_i$ have $n_i$ leaves, then the syntactic object $\gamma(T,T_1,\ldots,T_n)$ obtained in this
way has $n_1+\cdots+ n_k$ leaves. Note that this operad action is just a repeated application of the product operation $\fM$
in the magma $\cS\cO$, hence the description as algebra over $\fM$ and as magma as in \eqref{SOmagma} are equivalent.
\endproof

 \smallskip
 \subsubsection{Semantic spaces and operads}\label{SemaOperadSec}

 \smallskip
 
 The additional structure that we want to consider here, on the side of semantic spaces, is that of a partial algebra over the
 operad $\cM$.
 
 \begin{defn}\label{operadS} {\rm 
 Let $\cM$ be the operad freely generated by a single commutative binary operation $\fM$. A semantic
 space $\cS$ is a {\em compositional semantic space} if it has the following properties:
 \begin{enumerate}
 \item There is a map $s: {\rm Dom}(h) \to \cS$ extending $s: \cS\cO_0 \to \cS$.
 \item There is an action of the operad $\cM$ on $\cS$
 \begin{equation}\label{SalgoperadM}
 \gamma_\cS: \cM(n) \times \cS^n \to \cS \, .
\end{equation}  
 \item For $T\in \cM(n)$ and for $T_1, \ldots, T_n \in {\rm Dom}(h)\subset \cS\cO$  such that
 $$  \gamma_{\cS\cO} (T, T_1, \ldots, T_n)\in {\rm Dom}(h) $$ we have
 \begin{equation}\label{Soperad2}
 \gamma_\cS (T, s(T_1), \ldots, s(T_n)) = s( \gamma_{\cS\cO} (T, T_1, \ldots, T_n) ) \, .
 \end{equation}
\end{enumerate}  }
\end{defn} 

The last condition ensures that the structure of $\cS\cO$ as an algebra over the operad $\cM$ and
the structure of $\cS$ as a partial algebra over the same operad $\cM$ are compatible through the
map $s: {\rm Dom}(h) \to \cS$ from syntax to semantics.

\smallskip

One can more generally consider partial actions of an operad and a corresponding notion of {\em partial algebra over an operad}
 (introduced in \cite{KrizMay}), where the operad action \eqref{algoperad} is defined on a subdomain $\cA_0 \subset \cA$,
 \begin{equation}\label{algoperad}
 \gamma_\cA: \cO(n) \times \cA_0^n \to \cA \, .
\end{equation} 

\smallskip

For a compositional semantic space as in Definition~\ref{operadS} the predicate saturation operation of Pietroski's semantics,
in a form compatible with syntactic Merge, can be can be  interpreted as the operad action \eqref{SalgoperadM} that saturates the
arguments of an $n$-ary operation by inputs in $\cS_0$ (a concept of adicity $n$ combined with $n$ semantic arguments). 
The partial compositions $\circ_i$ correspondingly give the combinations of a concept of adicity $n$ with one 
semantic argument that give a concept of adicity $n-1$. Note, however, that there is an important difference here. In this model
the operations of adicity $n$ in $\cM(n)$ are part of the syntax core computational mechanism. They are not on the semantic
side, so they cannot directly be identified with the ``concept of adicity $n$" described in \cite{Pietro}. It is only through the
relation \eqref{Soperad2} that they acquire that role. 

\smallskip
\subsubsection{Syntax-driven compositional semantics}\label{SyntaxOperadSec}

The notion of compositional semantic space that we described in Definition~\ref{operadS}
is based on two operad actions, one (that we called $\gamma_\cS\cO$) on the side of syntax and one (that we
called $\gamma_\cS$) on the side of semantics, with the compatibility \eqref{Soperad2}. This is similar to
the formulation of Pietroski's semantics in \cite{Pietro}. However, we show now that in fact the operad
action $\gamma_{\cS\cO}$ on syntax suffices to completely determine its counterpart $\gamma_\cS$.

\begin{prop}\label{SOgammaS}
Let $\cS^+=\cS\cup \{ s_\infty \}$ be the Alexandrov one-point compactification of $\cS$, where we denote the added point with the symbol $s_\infty$. 
The action of the operad $\cM$ on syntactic objects, described in Lemma~\ref{SOoperad}, together with a function $s:{\rm Dom}(h)\to \cS$ 
uniquely determine an action of the operad $\cM$ on $\cS$ by setting
\begin{equation}\label{setgammaS}
\gamma_\cS(T, s_1, \ldots, s_n) :=\left\{ \begin{array}{ll} s(\gamma_{\cS\cO}(T, T_1, \ldots, T_n)) & \text{if }\, s_i=s(T_i) \text{ and } \gamma_{\cS\cO}(T, T_1, \ldots, T_n)\in {\rm Dom}(h) \\[3mm]
s_\infty & \text{otherwise.} 
\end{array} \right. 
\end{equation}
for $T\in \cM(n)$ and $s_1,\ldots,s_m\in \cS^+$. 
\end{prop}

\proof We construct $\gamma_\cS$ using $\gamma_{\cS\cO}$ and the compatibility relation \eqref{Soperad2} using \eqref{setgammaS}, 
In order to show that \eqref{setgammaS} does indeed define an operad action on $\cS$, we need to check 
the compatibility of $\gamma_\cS$ with the operad composition $\gamma$ in $\cM$, given by the condition \eqref{opactioncond}. The left-hand-side
of \eqref{opactioncond} gives
\begin{equation}\label{condlhs}
 \gamma_\cS (\gamma_{\cM}(T, T_1, \ldots, T_m), s_{1,1}, \ldots, s_{1,n_1}, \ldots, s_{m,1}, \ldots, s_{m,n_m})\, .  
\end{equation} 
This is equal to $s_\infty$ unless both of the two conditions
\begin{itemize}
\item all the $s_{i,j}=s(T_{i,j})$ for some $T_{i,j}$ in ${\rm Dom}(h)\subset \cS\cO$; 
\item the syntactic object 
\begin{equation}\label{lhsSO}
 \gamma_{\cS\cO}(\gamma_{\cM}(T, T_1, \ldots, T_m), T_{1,1}, \ldots, T_{1,n_1}, \ldots, T_{m,1}, \ldots, T_{m,n_m}) 
\end{equation} 
is in ${\rm Dom}(h)$
\end{itemize}
are satisfied, in which case \eqref{condlhs} is equal to
\begin{equation}\label{condlhs2}
 s( \gamma_{\cS\cO}(\gamma_{\cM}(T, T_1, \ldots, T_m), T_{1,1}, \ldots, T_{1,n_1}, \ldots, T_{m,1}, \ldots, T_{m,n_m}) ) \, . 
 \end{equation}
 The compatibility of the action $\gamma_{\cS\cO}$ with the operad composition implies that \eqref{lhsSO} is equal to
 \begin{equation}\label{lhsSO2}
 \gamma_{\cS\cO}(T, \gamma_{\cS\cO}(T_1,T_{1,1}, \ldots, T_{1,n_1}),\ldots, \gamma_{\cS\cO}(T_m,T_{m,1}, \ldots, T_{m,n_m})) \, . 
 \end{equation} 
 Note that if the full composition in \eqref{lhsSO} is in ${\rm Dom}(h)$ by the properties of abstract head functions
 all the substructures $\gamma_{\cM}(T_i,T_{i,1}, \ldots, T_{i,n_i})$, $i=1,\ldots, m$, are also in  ${\rm Dom}(h)$. 
Thus, the point \eqref{condlhs2} in $\cS$ is the same as the point 
$$ s( \gamma_{\cS\cO}(T, \gamma_{\cS\cO}(T_1,T_{1,1}, \ldots, T_{1,n_1}),\ldots, \gamma_{\cS\cO}(T_m,T_{m,1}, \ldots, T_{m,n_m}))) = $$
$$ \gamma_\cS( T, s(\gamma_{\cS\cO}(T_1,T_{1,1}, \ldots, T_{1,n_1})),\ldots, s(\gamma_{\cS\cO}(T_m,T_{m,1}, \ldots, T_{m,n_m}))) = $$
$$ \gamma_\cS( T, \gamma_\cS(T_1, s_{1,1}, \ldots, s_{1,n_1}), \ldots, \gamma_\cS(T_m, s_{m,1}, \ldots, s_{m,n_m})) \, ,  $$
which gives the right-hand-side of \eqref{opactioncond}. 
\endproof

The structure of algebra over the operad $\cM$ on $\cS^+$ makes $\cS$ a partial algebra over $\cM$. 

\smallskip

Note that we are everywhere somewhat simplifying the picture, as we do not include the possibility that 
different syntactic objects in ${\rm Dom}(h)\subset \cS\cO$ may sometime map to the
{\em same} value in $\cS$ under $s: {\rm Dom}(h)\to \cS$ and also the
possibilities of {\em ambiguities} of 
semantic assignment where $s: {\rm Dom}(h)\to \cS$ may sometimes be multivalued. These possibilities
would affect the construction \eqref{setgammaS} of $\gamma_\cS$ and would require a modified argument. 

\smallskip
\subsection{Adjunction, embedded constructions, and the Pair-Merge problem} \label{RinyExSec}

Our model of the map 
$s: {\rm Dom}(h) \to \cS$, that extends the assignment $s: \cS\cO_0 \to \cS$ 
of semantic values defined on lexical items, is a very simple model built using only 
the head function and proximity relations (and geodesic distance) in semantic space $\cS$.
In particular, since we start with the syntactic objects produced by the free
symmetric Merge, the only factor that introduces asymmetry in this 
construction is coming from the head function.

\smallskip

We discuss here briefly how one can try, within the limits of such an oversimplified 
model, to address the question of misalignments between hierarchical syntax 
and compositional semantics that occur as a consequence of the particular
behavior of adjunction, and in particular what is sometimes referred to as
the invisibility of adjuncts to syntax. This question was posed to us by Riny Huijbregts. 

\smallskip

A proposal for handling this type of problem is to postulate 
the existence of an asymmetrical Pair-Merge operation accounting for 
argument-adjunct asymmetry (see \cite{ChomskyPairMerge}), in addition 
to the free symmetric Merge.
This proposal has undesirable features, as it requires the introduction of
an additional form of asymmetric Merge dealing with the peculiar behavior
of adjunction, while one expects that the computational mechanism of
syntax should just rely entirely on the free symmetric Merge. An
alternative proposal (see for instance \cite{Oseki}) involves the use of
``two-peaked" structures (see Figure~\ref{2peakFig}) 
with $\{ XP , YP \}$ an adjunction. This proposal has the
drawback that, if one considers such ``two-peaked" structures as
part of syntax, then one needs to justify them in terms of the
free Merge generative process, and this is problematic because
the elements of the magma $\cS\cO=\fT_{\cS\cO_0}$ do not
contain such structures, nor does the action of Merge on workspaces
(as can be also seen in the formalization given in our paper \cite{MCB}).
The proposal of ``two-peaked" structures in  \cite{Oseki} is based on \cite{EpKiSe},
but is not formulable within the generative process of a free symmetric Merge.
We are going to discuss briefly what
this means in terms of our model.

\smallskip

The reason why adjunction appears problematic in our setting is that
adjunction can be seen as an instance of syntactic objects $\{ XP, YP \}$
which do not have a well defined head function in the sense we have
been using above, $\{ XP, YP \} \notin {\rm Dom}(h)$. This creates
a problem with our simple model of mapping to semantics, which is 
defined only on ${\rm Dom}(h)$. We want to argue here that 
this problem can be to some
extent bypassed without the need to significantly alter the
construction of the mapping $s: {\rm Dom}(h) \to \cS$, although, of
course we expect that the naive model for this map based 
on the datum of the head function may be replaced by some
more elaborate versions. 

\smallskip

Suppose given a syntactic object of the form $\{ XP, YP \} \notin {\rm Dom}(h)$,
where both $XP$ and $YP$ are in ${\rm Dom}(h)$. 
In terms of our construction, the fact that the head function is not well 
defined on the object $\{ XP, YP \}$ implies that we do not have a
choice of orientation on the geodesic arc between $s(XP)$ and $s(YP)$
in $\cS$ and a corresponding point along this arc at a distance $\P(s(XP), s(YP))$
from the image of the head. We do still have the geodesic arc, though, 
and the measurement $\P(s(XP), s(YP))$ of syntactic relatedness 
between its endpoints. So in terms of this construction, all that a
hypothetical asymmetric Pair-Merge would provide is a choice of
orientation on the geodesic arc. Such a datum is a geometric datum in $\cS$
and does not necessarily require the existence of Pair-Merge as an additional
part of the computational structure of syntax. One can extend
$s: {\rm Dom}(h) \to \cS$ to a slightly larger domain that includes
adjunctions just by the requirements that geodesic arcs in $\cS$
whose endpoints are the two terms of an adjunction come with a
preferred choice of orientation. This choice has the same effect of
a Pair-Merge $\langle XP, YP \rangle$ signifying that the first 
element should be taken to be the ``head" while the second element is 
to be seen as an ``adjunct". Such choice of orientation then ensures
that we can extend the same construction of $s: {\rm Dom}(h) \to \cS$ also
to adjunctions $\{ XP, YP \} \notin {\rm Dom}(h)$. In general one does
not expect that this orientation requirement should be extendable to
other types of syntactic objects $\{ XP, YP \} \notin {\rm Dom}(h)$ that
are not adjunctions. 
The fact that this mechanism does not require any modification of
the syntactic generative process and only involves a metric property
in $\cS$ is consistent with the idea that adjuncts are on a ``separate plane"
(see \cite{ChomskyPairMerge}). 

\smallskip

While this approach can be accommodated within our setting,
it leaves open the question of assigning general criteria for
orientations of geodesic arcs in $\cS$ that generalize the 
choice resulting from a had function, incorporating the case
of adjunctions, but not the case of arbitrary  $\{ XP, YP \} $ objects. 

\smallskip

Thinking in terms of ``two-peaked" structures, on the other hand, 
presents another possibility for treating this problem of adjunctions
in our geometric setting. Given a syntactic object of the form 
$\{ XP, YP \} \notin {\rm Dom}(h)$, where both $XP$ and $YP$ 
are in ${\rm Dom}(h)$, consider in $\cS$ the points $s(XP)$ and
$s(YP)$ and the geodesic arc between them (now without any
preferred assignment of orientation). Suppose given also a 
syntactic object $T=\{ Z, XP \} \in {\rm Dom}(h)$. Since this is
in the domain of the head function, it defines a point $s(T)$
on the geodesic arc between $s(Z)$ and $s(XP)$, where the
geodesic arc is oriented from the end that corresponds to the
head to the other. Thus, we do indeed obtain a ``two-peaked" structure, as
in Figure~\ref{2peakFig}. Note that, while the geodesic arc
between $s(XP)$ and $s(YP)$ does not have an a priori choice
of orientation, the orientation induced by the head on the
geodesic arc between $s(Z)$ and $s(XP)$ induces a unique
consistent orientation on the arc between $s(XP)$ and $s(YP)$.

\begin{figure}[h]
\begin{center}
\includegraphics[scale=0.2]{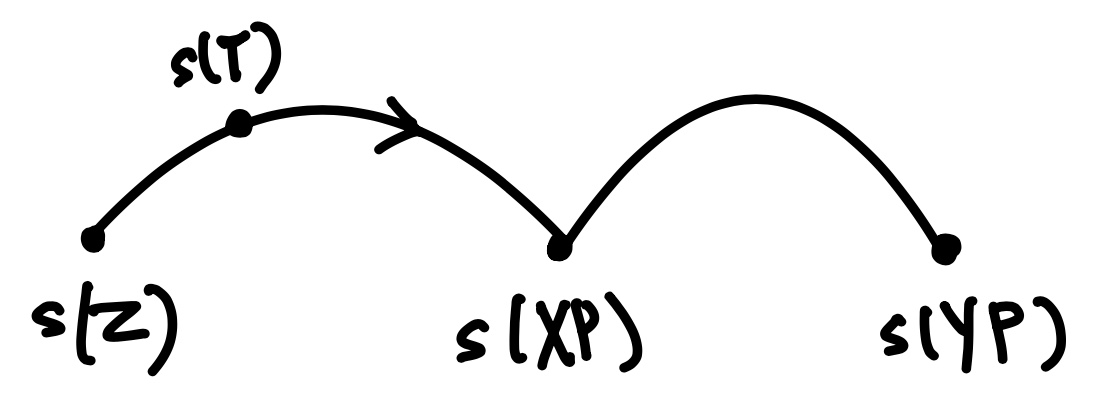}   \ \ \ 
\includegraphics[scale=0.2]{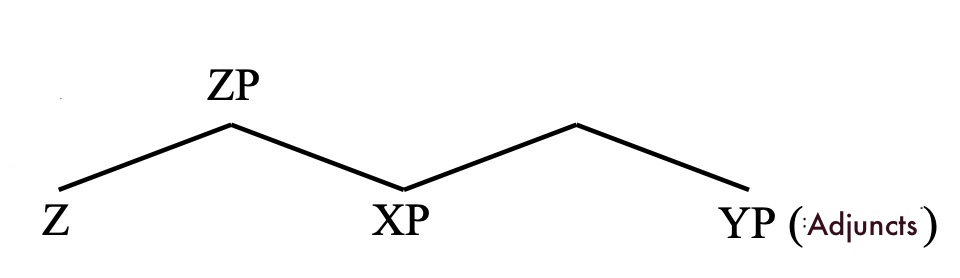}  
\caption{``Two-peaked" structures inside $\cS$. \label{2peakFig}}
\end{center}
\end{figure}

It is important to stress here the difference between the type
of ``two-peaked" structures we are describing and the proposed
 ``two-peaked" structures in the syntactic setting, as in \cite{Oseki}. 
 Here this structure does not exist in the magma $\cS\cO$, it only
 exists in the image of syntax under the map to semantic space $\cS$.
 In other word, these ``two-peaked" structures are not part of the
 computational process of syntax and do not need to be justified
 by any additional form of Merge. They exist because the images
 of syntactic objects inside $\cS$ can intersect, even though the
 resulting configuration (like the one in Figure~\ref{2peakFig})
 is not itself the image of a syntactic object.

\medskip
\section{No, they don't: transformers as characters}\label{TransfSec} 

Recently, it has become fashionable to claim that the so-called
transformer architectures underlying the functioning of many current
large language models (LLMs) somehow ``disprove'' or undermine the
theory of generative linguistics. They don't. Such claims are vacuous:
not only on account that they lack any accurate description of what is
allegedly being disproved, but also more specifically because one can
show, as we will discuss in this section, that the functioning
of the attention modules of transformer architectures fits remarkably
well within the same general formalism we have been illustrating in
the previous sections, and is consequently {\em fully compatible} with a
generative model of syntax based on Merge and Minimalism.   While this
can be discussed more at length elsewhere, we will show here briefly that
the weights of attention modules in transformer architectures can be
regarded as another (distinct from human) way of embedding an
image of syntax inside semantics, with formal properties similar to
other examples we talked about earlier in this paper. 

\smallskip

This does not mean, of course, that LLMs based on such architectures
{\em necessarily} mimic the interaction between syntax and semantics
as it occurs in human brains. In fact, most certainly that is {\em not} the
case in anything close to their present form, given well known
considerations regarding the ``poverty of the stimulus,''  in human language acquisition (see
\cite{BPYC}), compared to what one may
call the ``overwhelming richness of the stimulus'' in the training of
LLMs. Some have attempted to deal with this issue by limiting
the amount of training data to something argued to align more
with the data available to children (e.g., as in \cite{WarBow22} and
\cite{Huebner21}, among others, including an upcoming 2023CoNLL/CMCL
``Baby LMChallenge'' \cite{BabyLM23}). However, at least so far there are
  still problems with such approaches with regard to both performance
  on certain test-bed datasets, and accurately mirroring the
  developmental trajectory of human language acquisition, with respect
  to training data sample sizes. This matter is discussed in more
  detail in \cite{Vazquez23}.
  
  \smallskip

This is not the main point of the discussion here, however,
since several examples we analyzed in the previous sections are also
not meant to model how syntax and semantics realistically interact in the human
brain, but are presented simply as illustrations of the general formal
algebraic properties of the mathematical model. The point we intend to
make here is that attention modules of transformer architectures can
function as another choice of a Hopf algebra character that fits
within the same very general algebraic formalism we illustrated in the
previous sections of this paper. Therefore, transformer architectures have no intrinsic
incompatibility, at this fundamental algebraic level, with generative
syntax. Note also that we are not going to include here any discussion
with regard to the efficiency of computational algorithms, as we are interested only
in analyzing their algebraic structure. We will only make some general
comments at the end of this section, in relation to the ``inverse
problem'' of reconstructing syntax from its image inside semantics,
that we already discussed in \S \ref{SyntImageSec}.

\smallskip

For our purposes, it suffices to consider the basic fundamental
functioning of attention modules in transformers, that we recall
schematically as follows.

\smallskip

We assume, as in our previous setting in \S \ref{HeadIdRenSec}, 
a given function $s: \cS\cO_0\to \cS$ from lexical items and syntactic
features to a semantic space $\cS$ that is here assumed to be
a vector space model. Thus, we can view elements $\ell\in \cS\cO_0$ 
as vectors $s(\ell)\in \cS$. In attention modules, in the case of so-called
self-attention that we focus on here, one considers three linear
transformations: $Q$ (queries), $K$ (keys), and $V$ (values), 
$Q,K\in \Hom(\cS,\cS')$ and $V\in \Hom(\cS,\cS'')$, where $\cS'$ and $\cS''$ 
are themselves vector spaces of semantic vectors (in general of dimensions 
not necessarily equal to that of $\cS$).

One usually assumes given identifications $\cS\simeq \R^n$, $\cS'\simeq \R^m$, $\cS''\simeq \R^d$ with 
Euclidean vector spaces, with assigned bases, and one works with the corresponding
matrix representations of $Q, K \in \Hom(\R^n,\R^m)$ and 
$V\in \Hom(\R^n,\R^d)$. The target Euclidean space 
$\cS'$ is endowed with an inner product $\langle \cdot, \cdot \rangle$, that can
be used to estimate semantic similarity. 

\smallskip

The query vector $Q(s(\ell))$, for $\ell\in \cS\cO_0$, can be thought of
performing a role analogous to the {\em semantic probes} discussed in our toy models
of \S \ref{HeadIdRenSec}. As in that case, we think of queries (or probes in
our previous terminology) as elements $q \in \cS^\vee$ where $\cS^\vee$ is
the dual vector space $\cS^\vee=\Hom(\cS,\R)$, so that a query matrix can
be identified with an element in $\cS^\vee \otimes \R^m \simeq \cS^\vee \otimes \cS'
=\Hom(\cS,\cS')$, that we can regard as an $m$-fold probe $Q$ evaluated on the
given semantic vector $s(\ell)$.

In a similar way, we can think of the key vector $K(s(\ell))$, for
$\ell\in \cS\cO_0$, also as an element $K\in \Hom(\cS,\cS')$, that we
interpret in this case as a way of creating an $m$-fold probe out of
the given vector $s(\ell)$.  Thus, the space $\cS'$ of ($m$-fold)
probes plays a dual role here, as given probes to be evaluated on an
input semantic vector $s(\ell)$, and as new probes generated by the
semantic vector $s(\ell)$. This dual interpretation explains the use
of the terminology ``query'' and ``key'' for the two given linear
transformations. The values vector $V(s(\ell))$ can be viewed as a
representation of the semantic vectors $s(\ell)$ {\em inside} $\cS''$, such
as, for example, an embedding of the set $s(L)$, for a given subset
$L\subset \cS\cO_0$, into a vector space $\cS''$, of dimension lower 
than $\cS$.  One refers to $d=\dim \cS''$ as the embedding dimension.

\smallskip

Next, one considers a set $L \subset \cS\cO_0$. Usually, this is regarded as
an ordered set, a list (also called a string), that would correspond to
an input sentence. However, in our setting, it is more convenient to
consider $L$ as an unordered set. In terms of transformer models, one 
then focuses on bi-directional architectures like BERT.
To an element $\ell \in L$ one
assigns an attention operator $A_\ell : L \subset \cS \to \cS'$, given by
$$ A_\ell(s(\ell'))=\sigma( \langle Q(s(\ell)), K(s(\ell')) \rangle )\, , $$
where $\sigma$ is the softmax function
$\sigma (x)_i= \exp(x_i)/\sum_j \exp(x_j)$, for $x=(x_i)$. 

Note that for simplicity of notation, we are ignoring here the usual
rescaling factor that divides by the square root of the embedding
dimension, since that has no influence on the algebraic structure of
the model, even through it has computational significance. We write
$A_{\ell,\ell'}:=A_\ell(s(\ell'))$ and refer to it as the attention
matrix.  The matrix entries $A_{\ell,\ell'}$ are regarded as a probability
measure of how the attention from position $\ell$ is distributed
towards the other positions $\ell'$ in the set $L$.  One then assigns an
output (in $\cS''$) to the input $s(L)\subset \cS$, as the vectors
$y_\ell =\sum_{\ell'} A_{\ell,\ell'} V(s(\ell'))$, where for each
$\ell\in L$, we have $y_\ell=(y_\ell)_{i=1}^d \in \cS''\simeq \R^d$.

Observe that in writing $A$ as a matrix one uses a choice of ordering of
the set $L$, but the linear operator $A_\ell$ itself is defined independently
of such an ordering.  Compatibly with the fact that we want to use free
symmetric Merge as generator of syntactic objects, we indeed focus
here on the case of bidirectional, non-causal attention, where the
non-trivial entries of the attention matrix are not limited to items
occurring in a specified linear order (i.e.~the matrix is not
necessarily lower or upper diagonal in a chosen basis/ordering). The
resulting $y_\ell$ is symmetric in the ordering of $L$, so linear
ordering also does not play a role in the output.

\smallskip
\subsubsection{Heads and heads}\label{headsSec}

In transformer architectures, one usually has several such attention modules
running in parallel, and one refers to this setting as multi-head attention. In
this case, the vectors $Q(s(\ell))=\oplus_i Q(s(\ell))_i$, $K(s(\ell))=\oplus_i K(s(\ell))_i$,
 and $V(s(\ell))=\oplus_j V(s(\ell))_j$ are split into
blocks, that correspond to a decomposition $\cS'=\oplus_{i=1}^N \cS'_i$,
and similarly for $\cS''$, with the inner product of $\cS'$ compatible with the
direct sum decomposition, inducing inner products $\langle \cdot, \cdot \rangle_{\cS'_i}$.
One can then compute attention matrices, for $i=1,\ldots, N$, 
$$ A^{(i)}_{\ell,\ell'} =\sigma( \langle Q(s(\ell))_i ,K(s(\ell))_i \rangle_{\cS'_i}) $$
that one refers to as attention distribution with {\em attention head} $i$.

\smallskip

It is important to keep in mind that there is an unfortunate clash of
notation here, between this meaning of ``head" as ``attention head" versus
the usual syntactic meaning of ``syntactic head", represented in the present paper by
the notion of ``head function" in Definition~\ref{headfunc}. 

\smallskip

For simplicity, and to avoid confusing notation, we will not consider
here multiple attention heads, and work only with a single attention
matrix, that suffices for our illustrative purposes, while we will be
referring to the term {\em head} only in its syntactic meaning as a head function. 

\smallskip
\subsection{Maximizing attention}

Since for fixed $\ell \in L$ the values $A_{\ell,\ell'}$ give a probability
measure on $L$, we can consider characters with values in the 
semiring $\cR=([0,1],\max,\cdot)$. For example, it is natural to
look for where the attention from position $\ell$ is maximized.
Thus, we can define a character on a subdomain
$$ \phi_A: \cH^{semi} \to \cR $$
by setting 
$$ \phi_A(T) = \max_{\ell\in L(T)} A_{h(T),\ell} \, , $$
if $T\in {\rm Dom}(h)$ and zero otherwise.

\begin{rem}\label{headAell}{\rm 
Note that, in order to make $\phi$ well defined for all $T\in \fT_{\cS\cO_0}$,
we need a uniform choice of the operator $A_\ell$ for an $\ell\in L(T)$, that is
to say, we need a consistent way of extracting the choice of a leaf from
each tree. This corresponds to the choice of a head function $h$ in the sense of
Definition~\ref{headfunc}. }\end{rem}

\smallskip

Once a head function $h$ is chosen, the attention matrix determines
an associated attention vector $A_{h(T),\ell}$ for $\ell\in L(T)$. 
In particular, we can choose the head function to be the same as the
syntactic head, although this is not necessary and any choice of a 
head function will work for this purpose. Note that head functions
are not everywhere defined on $\fT_{\cS\cO_0}$. This implies that
the choice of attention vector cannot be made compatibly with 
substructures simultaneously across all trees $T\in \fT_{\cS\cO_0}$. 
There is some maximal domain ${\rm Dom}(h)\subset \fT_{\cS\cO_0}$
over which such a consistent choice can be made. This issue
does not arise in the construction of attention matrices from text,
as sentences in text will always have a syntactic head, but it can
be relevant when sentences are stochastically generated from a
template (such as those used in tests of linguistic capacities
of LLMs, as in \cite{WarBow22}, \cite{Huebner21}).

\smallskip
\subsection{Attention-detectable syntactic relations}

Recent investigation of attention modules and syntactic relations
(like c-command, see \cite{Manning}) indicate that syntactic trees and examples of specific
syntactic relations such as syntactic head, prepositional object,
possessive noun, and the like, are embedded and detectable from the
attention matrix data.  We show that this result is to be expected,
given our model.

\smallskip

We consider the problem of detection of syntactic relations in the following form.

\begin{defn}\label{detectrels} {\rm
Suppose given a syntactic relation $\rho$, which we write as a 
collection $\rho=\rho_T$ of relations $\rho_T \subset L(T)\times L(T)$,
with $\rho_T(\ell,\ell')=1$ is $\ell,\ell'\in L(T)$ are in the chosen
relation and  $\rho_T(\ell,\ell')=0$ otherwise. We say that $\rho$ is 
{\em exactly attention-detectable} if there exist query/key linear maps
$Q_\rho, K_\rho\in \Hom(\cS,\cS')$ and  there exists a head function $h_\rho$
as in Definition~\ref{headfunc} such that 
$$ \rho_T(h_\rho(T), \ell_{\max,h_\rho}) =1 $$
for all $T\in {\rm Dom}(h_\rho)$, where
$$ \ell_{\max,h_\rho}= {\rm arg max}_{\ell \in L(T)} A_{h_\rho(T), \ell}\, , $$
with $A$ the attention matrix built from $Q_\rho, K_\rho$.

The relation $\rho$ is {\em approximately attention-detectable} if there exist query/key linear maps
$Q_\rho, K_\rho\in \Hom(\cS,\cS')$ and  there exists a head function $h_\rho$
as in Definition~\ref{headfunc} such that 
$$ \frac{1}{\# \cD} \sum_{T \in \cD} \rho(h_\rho(T), \ell_{\max,h_\rho}) \sim 1 $$
for some sufficiently large set $\cD \subset {\rm Dom}(h_\rho)$ of trees.  }
\end{defn}

\smallskip

Here the existence of query/key linear maps
$Q_\rho, K_\rho$ as above is relative to a specified context, such as a corpus, 
a dataset. 

\medskip

In the case of approximately attention-detectable syntactic relations, we think of the
subset $\cD$ as being, for instance, a sufficiently large syntactic treebank corpus, or
a corpus of annotated syntactic dependencies 
(for size estimates see \cite{Manning}). Cases where the existence
of query/key linear maps $Q_\rho, K_\rho$ and a head function $h_\rho$ with
the properties required above can be ensured can be extracted from the
experiments in \cite{Manning}.

\smallskip 
\subsection{Threshold Rota-Baxter structures and attention}

Using a threshold Rota-Baxter operator $c_\lambda$ of weight $+1$, 
we obtain
$$ \phi_{A,-}(T)=c_\lambda( \max \{ \phi_A(T), c_\lambda(\phi_A(F_{\underline{v}})) \cdot \phi_A(T/F_{\underline{v}}),
\ldots, c_\lambda(\phi_A(F_{\underline{v_N}}) ) \cdot \phi_A(F_{\underline{v}_{N-1}} /F_{\underline{v_N}}) \cdots 
\phi_A(T/F_{\underline{v}_1}) \}) \, .$$
As above, for simplicity we focus on the case of chains of subtrees $T_{v_N}\subset T_{v_{N-1}} \subset \cdots \subset T_{v_1} \subset T$ rather than more general subforests.
Note that $h(T/T_v)=h(T)$ for the quotient given by contraction of the subtree, hence
$$ \max_{\ell\in L(T/T_v)} A_{h(T),\ell}\leq \max_{\ell\in L(T)} A_{h(T),\ell} \, . $$
The value $\phi_-(T)$ corresponds then to the chains of nested accessible terms of the
syntactic object $T$ for which all the values 
$$ \phi_A(T_{v_i})=\max_{\ell \in L(T_{v_i})} A_{h(T_{v_i}),\ell} $$ are above the chosen 
threshold $\lambda$ and all the complementary quotients $T_{v_{i-1}}/T_{v_i}$ have
$$ \phi_A( T_{v_{i-1}}/T_{v_i} ) = \max_{\ell\in L(T_{v_{i-1}}/T_{v_i} )} A_{h(T_{v_{i-1}),\ell}} =
\max_{\ell\in L(T_{v_{i-1}})} A_{h(T_{v_{i-1}}),\ell}  =
\phi_A(T_{v_{i-1}}) \, . $$
The first condition implies that one is selecting only chains of accessible terms inside
the syntactic object $T$ where the maximal attention from the head of each subtree
in the chain is sufficiently large, while the second condition means that, among these chains one
is selecting only those for which the recipient of maximal attention from the head of
the given subtree is located outside of the next subtree. This second condition guarantees
that when considering the next nested subtree and trying to maximize for its attention
value, one does not spoil the optimizations achieved at the previous steps for the larger
subtrees. 

A similar procedure can be obtained by additionally introducing direct implementation
of some syntactic constraints. We can see this in the following way.

A syntactic relation $\rho$ determines a character $\phi_\rho$ on trees 
$T\in {\rm Dom}(h)\subset \fT_{\cS\cO_0}$
with values in the Boolean ring $\cB=(\{ 0,1 \}, \max, \cdot)$ where
$$ \phi_\rho(T)=\max_{\ell\in L(T)} \rho(h(T),\ell)\, . $$
This Boolean character detects whether the syntactic relation $\rho$ is realized in
the tree $T$ or not. 

Using a character $$ \phi_{A,\rho}(T)=\max_{\ell\in L(T)} \rho(h(T),\ell) \cdot A_{h(T),\ell} \, , $$
with values in $\cP=([0,1],\max,\cdot)$, one maximizes the attention from the
tree head over the set of $\ell\in L(T)$ that already satisfy the chosen syntactic relation
with respect to the head of the tree.  The corresponding Birkhoff factorization with
threshold Rota-Baxter operators again identifies chains of subtrees that
maximize the attention (above a fixed threshold), in a way that is recursively compatible 
with the larger trees as before, but where now maximization is done only on the set 
where the relation is implemented. Subtrees with $\phi_\rho(T_v)=0$ do not contribute
even if their value of $\max_\ell A_{h(T),\ell}$ is sufficiently large. 

Thus, comparison between the case with character $\phi_A$ and with
character $\phi_{A,\rho}$ identify attention-detectability of the syntactic
property considered and, if detectability fails, at which level in the tree 
(in terms of chains of nested subtrees) the attention matrix maximum
happens outside of where the syntactic relation holds. 

As shown in \cite{Vazquez23}, the current performance on syntactic capacities of LLMs 
trained on small scale data modeling falls significantly short of the human performance, 
when tested on LI-Adger datasets that include sufficiently diverse syntactic phenomena. 
This suggests a good testing ground for syntactic recoverability as outlined above and 
a possible experimental testing for aspects of the inverse problem of the syntax-semantics interface. 

\smallskip
\subsubsection{Syntax as an inverse problem: physics as metaphor}

The question of reconstructing the computational process of
syntax, in LLMs based on transformer architectures, can be
seen in the same light as the situation we illustrated in a simpler
example in \S \ref{SyntImageSec}, where one views the image
of syntax embedded inside a semantic space, and considers the
inverse problem of extracting syntax as a computational process working
from these images, which live in a semantic space that is not itself endowed
with the same type of computational structure. Here, the image of
syntax is encoded in the key/query vectors that live in vector spaces that
organize semantic proximity data, and in the resulting attention matrices. 
Inverse problems of this kind are usually expected to be computationally hard.
This does not mean that the computational mechanism of syntax cannot
be reconstructible, but that a significant cost in complexity, growing rapidly
with the depth of the trees, may be involved. 

\smallskip

Early results showed that RNN language models performed poorly on
tests of grammaticality aimed at capturing syntactic structures, on
a testbed dataset of pairs of sentences that differ only in their grammaticality, 
\cite{MarvLin},  while \cite{Gulo} showed that language models based on RNNs can
perform well on predicting long-distance number agreement even
in the absence of semantic clues (that is, when tested on nonsensical but 
grammatical sentences).  Results like this appear to indicate that syntax can,
in principle, be extracted and disentangled from its image inside semantics. 
It was shown in \cite{ShenTan} that Syntactic Ordered Memory (SOM) syntax-aware language models 
outperform the Chat-GPT2 LLM in syntactic generalization tests.
However, this entire area remains a matter of contention, dependent in
part on the testbed dataset used, as described more fully in
\cite{Vazquez21} and \cite{Vazquez23}.

A more systematic comparison of different language model architectures
and their performance on syntactic tests in \cite{Hu} revealed
substantial differences in syntactic generalization performance by
model architecture, more than by size of the dataset. One can suggest
that the indicators of poor performance on syntactic tests, along with
any other difficulties, might also reflect the computational
difficulty involved in extracting syntax as an inverse problem from
its image through the semantic interface, stored across values of the
weights of attention matrices, rather than in a direct syntax-first
mapping.

\smallskip

In this paper we have used physics as a guideline for identifying
mathematical structures that can be useful in modelling the relation
between syntax and semantics. We conclude here by using physics again,
this time only as a metaphor, for describing the relation of syntax as
a generative process and the functioning of LLMs. 

\smallskip

The generative structure underlying particle physics is given by the
Feynman diagrams of quantum field theory. Disregarding epistemological
issues surrounding the interpretation of such diagrams as events of
particle creation and decay, we can roughly say that, in a particle
physics experiment, what one detects is an image of such objects
embedded into the set of data collected by detectors.  Detecting a
particle, say the Higgs boson (the most famous recent particle physics
discovery), means solving an inverse problem that identifies inside
this enormous set of data the traces of the correct diagrams/processes
involving the creation of a Higgs particle from an interaction of
other particles (such as gluon fusion or vector-boson fusion) and its
subsequent decay into other particles (such as vector-boson pairs or
photons).  The enormous computational difficulty implicit in this task
arises from the need to solve this type of inverse problem, involving
the identification of events structure (for example a Higgs decay into
photons involving top quark loop diagrams) from the measurable data,
and a search for the desired structure in a background involving a
huge number of other simultaneous events. The direct map from quantum
field theory consists of the Higgs boson production cross sections,
which are calculated from perturbative expansions in the Feynman
diagrams of quantum chromodynamics and quantum electrodynamics,
involving significant higher-order quantum corrections.
Such perturbative QFT computations are where the algebraic formalism
recalled at the beginning of this paper plays a role.  The inverse
problem, instead, consists of measuring, for various possible decay
channels, mass and kinematic information like decay angles of
detectable particles of the expected type, produced either by the
expected decay event or by the background of productions of the same
particle types due to other events, and searching for an actual signal
in this background.

\smallskip

We can use this story as a metaphor, and imagine the generative
process of syntax embedded inside LLMs in a conceptually similar way,
its image scattered across a probabilistic smear over a large number of weights and vectors,
trained over large data sets. This view of LLMs as the technological
``particle accelerators" of linguistics, where signals of linguistic
structures are detectable against a background of probabilistic noise,
suggests that such models do not invalidate generative syntax any more
than particle detectors would ``invalidate" quantum field theory;
quite the contrary in fact.

\smallskip

While LLMs do not constitute a model of language in the human brain, they can 
still, in the sense described here, provide an apparatus for the experimental
study of inverse problems in the syntax-semantic interface. Here
however it is essential to recall again the physics metaphor. Data and
technology {\em without theory} do not constitute science, understood as a
model of the fundamental laws of nature that has both strong
predictive capacity {\em and} a high level of concise conceptual
clarity in its explanatory power.  The relation between the
computational process of syntax and the topological relational nature
of semantics is a problem of a conceptual nature.  In this sense, the
large language models may contribute a technological experimental
laboratory for the analysis of some aspects of this problem, rather
than a replacement for the necessary theoretical understanding of
fundamental laws in the structure of language.

\bigskip
\bigskip

\bigskip

\medskip
\subsection*{Acknowledgments} 
The first author acknowledges support from NSF grant DMS-2104330 and
FQXi grants FQXi-RFP-1 804 and FQXi-RFP-CPW-2014, SVCF grant
2020-224047, and support from the Center for Evolutionary Science at
Caltech. We thank Riny Huijbregts, Jack Morava, Norbert Hornstein, and
Barry Schein for helpful discussions and suggestions.

\end{document}